%% file: main_paper.tex
\documentclass{article}
\usepackage[preprint]{icml2026}
\usepackage{microtype}
\usepackage{graphicx}
\usepackage{subcaption}
\usepackage{booktabs} 

\usepackage{hyperref}

\usepackage{algorithm}
\usepackage{algorithmic} 

\usepackage{amsmath}

\usepackage{amssymb}
\usepackage{mathtools}
\usepackage{amsthm}

\usepackage[capitalize,noabbrev]{cleveref}

\input{math_commands}
\usepackage{import}
\usepackage{subcaption}

\usepackage[table]{xcolor}
\definecolor{lightyellow}{RGB}{255,250,205} 

\allowdisplaybreaks

\theoremstyle{plain}
\newtheorem{theorem}{Theorem}[section]
\newtheorem{proposition}[theorem]{Proposition}
\newtheorem{lemma}[theorem]{Lemma}
\newtheorem{corollary}[theorem]{Corollary}
\theoremstyle{definition}
\newtheorem{definition}[theorem]{Definition}
\newtheorem{assumption}[theorem]{Assumption}
\newtheorem{example}[theorem]{Example}
\theoremstyle{remark}
\newtheorem{remark}[theorem]{Remark}

\usepackage[textsize=tiny]{todonotes}

\icmltitlerunning{Periodic Regularized Q-Learning}

\begin{document}

\twocolumn[
  \icmltitle{Periodic Regularized Q-Learning}



  \icmlsetsymbol{equal}{*}

  \begin{icmlauthorlist}
    \icmlauthor{Hyukjun Yang}{equal,sch}
    \icmlauthor{Han-Dong Lim}{equal,sch}
    \icmlauthor{Donghwan Lee}{equal,sch}

  \end{icmlauthorlist}


  \icmlaffiliation{sch}{Department of Electrical Engineering, Korea Advanced Institute of Science and Technology, Daejeon, South Korea}

  \icmlcorrespondingauthor{Firstname1 Lastname1}{first1.last1@xxx.edu}

  \icmlkeywords{Reinforcement Learning, Q-learning, convergence}

  \vskip 0.3in
]



\printAffiliationsAndNotice{}  

\import{contents}{00_abstract}

\import{contents}{01_introduction}

\import{contents}{02_preliminaries}

\import{contents}{08_regularized_projection_operator}

\import{contents}{03_AVI_reg}

\import{contents}{04_PRQ}

\import{contents}{05_main_result}

\import{contents}{06_experiments}

\import{contents}{07_conclusion}

\nocite{langley00}

\bibliography{main_paper}
\bibliographystyle{icml2026}

\newpage
\appendix
\onecolumn

\import{contents}{XX_appendix}


\end{document}

%% file: math_commands.tex
\usepackage{amsmath,amsfonts,bm}

\def\eqref#1{equation~\ref{#1}}

\def\1{\bm{1}}

\def\eps{{\epsilon}}

\def\mGamma{{\bm{\Gamma}}}

\def\vmu{{\bm{\mu}}}
\def\vtheta{{\bm{\theta}}}
\def\va{{\bm{a}}}
\def\vb{{\bm{b}}}

\def\ve{{\bm{e}}}

\def\vv{{\bm{v}}}

\def\vx{{\bm{x}}}
\def\vy{{\bm{y}}}

\def\mA{{\bm{A}}}
\def\mB{{\bm{B}}}

\def\mD{{\bm{D}}}

\def\mI{{\bm{I}}}

\def\mP{{\bm{P}}}
\def\mQ{{\bm{Q}}}
\def\mR{{\bm{R}}}

\def\mPhi{{\bm{\Phi}}}

\def\vphi{{\bm{\phi}}}
\def\mPi{{\bm{\Pi}}}

\def\vmu{{\bm{\mu}}}

\def\vxi{{\bm{\xi}}}
\def\vDelta{{\bm{\Delta}}}

\DeclareMathAlphabet{\mathsfit}{\encodingdefault}{\sfdefault}{m}{sl}
\SetMathAlphabet{\mathsfit}{bold}{\encodingdefault}{\sfdefault}{bx}{n}

\def\gA{{\mathcal{A}}}

\def\gE{{\mathcal{E}}}
\def\gF{{\mathcal{F}}}

\def\gO{{\mathcal{O}}}
\def\gP{{\mathcal{P}}}

\def\gS{{\mathcal{S}}}
\def\gT{{\mathcal{T}}}

\def\sN{{\mathbb{N}}}

\newcommand{\E}{\mathbb{E}}

\newcommand{\R}{\mathbb{R}}

\DeclareMathOperator*{\argmax}{arg\,max}
\DeclareMathOperator*{\argmin}{arg\,min}

\usepackage{enumitem}
\usepackage{pifont}
\newcommand{\cmark}{\ding{51}} 
\newcommand{\xmark}{\ding{55}} 

%% file: contents/00_abstract.tex
\begin{abstract}

In reinforcement learning (RL), Q-learning is a fundamental algorithm whose convergence is guaranteed in the tabular setting. However, this convergence guarantee does not hold under linear function approximation. To overcome this limitation, a significant line of research has introduced regularization techniques to ensure stable convergence under function approximation. In this work, we propose a new algorithm, periodic regularized Q-learning (PRQ). We first introduce regularization at the level of the projection operator and explicitly construct a regularized projected value iteration (RP-VI), subsequently extending it to a sample-based RL algorithm. By appropriately regularizing the projection operator, the resulting projected value iteration becomes a contraction. By extending this regularized projection into the stochastic setting, we establish the PRQ algorithm and provide a rigorous theoretical analysis that proves finite-time convergence guarantees for PRQ under linear function approximation.

\end{abstract}

%% file: contents/01_introduction.tex
\section{Introduction}

Recent advances in deep reinforcement learning (deep RL) have achieved remarkable empirical success across a wide range of domains, including board games such as Go~\citep{silver2017mastering} and video games such as Atari~\citep{mnih2013playing}. At the foundation of these achievements lies one of the most fundamental algorithms in reinforcement learning (RL), known as Q-learning~\citep{watkins1992q}. 
Despite its simplicity and broad applicability, the theoretical understanding of the convergence properties of Q-learning is still incomplete. The tabular version of Q-learning is known to converge under standard assumptions, but when combined with function approximation, the algorithm can exhibit instability. This phenomenon is commonly attributed to the so-called deadly triad of off-policy learning, bootstrapping, and function approximation~\citep{sutton1998reinforcement}. Such instability appears even in the relatively simple case of linear function approximation. 
To address these challenges, a substantial body of research has sought to identify sufficient conditions for convergence~\citep{melo2007convergence, melo2008analysis, yang2019sample, lee2020unified, chen2022finite, lim2025understanding} or to design regularized or constrained variants of Q-learning that promote stable learning dynamics~\citep{gallici2024simplifying, limregularized, maei2010toward, zhang2021breaking, lu2021convex, devraj2017zap}. Among these approaches, our focus lies on regularization in Q-learning, where a properly designed regularizer facilitates convergence and stabilizes the iterative learning process. However, we hypothesize that regularization alone is insufficient for stable convergence in Q-learning. Introducing periodic parameter updates, which separate the update rule into an inner convex optimization and an outer Bellman update, is the key structure to stabilize learning and successfully converge to the desired solution.
Building on this perspective, we propose a new framework that introduces the principles of periodic updates into the structure of a regularized method. We refer to this unified approach as \textit{periodic regularized Q-learning} (PRQ). By incorporating a parameterized regularizer into the projection step, PRQ induces a contraction mapping in the projected Bellman operator. This property ensures both stable and provable convergence of the learning process.

\subsection{Related works}
\paragraph{Regularized methods and Bellman equation} 
RL with function approximation frequently suffers from instability. A prominent approach to address this issue is to introduce regularization into the algorithm, a direction explored by several prior works. Regularization has been widely employed to stabilize temporal-difference (TD) learning~\citep{sutton1998reinforcement} and Q-learning, improving convergence under challenging conditions. \citet{farahmand2016regularized} studied a regularized policy iteration which solves a regularized policy evaluation problem and then takes a policy improvement step. The authors derived the performance loss and used a regularization coefficient which decreases as the number of samples used in the policy evaluation step increases.~\citet{bertsekas2011temporal} applied a regularized approach to solve a policy evaluation problem with singular feature matrices.~\citet{zhang2021breaking} studied convergence of Q-learning with a target network and a projection method.~\citet{limregularized} studied convergence of Q-learning with regularization without using a target network or requiring projection onto a ball.~\citet{manek2022pitfalls} studied fixed points of off-policy TD-learning algorithms with regularization, showing that error bounds can be large under certain ill-conditioned scenarios. Meanwhile, a different line of research~\cite{geist2019theory} focuses on regularization on the policy parametrization.

\paragraph{Target-based update} In a broader sense, our periodic update mechanism can be viewed as a target-based approach, as it intentionally holds one set of parameters stationary while updating the other. This target-based paradigm was originally introduced in temporal-difference learning to improve stability and convergence, and has since been extended to Q-learning.~\citet{lee2019target} studied finite-time analysis of TD-learning, followed by~\citet{lee2020periodic}, who presented a non-asymptotic analysis under the tabular setup. Further research has addressed specific algorithmic modifications. For instance,~\citet{chen2023target} examined truncation methods, while~\citet{che2024target} explored the effects of overparameterization.~\citet{asadi2024td} studied target network updates of TD-learning. Focusing on off-policy TD learning,~\citet{fellows2023target} investigated a target network update mechanism combined with a regularization term that vanishes when the target parameters and the current iterate coincide, under the assumption of bounded variance. Finally,~\citet{wu2025unifying} studied convergence of TD-learning and target-based TD learning from a matrix splitting perspective.

\subsection{Contributions}

Our main contributions are summarized as follows:

\begin{enumerate}
    \item We formulate the regularized projected Bellman equation (RP-BE) and the associated regularized projected value iteration (RP-VI), and provide a convergence analysis of the resulting operator. Building on its convergence analysis, we develop PRQ, a fully model-free RL algorithm.
    \item We develop a rigorous theoretical analysis of PRQ establishing finite-time convergence and sample-complexity bounds under both i.i.d. and Markovian observation models. Our results provide non-asymptotic convergence guarantees for Q-learning with linear function approximation using a single regularization mechanism. These guarantees hold in a broad range of settings without relying on truncation, projection, or strong local convexity assumptions~\citep{zhang2021breaking, chen2023target, limregularized, zhang2023convergence}.

    \item We empirically demonstrate that the joint use of periodic target updates~\citep{lee2020periodic} and regularization~\citep{limregularized} is crucial for stable learning. In particular, we provide counterexamples showing that the algorithm can fail when either component is removed, while stable learning is achieved only when both mechanisms are employed.
\end{enumerate}

%% file: contents/02_preliminaries.tex
\section{Preliminaries and notations}
\label{sec:prelim}

\paragraph{Markov decision process} A Markov decision process (MDP) consists of a 5-tuple $(\gS,\gA,\gamma,\gP,r)$, where $\gS:=\{1,2,\dots,|\gS|\}$ and $\gA:=\{1,2,\dots,|\gA|\}$ are the finite sets of states and actions, respectively, and $\gamma\in(0,1)$ is the discount factor. $\gP:\gS\times\gA\to \Delta(\gS)$ is the Markov transition kernel, and $r:\gS\times\gA\times\gS\to\R$ is the reward function. A policy $\pi:\gS\to\Delta^{\gA}$ defines a probability distribution over the action space for each state, and a deterministic policy $\pi:\gS\to\gA$ maps a state $s$ to an action $a\in\gA$. The set of deterministic policies is denoted as $\Omega$. An agent at state $s$ selects an action $a$ following a policy $\pi$, transitions to the next state $s^{\prime}\sim \gP(\cdot\mid s,a)$, and receives a reward $r(s,a,s^{\prime})$. The action-value function induced by policy $\pi$ is the expected sum of discounted rewards following a policy $\pi$, i.e., $Q^{\pi}(s,a)=\E\left[ \sum^{\infty}_{k=0}\gamma^k r(s_k,a_k,s_{k+1})\mid (s_0,a_0)=(s,a) \right]$. The goal is to find a policy $\pi$ that maximizes the overall sum of rewards $\pi^*:=\argmax_{\pi\in\Omega}\E\left[ \sum^{\infty}_{k=0}\gamma^k r(s_k,a_k,s_{k+1}) \middle| \pi \right]$. We denote the action-value function induced by $\pi^*$ as $Q^*:\gS\times\gA\to \R$, and $\pi^*$ can be recovered from $Q^*$ by the greedy policy, i.e., $\pi^*(s)=\argmax_{a\in\gA}Q^*(s,a)$. $Q^*$ can be obtained by solving the Bellman optimality equation: $Q^*(s,a)=\E[r(s,a,s^{\prime})+\gamma \max_{u\in\gA} Q^*(s^{\prime},u)\mid s,a]$.

\paragraph{Notations} Let us introduce some matrix notations used throughout the paper. $\mD\in\R^{|\gS||\gA|\times|\gS||\gA|}$ is a diagonal matrix such that $[\mD]_{(s-1)|\gA|+a,(s-1)|\gA|+a}=d(s,a)$ where $d$ is a probability distribution over the state-action space, which will be clarified in a further section; $\mP\in\R^{|\gS||\gA|\times|\gS|}$ is defined such that $[\mP]_{(s-1)|\gA|+a,s^{\prime}}=\gP(s^{\prime}\mid s,a)$; and $\mR\in\R^{|\gS||\gA|}$ is such that $[\mR]_{(s-1)|\gA|+a}=\E\left[ r(s,a,s^{\prime}) \middle | s,a \right]$. For a vector $\mQ\in\R^{|\gS||\gA|}$, the greedy policy with respect to $\mQ$, $\pi_{\mQ}:\gS \to \gA$ is defined as $\pi(s)=\argmax_{a\in\gA}  (\ve_s\otimes\ve_a)^{\top}\mQ$ where $\ve_s\in\R^{|\gS|}$ and $\ve_a\in\R^{|\gA|}$ are unit vectors whose $s$-th and $a$-th elements are one, while all others are zero, respectively. $\otimes$ denotes the Kronecker product. Moreover, we denote a policy defined by a deterministic policy $\pi\in\Omega$ as a matrix notation $\mPi_{\pi}\in\R^{|\gS|\times |\gS||\gA|}$ such that the $s$-th row vector is $(\ve_{s}\otimes \ve_{\pi(s)})^{\top}$ for $s\in\gS$. For simplicity, we denote $\mPi_{\mQ}:=\mPi_{\pi_{\mQ}}$. A linear parametrization is used to represent an action-value function induced by a policy $\pi$, $Q^{\pi}(s,a)\approx\vphi^{\top}(s,a)\vtheta$ given a feature map $\vphi:\gS\times\gA\to\R^h$. $\vtheta$ is the learnable parameter and $h$ is the feature dimension. We denote by $\mPhi \in \mathbb{R}^{|\mathcal{S}||\mathcal{A}|\times h}$ the feature matrix, where the row indexed by $(s-1)|\mathcal{A}| + a$ corresponds to $\vphi(s,a)^{\top}$. Throughout the paper, let us adopt the following standard assumption on the feature matrix:
\begin{assumption}\label{assmp:feature_matrix}
    $\mPhi\in\R^{|\gS||\gA| \times h}$ is a full-column rank matrix and $||\mPhi||_{\infty}\leq 1$.
\end{assumption}

\subsection{Projected Bellman equation}

The Bellman operator  $\gT\mQ=\mR+\gamma\mP\mPi_{\mQ}\mQ$ is a non-linear operator that may yield a vector outside the image of $\mPhi\in\R^{|\gS||\gA|\times h}$. Therefore, a composition of the Bellman operator and the weighted Euclidean projection is often used, yielding the following equation
\begin{align}
 \mGamma \gT\mPhi\vtheta = \mPhi\vtheta \label{eq:pbe_1}   
\end{align}
 where $\mGamma:=\mPhi(\mPhi^{\top}\mD\mPhi)^{-1}\mPhi^{\top}\mD$ is the weighted Euclidean projection operator. This equation is called the projected Bellman equation (P-BE). To find the solution of the above equation (we defer the discussion of existence and uniqueness of the solution to a later section), we consider minimizing the following objective function:
\begin{align}
    f(\vtheta) =\frac{1}{2}\left\| \mGamma (\mR+\gamma\mP\mPi_{\mPhi\vtheta}\mPhi\vtheta)-\mPhi\vtheta \right\|_{\mD}^2 \label{eq:mspbe}.
\end{align}
 Since the max operator $\mPi$ introduces nonsmoothness, the function $f$ is non-differentiable at certain points. Therefore, to find the minimizer of $f(\vtheta)$, we investigate the Clarke subdifferential~\cite{clarke1981generalized} of the above objective, which satisfies
\begin{align*}
    \partial \! f(\vtheta)\! \subseteq & \mathrm{conv}\{ \!(\gamma\mP\mPi_{\beta}\mPhi \!-\!\mPhi)^{\top}\!\mD\mGamma(\gT\mPhi\vtheta \!-\!\mPhi\vtheta) \! \mid \! \beta \!\in\! \Lambda(\vtheta)\! \}
\end{align*}
where $\Lambda(\vtheta):=\{ \pi \in \Omega : \pi(s) \in \argmax_{a\in\gA} \vphi(s,a)^{\top}\vtheta \}$ and $\mathrm{conv}(A)$ for a set $A$ denotes the convex hull of the set $A$. The detailed derivation is deferred to Lemma~\ref{lem:subdifferential_f_cvx_hull_derivation} in the Appendix. A necessary condition for some point $\vtheta\in\R^h$ to be a minimizer of $f$ is
\begin{align*}
    0 \in \partial f(\vtheta).
\end{align*}
Such a point $\vtheta$ is called a (Clarke) stationary point~\cite{clarke1981generalized}. At a stationary point $\vtheta$, there exists some policy $\beta$ such that
\[{(\gamma \mP{\mPi _\beta }\mPhi  - \mPhi )^\top} \mD \mGamma (\gT\mPhi \vtheta  - \mPhi \vtheta ) = \bm{0}\]
or equivalently
\begin{align*}
&{\mPhi ^\top}{(\gamma \mP{\mPi_{\beta}} - \mI)^\top}\mD\mPhi \vtheta \\
=& {\mPhi ^\top}{(\gamma \mP{\mPi_{\beta}} - \mI)^\top}\mD\mGamma (\mR + \gamma \mP{\mPi _{\mPhi \vtheta }}\mPhi \vtheta)    .
\end{align*}

Assuming that ${\mPhi ^\top}{(\gamma \mP{\mPi _\beta } - \mI)^\top}\mD\mPhi$ is invertible, we obtain the P-BE in~(\ref{eq:pbe_1}). Since a stationary point always exists, a solution to the P-BE also exists, under the assumption that ${\mPhi ^\top}{(\gamma \mP{\mPi_{\beta}} - \mI)^\top}\mD\mPhi$ is invertible at the stationary point. It will admit a unique solution if $\mGamma \gT$ is a contraction. This P-BE can be equivalently written as\begin{align}
\mPhi^{\top}\mD\mR+\gamma\mPhi^{\top}\mD\mP\mPi_{\mPhi\vtheta}\mPhi\vtheta=(\mPhi^{\top}\mD\mPhi)\vtheta. \label{eq:pbe_1_equivalant}
\end{align}
Despite its simple appearance, the P-BE is not guaranteed to have a unique solution, and in some cases may not admit any solution at all~\citep{de2000existence, meyn2024projected}.
If the P-BE does not admit a fixed point, this means that, at any stationary point $\vtheta$, $\beta$ satisfying $0 \in \partial f(\vtheta)$ fails to make ${\mPhi ^\top}{(\gamma \mP{\mPi_{\beta}} - \mI)^\top}\mD\mPhi $ invertible. Moreover, if $\mD=\mI$, then ${\mPhi ^\top}{(\gamma \mP{\mPi_{\beta}} - \mI)^\top}\mD\mPhi $ is always invertible, and hence, the fixed point of the P-BE always exists even if $\mGamma \gT$ is not a contraction. There may exist multiple fixed points of the P-BE.

In summary, if we can find a stationary point of~(\ref{eq:mspbe}), then we obtain a solution to the P-BE, which is referred to as the Bellman residual method~\citep{baird1995residual}. However, directly optimizing~(\ref{eq:mspbe}) is challenging because~(\ref{eq:mspbe}) is a nonconvex and nondifferentiable function; hence, one typically has to resort to subdifferential-based methods~\citep{clarke1981generalized}, which are often not computationally efficient. Moreover, when extending to model-free RL, a double-sampling issue~\citep{baird1995residual} arises. For these reasons, one often instead considers dynamic programming approaches~\citep{bertsekas2012dynamic} such as value iteration. For instance, we can consider the following projected value iteration (P-VI):
\begin{align}
  \mPhi {\vtheta _{k + 1}} = \mGamma \gT\mPhi {\vtheta _k} \label{eq:p-vi}  
\end{align}
which however is not guaranteed to converge unless $\mGamma \gT$ is a contraction. To mitigate these issues, in the next section we introduce RP-VI, which incorporates an additional regularization term.

%% file: contents/08_regularized_projection_operator.tex
\section{Regularized projection operator}

\begin{figure}[H]
    \centering
    \includegraphics[width=0.7\linewidth]{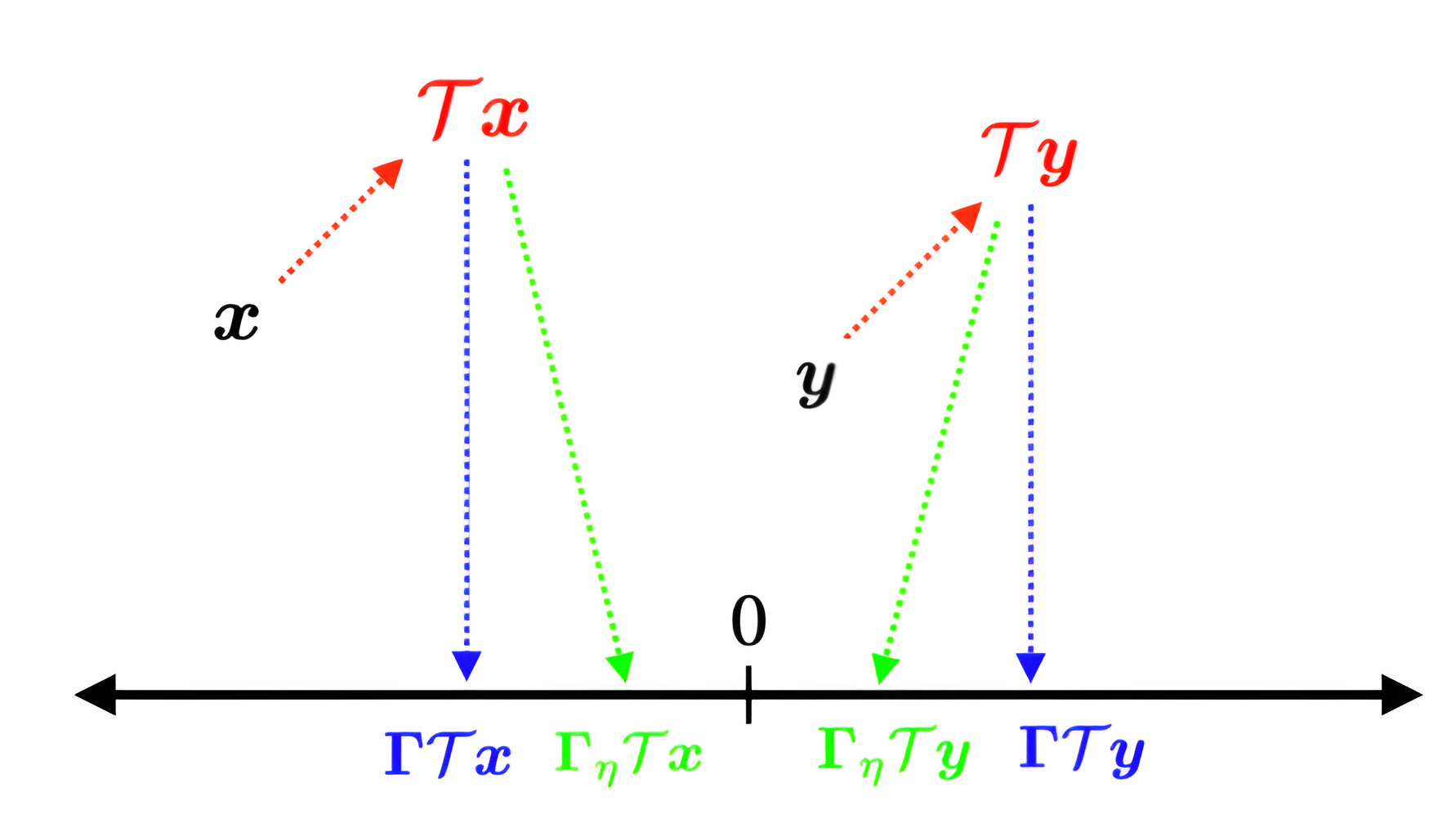}
    \caption{Illustration of the regularized projection. With a proper choice of $\eta$, $\mGamma_{\eta}\gT\vx$ and $\mGamma_{\eta}\gT\vy$ will be close to the origin and $||\mGamma_{\eta}\gT(\vx-\vy)||_2\leq||\mGamma_{}\gT(\vx-\vy)||_{2} $.}
    \label{fig:reg_projection}
\end{figure}

Let us begin with the standard P-VI in~(\ref{eq:p-vi}). P-VI can be equivalently written as the following optimization problem:
\begin{align}
{\vtheta _{k + 1}} = \arg {\min _{\vtheta  \in {\R^h}}}L(\vtheta ,{\vtheta _k}): = \frac{1}{2}\left\| {\mGamma \gT\mPhi {\vtheta _k} - \mPhi \vtheta } \right\|_{\mD}^2. \label{eq:target_loss}    
\end{align}
As mentioned before, this P-VI does not guarantee convergence unless $\mGamma \gT$ is a contraction. To address the potential ill-posedness of solving~(\ref{eq:mspbe}) and the projected Bellman equation (P-BE), we introduce an additional parameter vector $\vtheta^{\prime}$ (called target parameter) to approximate the next state-action value and a regularized formulation. In particular, we modify the objective function in~(\ref{eq:target_loss}) as follows:
\begin{align}
  L_\eta (\vtheta,\vtheta^{\prime}) \!=\! \frac{1}{2}\left\| \mGamma (\mR+\gamma\mP\mPi_{\mPhi\vtheta^{\prime}}\mPhi\vtheta^{\prime})\!-\!\mPhi\vtheta \right\|_{\mD}^2+\frac{\eta}{2}\left\|  \vtheta \right\|^2_2 \label{eq:reg-mspbe}
\end{align}
where $\eta\in [0,\infty)$ is a non-negative constant. The objective in~(\ref{eq:reg-mspbe}) differs from the original formulation in~(\ref{eq:mspbe}) in two key aspects. First, we separate the parameters for estimating the next state-action value and the current state-action value. Optimizing with respect to $\vtheta$ and considering $\vtheta^{\prime}$ as a fixed parameter, we can avoid the problem of non-differentiability from the max-operator in the original formulation in~(\ref{eq:mspbe}). Second, a quadratic regularization term is incorporated to ensure the contraction property of the regularized projection operator, thereby facilitating the convergence.

By taking the derivative of $L_{\eta}(\vtheta,\vtheta^{\prime})$ with respect to $\vtheta$, and using the first-order optimality condition for convex functions, we find that the minimizer of~(\ref{eq:reg-mspbe}) satisfies 
\begin{align}
    \mPhi^{\top}\mD\mR+\gamma\mPhi^{\top}\mD\mP\mPi_{\mPhi\vtheta^{\prime}}\mPhi\vtheta^{\prime}=(\mPhi^{\top}\mD\mPhi+\eta\mI)\vtheta \label{eq:target_rpbe}
\end{align}
Equivalently, multiplying both sides by $\mPhi(\mPhi^{\top}\mD\mPhi+\eta\mI)^{-1}$ yields
\begin{align*}
    \mPhi\vtheta= &\underbrace{\mPhi(\mPhi^{\top}\mD\mPhi+\eta\mI)^{-1}\mPhi^{\top}\mD}_{:=\mGamma_{\eta}}(\mR+\gamma \mP\mPi_{\mPhi\vtheta^{\prime}}\mPhi\vtheta^{\prime})\\ 
    \Leftrightarrow& {\mPhi \vtheta  = \mGamma_\eta \gT\mPhi \vtheta'}
\end{align*}
where $\mGamma_{\eta}$ is referred to as the regularized projection~\cite{limregularized}. We will discuss it in more detail soon. When $\vtheta$ and $\vtheta^{\prime}$ coincide, we recover a variant of P-BE in~(\ref{eq:pbe_1}) with an additional identity term, which corresponds to the RP-BE:
\[\mPhi \vtheta  = {\mGamma _\eta }\gT\mPhi \vtheta \]
which can be equivalently written as
\begin{align}
    \mPhi^{\top}\mD\mR+\gamma\mPhi^{\top}\mD\mP\mPi_{\mPhi\vtheta}\mPhi\vtheta=(\mPhi^{\top}\mD\mPhi+\eta\mI)\vtheta \label{eq:rpbe}
\end{align}
Let us denote the solution to~(\ref{eq:rpbe}) as $\vtheta^*_{\eta}$. Especially,~\citet{zhang2021breaking} consider a solution in a certain ball and~\citet{limregularized} choose a sufficiently large $\eta$ to guarantee the existence and uniqueness of the solution to the above equation in $\R^h$.

We can see that $\mGamma_{\eta}$ plays a central role in characterizing the existence of the solution to~(\ref{eq:rpbe}). Before proceeding further, let us first investigate the limiting behavior of the regularized projection operator:
\begin{lemma}\label{property-of-projection}[Lemma 3.1 in~\citet{limregularized}]
The matrix $\mGamma_\eta$ satisfies the following properties:
 $\mathop {\lim }\limits_{\eta  \to \infty } {\mGamma _\eta } \!=\! 0$ and $\mathop {\lim }\limits_{\eta  \to 0} {\mGamma _\eta } = \mGamma$.
\end{lemma}

In view of this limiting behavior, it follows that with sufficiently large $\eta$, the composition of the regularized projection operator and the Bellman operator becomes a contractive operator. Figure~\ref{fig:reg_projection} provides a geometric illustration of this effect. As $\eta$ increases, the image of $\mGamma_{\eta}$ is concentrated near the origin. Leveraging this observation, the following lemma characterizes conditions under which~(\ref{eq:rpbe}) admits a unique solution, for which the contractivity of the operator $\mGamma_{\eta}\gT(\cdot)$ is sufficient.

\begin{lemma}\label{lem:mGamma<1}[Lemma 3.2 in~\citet{limregularized}]
    The solution of~(\ref{eq:rpbe}) exists and is unique if
    \begin{align}
        \gamma  \left\| \mGamma_{\eta} \mP \right\|_{\infty}<1. \label{ineq:contraction}
    \end{align}
\end{lemma}
\begin{remark}
    Note that this is only a sufficient condition but not a necessary condition for the existence and uniqueness of~(\ref{eq:rpbe}).
\end{remark}
\begin{remark}\label{lem:eta-condition-ineq:contraction}
        If $\eta>2$ and $\left\|\mPhi\right\|_{\infty}\leq 1$, then~(\ref{ineq:contraction}) is satisfied. The proof is given in Appendix~\ref{subsec:lem:eta-condition-ineq:contraction}. If each element of $\mPhi$ is uniformly sampled from $[0,1]$, then only $\frac{1}{h}$ scaling is sufficient to ensure the condition $\left\|\mPhi\right\|_{\infty}\leq 1$.
\end{remark}

%% file: contents/03_AVI_reg.tex
\section{Regularized projected value iteration}
\label{sec:RAVI}

In this section, we present a theoretical analysis of the behavior of RP-VI, the regularized version of projected value iteration designed to solve~(\ref{eq:rpbe}). While this approach relies on knowledge of the model and reward, it serves as a foundational step toward the development of practical algorithms, which will be discussed in a later section. The RP-VI algorithm for solving~(\ref{eq:rpbe}) is given by
\begin{align}\label{eq:pr_vi}
\mPhi {\vtheta _{k + 1}} = \mGamma_\eta \gT\mPhi {\vtheta _k}    
\end{align}
or equivalently, it can be written as, for $\vtheta_0\in\mathbb{R}^h$,
\begin{equation}
\label{eq:avi}
\begin{split}
\vtheta_{k+1} &\!=\! (\mPhi^{\top}\mD\mPhi+\eta\mI)^{-1}\mPhi^{\top}\mD (\mR+\gamma\mP\mPi_{\mPhi\vtheta_k}{\mPhi\vtheta}_k).
\end{split}
\end{equation}
Note that~\cref{eq:avi} can be expressed as 
\begin{align}
{\vtheta _{k + 1}} = \arg {\min _{\vtheta  \in {\R^h}}}L_{\eta}(\vtheta ,{\vtheta _k}) \label{eq:target_loss_reg}    
\end{align}

which differs from~(\ref{eq:target_loss}) by replacing $L(\cdot,\cdot)$ with $L_{\eta}(\cdot,\cdot)$. This reformulation will be key to our subsequent development of the model-free version of this approach. The convergence of the above update can be characterized as follows: 

\begin{lemma}\label{lem:dp-contraction}
  Suppose that there exists a unique solution $\vtheta^*_{\eta}$ to~(\ref{eq:rpbe}), and consider the update in~(\ref{eq:avi}). We have
    \begin{align*}
        \left\| \mPhi(\vtheta_k-\vtheta^*_{\eta}) \right\|_{\infty}\leq  \left( \gamma \left\| \mGamma_{\eta}\right\|_{\infty}  \right)^{k+1} \left\|\mPhi\vtheta_0-\mPhi\vtheta^*_{\eta}\right\|_{\infty}.
    \end{align*}
\end{lemma}
The proof is given in Appendix~\ref{subsec:lem:dp-contraction}. From the above lemma, if $\gamma \left\|\mGamma_{\eta}\right\|_{\infty}<1$, then $\mPhi\vtheta_k\to\mPhi\vtheta_\eta^*$ at the rate of $\gamma \left\|\mGamma_{\eta}\right\|_{\infty}$.

%% file: contents/04_PRQ.tex
\section{Periodic regularized Q-learning}

In this section, we present PRQ, our main algorithmic contribution. Conceptually, PRQ can be seen as a stochastic version of RP-VI in~(\ref{eq:avi}). The idea of PRQ is to approximate the RP-VI update in~(\ref{eq:avi}), which cannot be implemented directly in a model-free setting due to the matrix inverse and the requirement for knowledge of system parameters. The key idea for implementing RP-VI in a model-free RL setting is that RP-VI can be reformulated in the optimization form in~(\ref{eq:target_loss_reg}). The optimization in~(\ref{eq:target_loss_reg}) can be solved to an arbitrarily accurate approximate solution via the stochastic gradient descent method. Therefore, we can develop an efficient algorithm based on stochastic gradient descent. The algorithm operates in two stages: the inner loop and the outer loop update. Each loop updates separate learning parameters, the inner loop iterate and the outer loop iterate, respectively. The inner loop involves a stochastic gradient descent method applied to a loss function, while the outer loop update adjusts the second argument in the objective function in~(\ref{eq:target_loss_reg}), which is referred to as the target parameter.

The overall algorithm is summarized in ~\cref{algo:prq}. Let $\vtheta_{t,k}$ denote the parameter vector at the $k$-th step of the inner loop during the $t$-th outer iteration. The objective of the inner loop is to approximate the update in~(\ref{eq:avi}) given $\vtheta_{t,0}$. Specifically, the inner loop aims to approximately solve the optimization problem $\min_{\vtheta\in \R^h} L_{\eta}(\vtheta, \vtheta_{t,0})$; accordingly, after $K$ steps of inner iterations,
\begin{align}
    \vtheta_{t,K} \approx \vtheta^*(\vtheta_{t,0}):= \argmin_{\vtheta\in\R^h} L_{\eta}(\vtheta, \vtheta_{t,0})\label{def:theta^*}
\end{align}
where we define a function $\vtheta^*:\R^h\to\R^h$ for the simplicity of the notation. The stochastic gradient descent method to solve the inner loop minimization problem can be applied in the following manner: upon observing $o=(s,a,s^{\prime})\in\gS\times\gA\times\gS$, where $(s,a)\sim d(\cdot,\cdot), \; s^{\prime}\sim\gP(\cdot\mid s,a)$, we construct the stochastic gradient estimator 
\begin{equation*}
\begin{split}
g(\boldsymbol{\theta}, \boldsymbol{\theta}^{\prime}; o) 
&= -\Bigl( r(s,a,s^{\prime}) + \gamma \max_{a \in \mathcal{A}} \boldsymbol{\phi}(s^{\prime},a)^{\top} \boldsymbol{\theta}^{\prime} \\
&\quad - \boldsymbol{\phi}(s,a)^{\top} \boldsymbol{\theta} \Bigr) \boldsymbol{\phi}(s,a) + \eta \boldsymbol{\theta}
\end{split}
\end{equation*}
which satisfies $\E[g (\vtheta,\vtheta^{\prime};o)]=\nabla_{\vtheta} L_{\eta}(\vtheta,\vtheta^{\prime})$. Therefore, given a step-size $\alpha\in(0,1)$, the inner loop update can be written as follows:
\begin{align}
    \vtheta_{t,k+1}  = \vtheta_{t,k} + \alpha \left( - g (\vtheta_{t,k},\vtheta_{t,0};o)\right). \label{eq:prq-update}
\end{align}

\begin{algorithm}[t]

\caption{Periodic Regularized Q-learning}\label{algo:prq}
\begin{algorithmic}[1]
    \STATE \textbf{Input:} total periods $T$, period length $K$
    \STATE \textbf{Output:} $\vtheta_{T,K}$
    \STATE Initialize $\vtheta_{0,K}\in\R^h$
    \FOR{$t=1$ to $T$}
        \STATE $\vtheta_{t,0}\gets \vtheta_{t-1,K}$
        \FOR{$k=0$ to $K-1$}
            \STATE Observe $(s_{t,k},a_{t,k})\! \sim\! d(\cdot,\cdot)$ \\
            \STATE Observe $s'_{t,k}\!\sim\! \gP(\cdot \!\mid s_{t,k},a_{t,k})$
            \STATE Receive $r_{t,k}\gets r(s_{t,k},a_{t,k},s'_{t,k})$
            \STATE Update $\vtheta_{t,k+1}$ using~(\ref{eq:prq-update})
        \ENDFOR
    \ENDFOR
\end{algorithmic}
\end{algorithm}
After $K$ steps in the inner loop update,  we update the target parameter $\vtheta_{t+1,0} \xleftarrow{}\vtheta_{t,K}$ and then repeat the inner loop procedure. This combined process is an approximation of RP-VI in~(\ref{eq:avi}), with stochastic gradient descent. Consequently, the period length $K$ plays a critical role in controlling approximation error; a sufficiently large $K$ ensures accurate regularized projection, thereby guaranteeing stability and convergence.

%% file: contents/05_main_result.tex
\section{Main theoretical result}

In this section, we present the theoretical analysis of PRQ. We first derive a loop error decomposition and present a key proposition. We then analyze convergence under the independent and identically distributed (i.i.d.) observation model and subsequently extend the results to the Markovian observation model.

\subsection{Outer loop decomposition}
Before proceeding to the error analysis, we establish a structural decomposition of the overall approximation error in the PRQ procedure. One component is the inner loop error, which arises from stochastic gradient descent on the regularized objective. The other component is the outer loop error, which is induced by the RP-VI update.

\begin{proposition}\label{prop:outer-loop-decomposition}
For $t\in\mathbb{N}$ and $\delta > 0$, we have
\begin{equation*}
\begin{aligned}
&\mathbb{E}\left[ \| \boldsymbol{\Phi}\boldsymbol{\theta}_{t,K}- \boldsymbol{\Phi}\boldsymbol{\theta}^*_{\eta} \|^2_{\infty} \right] \\
&\le \tfrac{2(1+\delta)}{\mu_{\eta}}
\mathbb{E} \Bigl[ L_{\eta}(\boldsymbol{\theta}_{t,K},\boldsymbol{\theta}_{t-1,K}) 
- L_{\eta}(\vtheta^*(\vtheta_{t-1,K}),\boldsymbol{\theta}_{t-1,K}) \Bigr] \\
&\quad + \gamma^2  \left\| \mGamma_\eta \right\|^2_\infty (1+\delta^{-1})
\mathbb{E} \left[ \| \boldsymbol{\Phi}\boldsymbol{\theta}_{t-1,K}-\boldsymbol{\Phi}\boldsymbol{\theta}^*_{\eta} \|^2_{\infty} \right].
\end{aligned}
\end{equation*}
\end{proposition}

\begin{remark}
The proof is provided in Appendix~\ref{app:sec:prop:outer-loop-decomposition}. The first term in the above proposition can be controlled via the inner loop update. The second term captures the contraction effect induced by the outer-loop update under the RP-VI scheme and decays at a rate governed by $\gamma \lVert \mGamma_{\eta} \rVert_{\infty}$. Here, $\mu_{\eta}$ represents the strong convexity constant of $L_{\eta}$, the explicit definition of which is provided in Lemma~\ref{lem:strongconvexandsmoothinmain}. 

The above result is independent of the observation model; in particular, it holds under both the i.i.d. and Markovian observation settings.
\end{remark}

\subsection{i.i.d.\ observation model}

In this section, we present our main theoretical result, showing that the proposed PRQ algorithm achieves an error bound of
$\E[\|\mPhi(\vtheta_{t,K}-\vtheta^*_{\eta})\|^2_{\infty}] \le \epsilon$
under appropriate choices of the step size, the number of inner iterations, and the number of outer updates. The proof follows a standard approach to the analysis of strongly-convex and smooth objectives in the optimization literature~\cite{bottou2018optimization}. 

\begin{lemma}[Strong convexity and smoothness of $L_{\eta}(\vtheta,\vtheta^{\prime})$]
For any fixed $\vtheta^{\prime}$, the function $L_{\eta}(\vtheta,\vtheta^{\prime})$ is
$\mu_{\eta}$-strongly convex and $l_{\eta}$-smooth with respect to $\vtheta$, where $\mu_{\eta} \coloneqq \lambda_{\min}\!\left(\mPhi^{\top}\mD\mPhi\right) + \eta, \;\;
l_{\eta} \coloneqq \lambda_{\max}\!\left(\mPhi^{\top}\mD\mPhi\right) + \eta.$\label{lem:strongconvexandsmoothinmain}

\end{lemma}
The detailed proofs are deferred to Lemma~\ref{lem:strong_convexity} and~\ref{lem:smooth} in the Appendix.

\begin{theorem}
\label{thm:epsilon-accuracy-main}
Suppose $\alpha\leq \min \left\{ \bar{\alpha}_1, \bar{\alpha}_2,\bar{\alpha}_3, \bar{\alpha}_4\right\}$, which are defined in Appendix~\ref{thm:epsilon-accuracy}. For  $ \E\left[ \left\|\mPhi(\vtheta_{t,K}-\vtheta^*_{\eta}) \right\|^2_{\infty} \right] \leq \eps$
to hold, we need at most the following number of iterations:
\[
K
\!=\!
\gO\!\left(
\frac{
l_\eta\,\|\vtheta^*_{\eta}\|^2_2
}{
\epsilon\mu_\eta^{3}
\bigl(1-\gamma\|\mGamma_\eta\|_\infty\bigr)^{2}
}
\right),\quad
t
\!=\!
\gO\!\left(
\frac{1}{
1-\gamma\|\mGamma_\eta\|_\infty
}
\right).
\]
\end{theorem}

The detailed proof of Theorem~\ref{thm:epsilon-accuracy-main} is deferred to Appendix~\ref{thm:epsilon-accuracy}. Table~\ref{tab:onecol} situates our contribution within the literature on Q-learning with target network updates. Early work by \citet{lee2019target} establishes non-asymptotic convergence guarantees, but the analysis is restricted to the tabular setting. Subsequent studies extend the scope to function approximation. In particular, \citet{zhang2021breaking} considers linear function approximation and ensures asymptotic convergence through projection and regularization. \citet{chen2023target} derives non-asymptotic guarantees under linear function approximation by introducing truncation, but convergence is only shown to a bounded set rather than a single point. More recently, \citet{zhang2023convergence} establishes non-asymptotic point convergence for neural network approximation, albeit under restrictive local convexity assumptions. In contrast, our work provides non-asymptotic convergence guarantees under linear function approximation using a single regularization mechanism. This unifies and strengthens existing results by simultaneously achieving finite-time guarantees, non-asymptotic convergence, and broad applicability, without relying on truncation, projection, or strong local convexity assumptions.

Now, let us briefly discuss the sample complexity result. From Theorem~\ref{thm:epsilon-accuracy-main}, the total sample complexity is given by:
\[tK = \gO\!\left(
\frac{l_\eta\ \|\vtheta^*_{\eta}\|^2_2}{\epsilon\mu_\eta^{3}(1-\gamma\left\| \mGamma_\eta \right\|_\infty)^{3}}
\right).\]
Compared with~\citet{lee2020periodic}, which provides a sample complexity bound measured in terms of $\E\!\left[\|\hat{\mQ} - \mQ^*\|_{\infty}\right]$, our result is expressed in terms of the squared error $\E\!\left[\|\mPhi(\vtheta_{t,K}-\vtheta^*_{\eta})\|^2_{\infty}\right]$.
To ensure a fair comparison, we adjust the $\epsilon$–dependence in the complexity result of~\citet{lee2020periodic} accordingly, yielding an equivalent form of the bound
\[
\gO\!\left(
   \frac{|\gS|^3\,|\gA|^3}{\epsilon (1-\gamma)^4}
   \,\right).
\]

\noindent
Under the same measurement, our PRQ analysis in the tabular limit ($\eta\to0$, $\mD=\tfrac{1}{|\gS||\gA|}\mI, \mPhi=\mI$, $\left\| \mGamma_\eta \right\|_\infty \to 1$) yields
\[tK = \gO\!\left(
\frac{|\gS|^2|\gA|^2}{\epsilon(1-\gamma)^{5}}
\right),\]

\noindent
since $||\vtheta^*_{\eta}||=||\mQ^*|| =\gO(1/(1-\gamma))$. More generally, while~\cite{lee2020periodic} focuses on the tabular case, our framework allows linear function approximation.


\begin{table*}[t]
\centering
\caption{Comparison with existing works using Q-learning with target-based update. The symbol \cmark\  indicates that the corresponding attribute is present, whereas \xmark\  indicates its absence.}
\label{tab:onecol}
\begin{tabular}{c c c c c}
\hline
 & Non-asymptotic & Convergence result &  Function approximation & Modification\\
\hline
~\citet{lee2019target} & \cmark & point & tabular & \xmark \\
~\citet{zhang2021breaking} & \xmark & point & linear  & projection and regularization \\
~\citet{chen2023target} & \cmark & bounded set & linear  & truncation \\
~\citet{zhang2023convergence} & \cmark & point & neural network & local convexity\\ 
\rowcolor{lightyellow} Our work & \cmark  & point  & linear & regularization \\
\hline
\end{tabular}
\end{table*}

\subsection{Markovian observation model}\label{sec:markovian_obs_model}

In this subsection, we analyze the behavior of PRQ with a single trajectory generated under a fixed behavior policy $\beta$. We assume that the underlying Markov chain is irreducible. Consequently, for a finite state space, the chain admits a unique stationary distribution $\mu_{\infty}\in\Delta(\gS\times\gA)$ satisfying $\mu_{\infty}(s,a)=\sum_{(\Tilde{s},\Tilde{a})\in\gS\times\gA} P_{\beta}(s,a\mid \Tilde{s},\Tilde{a}) \mu_{\infty}(\Tilde{s},\Tilde{a})$ and $P_{\beta}(\cdot\mid \Tilde{s},\tilde{a})\in \Delta(\gS\times\gA)$ such that $P_{\beta}(s,a\mid\Tilde{s},\Tilde{a})=\beta(a\mid s)P(s\mid \tilde{s},\tilde{a})$. Let us denote the corresponding vector and matrix form of $\mu_{\infty}$ and $P_{\beta}$ as $\vmu_{\infty}$ and $\mP_{\beta}$, respectively. Given a stochastic process $\{(S_k,A_k)\}_{k=0}^{\infty}$ where $(S_k,A_k)$ are random variables induced by the Markov chain, we define the hitting time $\tau(\tilde{s},\tilde{a})=\inf\{ n \geq 1 : (S_n, A_n)=(\tilde{s},\tilde{a}) \}$ for some $(\tilde{s},\tilde{a})\in\gS\times\gA$, and denote $\tau_{\max}:=\max_{(s,a)\in\gS\times\gA}\tau(s,a)$.

Recently,~\citet{haque2024stochastic} utilized Poisson’s equation to analyze stochastic approximation schemes under the Markovian observation model. Building upon their approach and extending the i.i.d. model analysis presented in the previous section, we establish the following result, with the detailed proof provided in Appendix~\ref{app:thm:markovian_obs_model}.

\begin{theorem}\label{thm:markovian_obs_model}
Suppose $\alpha\leq \min \left\{ \bar{\alpha}_1, \bar{\alpha}_2,\bar{\alpha}_3, \bar{\alpha}_4 \right\}$ which are defined in~(\ref{constants:alpha}) in the Appendix. For  $\E\left[ \left\|\mPhi(\vtheta_{t,K}-\vtheta^*_{\eta}) \right\|^2_{\infty} \right] \leq \eps$ to hold, we need at most the following number of iterations:
    \begin{gather*}
    K \!=\! \mathcal{O} \left( \frac{(l_{\eta}+\kappa)\tau_{\max}\eta^2\|\boldsymbol{\theta^*}_{\eta}\|^2_2}{\mu_{\eta}^2(1-\gamma\|\boldsymbol{\Gamma}_{\eta}\|_{\infty})} \right), \; t \!=\! \mathcal{O} \left( \frac{1}{1-\gamma\|\boldsymbol{\Gamma}_{\eta}\|_{\infty}}\right)
\end{gather*}
where $\kappa=l_{\eta}/\mu_{\eta}.$
\end{theorem}

\begin{remark}
  In addition to the result of the i.i.d. analysis, we have an additional factor of the hitting time $\tau_{\max}$.


\end{remark}

\begin{algorithm}[H]
\caption{Periodic regularized Q-learning with Markovian observation model}
\label{algo:prq_markovian}
\begin{algorithmic}[1]
\STATE \textbf{Input:} total iterations $T$, period length $K$
\STATE \textbf{Output:} learned parameter $\vtheta_{T,K}$

\STATE Initialize $\vtheta_{0,0}\in\mathbb{R}^h$
\STATE Sample initial state $s_{0,K}$ from an arbitrary initial distribution over the state space
\FOR{$t = 1...T$}
    \STATE $\vtheta_{t,0} \gets \vtheta_{t-1,K}$ and $s_{t,0} \gets s_{t-1,K}$
    \FOR{$k = 0,...,K-1$}
        \STATE Sample $a_{t,k} \sim \beta(\cdot \mid s_{t,k})$
        \STATE Sample $s_{t,k+1} \sim \gP(\cdot \mid s_{t,k}, a_{t,k})$
        \STATE $r_{t,k} \gets r(s_{t,k}, a_{t,k}, s_{t,k+1})$
        \STATE Update $\vtheta_{t,k+1}$ using~(\ref{eq:prq-update})
    \ENDFOR
\ENDFOR
\end{algorithmic}
\end{algorithm}

%% file: contents/06_experiments.tex
\section{Experiments}
In this section, we investigate the behavioral differences between the proposed PRQ and regularized Q-learning (RegQ)~\cite{limregularized}, with a particular focus on the learning trajectories induced under linear function approximation. We consider an MDP that is deliberately chosen so that no solution exists for the P-BE in the unregularized setting. RegQ employs a direct semi-gradient update with $\ell_2$ regularization and does not incorporate any form of target-based or periodic update mechanism. In contrast, PRQ periodically resets the optimization target. Throughout this experiment, we observe that although both RegQ and PRQ can induce solutions to a RP-BE through the use of regularization, their resulting learning trajectories exhibit qualitatively different behaviors. The MDP considered in this experiment is summarized in the example below.

\begin{example}\label{example:1}
Consider the following MDP with $|\gS|=|\gA|=2$ and $h=2$:
\[
\begin{aligned}
\mPhi &\!=\! 
\begin{bmatrix}
0.25 & -0.81\\
0.88 & -0.92\\
1    & -0.93\\
0.03 & -0.19
\end{bmatrix}\!,\;\mP \!=\!
\begin{bmatrix}
0.90 & 0.10\\
0.94 & 0.06\\
0    & 1\\
0.44 & 0.56
\end{bmatrix}\!,\;
\mR \!=\!
\begin{bmatrix}
-0.63\\
0.24\\
0.50\\
0.92
\end{bmatrix}\!.
\end{aligned}
\]
\end{example}
Let $\beta(1\mid 1)=0.13$ and $\beta(1\mid 2)=0.63$. Then, no solution exists for P-BE, which is the case for $\eta=0$. Based on this MDP, we divide our experiments into two main settings: a model-based setting and a sample-based setting. In the model-based setting, full knowledge of the transition dynamics is assumed, allowing updates to be performed using the complete transition matrices without sampling. This setting serves to isolate the intrinsic algorithmic behavior of PRQ and RegQ. The sample-based setting is further divided into an i.i.d.\ sampling regime and a Markovian sampling regime. In the i.i.d.\ regime, state-action pairs are drawn independently from a fixed distribution, whereas in the Markovian regime, samples are generated sequentially along trajectories induced by the predefined policy.

\subsection{Model-based setting}
In a model-based simulation, sampling is skipped and updates are performed using the full transition matrices. For PRQ, this setting can be implemented straightforwardly by directly applying the update rule of RP-VI described in  Section~\ref{sec:RAVI}. In contrast, for RegQ, we reimplement the deterministic, model-based update equation following~\citet{limregularized}.
The resulting update can be expressed in matrix form as
\begin{align*}
    \vtheta_{k+1}
    = \vtheta_k
    + \alpha \mPhi^{\top} \mD
    \bigl(
        \mR
        + \gamma \mP \mPi_{{\mPhi\vtheta_k}} \mPhi \vtheta_k
        - \mPhi \vtheta_k
        - \eta \vtheta_k
    \bigr).
\end{align*}

When $\eta = 0$, the update reduces to the standard model-based Q-learning update under linear function approximation. For $\eta > 0$, the additional term $-\eta \vtheta_k$ acts as an $\ell_2$ regularizer, yielding the regularized Q-learning (RegQ) algorithm. For the MDP presented in Example~\ref{example:1}, we observe that with $\eta = 0.01$, only the model-based version of PRQ in~(\ref{eq:avi}) converges, whereas RegQ exhibits persistent oscillations and fails to converge, as shown in~\cref{fig:model-exp}. Importantly, the RP-BE admits a unique solution in this setting. However, despite the existence and uniqueness of the solution, RegQ fails to converge to it, while PRQ follows a stable and efficient trajectory in the two-dimensional parameter space and successfully converges.

\begin{figure}[t]
    
    \centering
    
    \begin{subfigure}[t]{\linewidth}
        \centering
        \includegraphics[height=3cm]{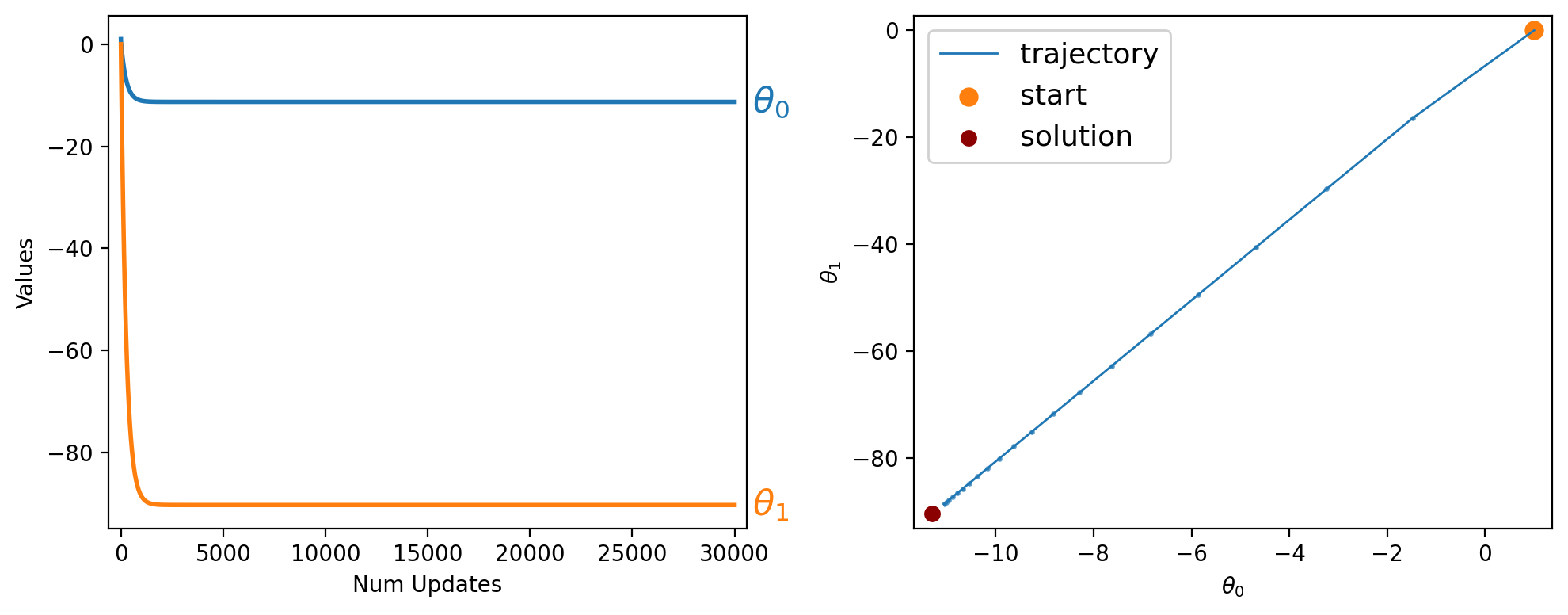}
        
    \end{subfigure}

    \vspace{0.8em} 

    \begin{subfigure}[t]{\linewidth}
        \centering
        \includegraphics[height=3cm]{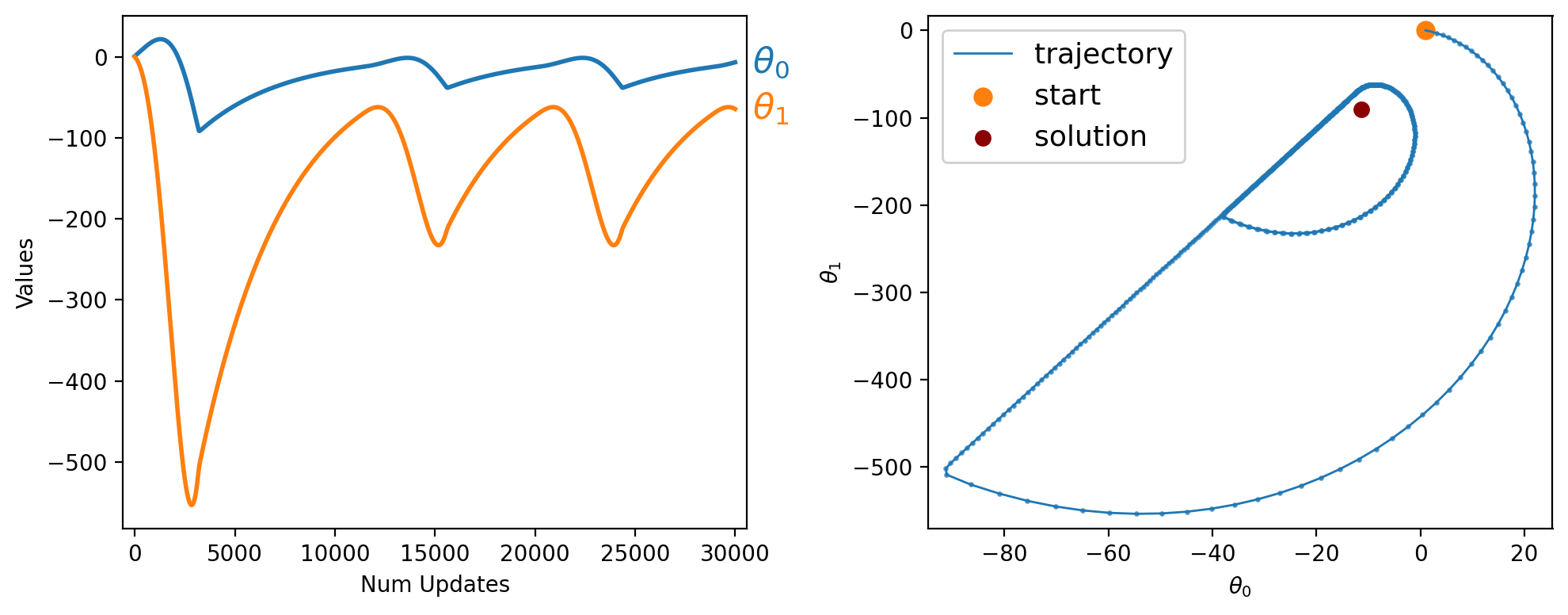}
    \end{subfigure}
    \caption{Comparison of PRQ and RegQ in the model-based setting with $\eta=0.01$. The top row corresponds to PRQ, while the bottom row corresponds to RegQ. In each row, the left subplot shows the temporal evolution of the parameters $\theta_0$ and $\theta_1$ during iterations, and the right subplot shows the corresponding trajectory in the two-dimensional $(\theta_0, \theta_1)$ parameter space, including the initialization point and the RP-BE solution. PRQ exhibits stable convergence toward the solution, whereas RegQ displays periodic behavior and fails to show convergence.}

    \label{fig:model-exp}
\end{figure} 

\subsection{Sample-based setting}
Beyond the model-based setting, which requires full knowledge of the transition dynamics, the sample-based setting assumes that the agent has access only to a single transition sample at each step. In the sample-based setting, the sampling scheme may vary depending on whether the underlying probability distribution is i.i.d.\ or Markovian. Under the i.i.d.\ setting, PRQ is applied directly using the sampling procedure in~\cref{algo:prq}, while RegQ follows the update rule of~\citet{limregularized}. Despite the additional variance induced by stochastic sampling, convergence of both algorithms in the i.i.d.\ setting is theoretically guaranteed if $\eta$ is sufficiently large: convergence for PRQ is established in this paper, and for RegQ in~\citet{limregularized}. For the Markovian setting, the algorithmic structure remains unchanged and only the sampling procedure differs: trajectories are generated by rolling out the transition dynamics under a behavior policy~$\beta$, as in Example~\ref{example:1}. In the Markovian setting, PRQ admits a finite-time convergence guarantee if $\eta$ is sufficiently large (\cref{thm:markovian_obs_model}), whereas no such guarantee is available for RegQ. The experimental results are presented in~\cref{fig:iid-exp} and~\cref{fig:markov-exp}. Despite sharing the same theoretical solution defined by~(\ref{eq:rpbe}), the two algorithms display distinct convergence properties. In particular, PRQ demonstrates a stochastic yet consistent and efficient trajectory toward the solution, remaining in a small neighborhood once it converges. In contrast, RegQ exhibits extreme oscillations in both $\theta_0$ and $\theta_1$, and its trajectory forms large periodic excursions in the parameter space. More specifically, although the RegQ trajectory may occasionally pass near the solution point, it shows a weak tendency to remain in its neighborhood.

\begin{figure}[H]
    \centering

    \begin{subfigure}[t]{\linewidth}
        \centering
        \includegraphics[height=3cm]{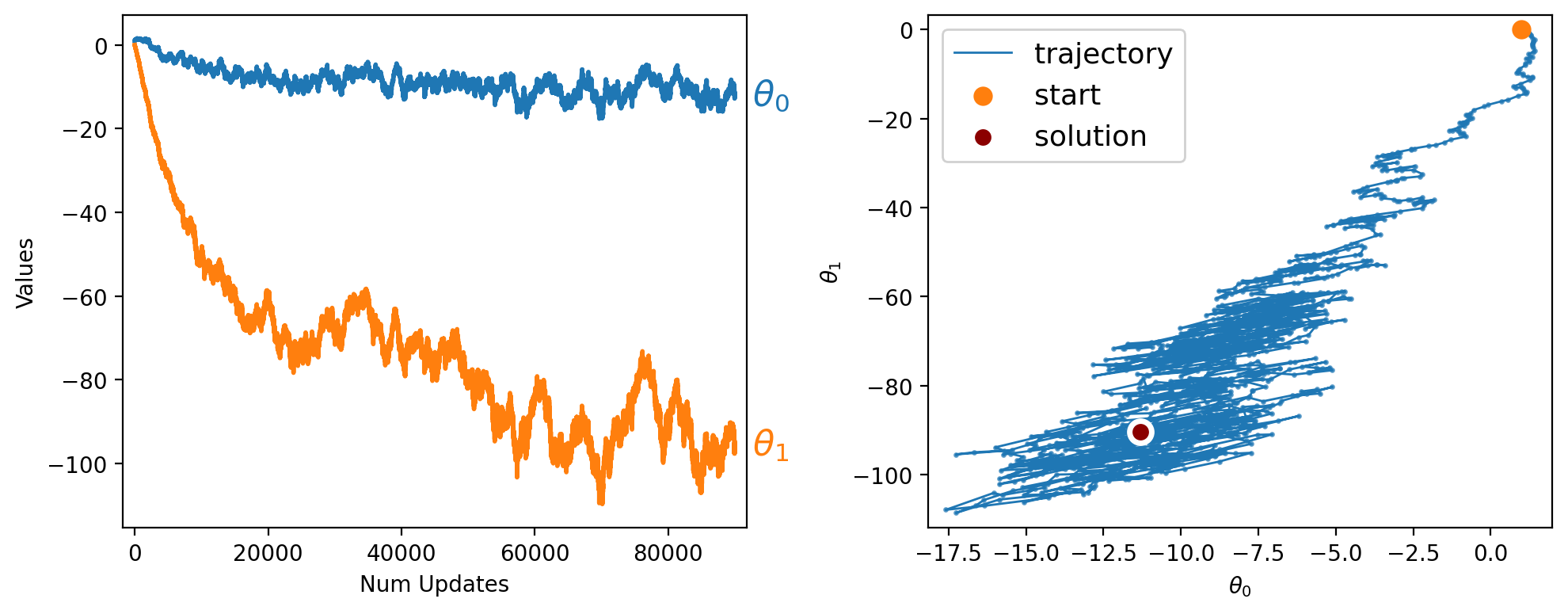}
    \end{subfigure}

    \vspace{0.8em}

    \begin{subfigure}[t]{\linewidth}
        \centering
        \includegraphics[height=3cm]{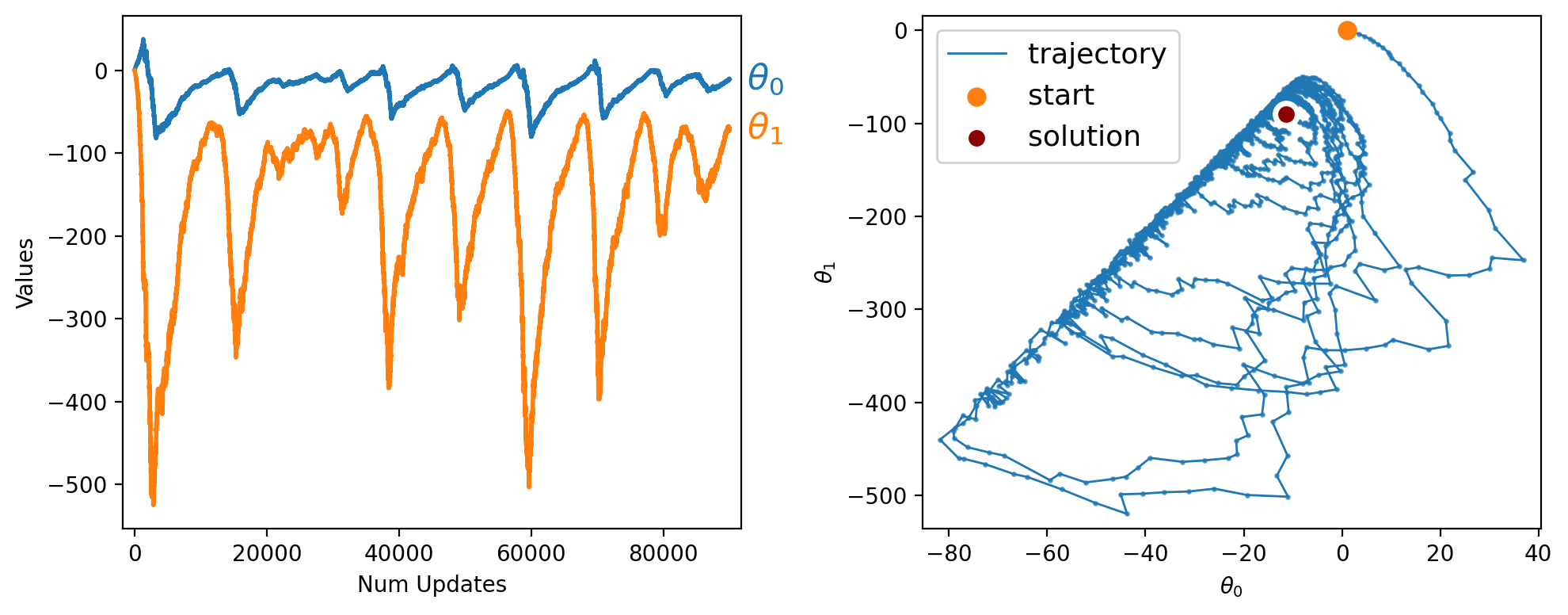}
        
    \end{subfigure}
    \caption{Comparison of PRQ and RegQ in the i.i.d. sample-based setting with $\eta=0.01$. The figure follows the same layout as \cref{fig:model-exp}.}
        \label{fig:iid-exp}
\end{figure}

\begin{figure}[H]
    \centering

    \begin{subfigure}[t]{\linewidth}
        \centering
        \includegraphics[height=3cm]{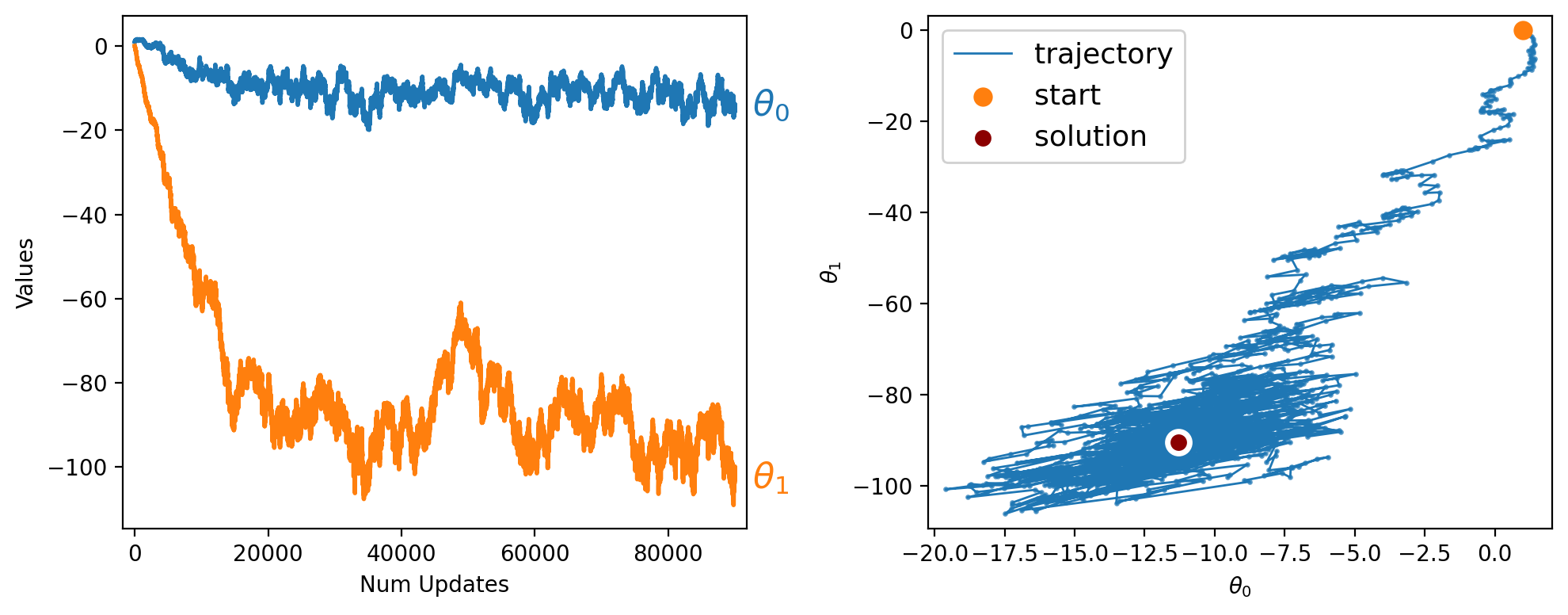}

    \end{subfigure}

    \vspace{0.8em}

    \begin{subfigure}[t]{\linewidth}
        \centering
        \includegraphics[height=3cm]{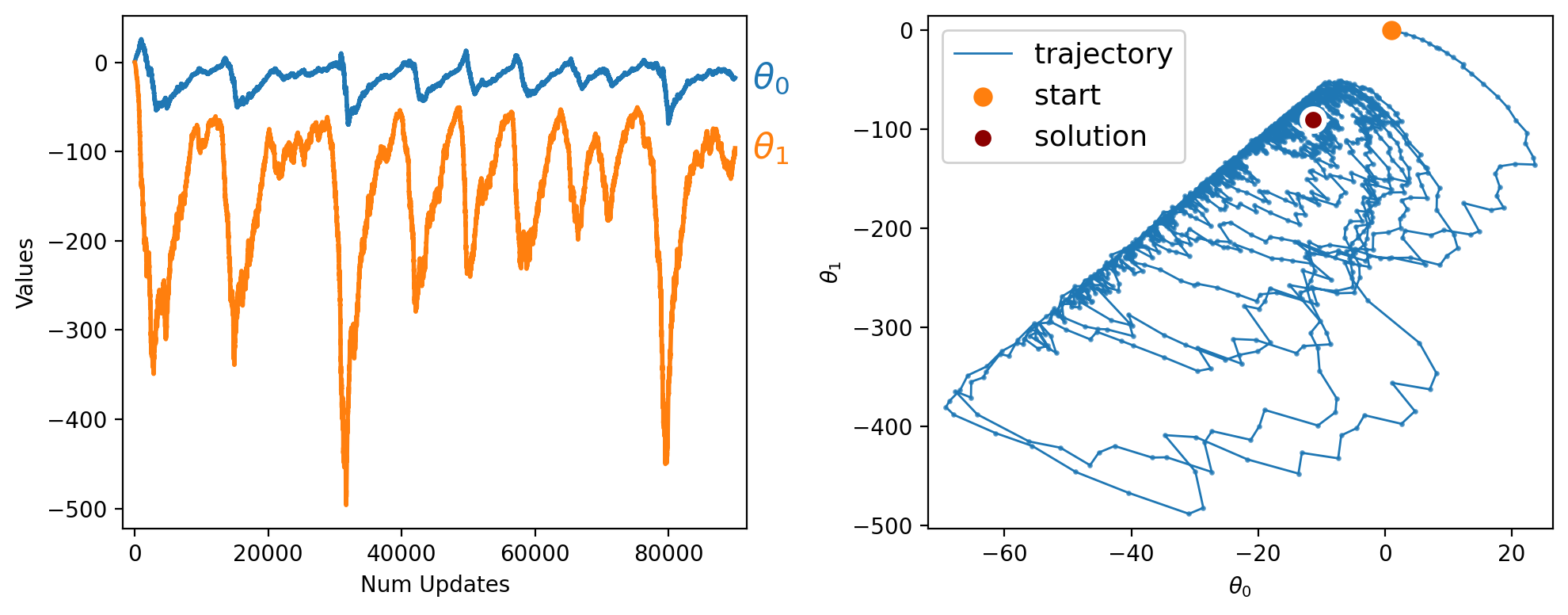}
        
    \end{subfigure}
    \caption{Comparison of PRQ and RegQ under the Markovian sample-based setting with $\eta=0.01$. The figure follows the same layout as \cref{fig:model-exp}.}
    \label{fig:markov-exp}
\end{figure}

%% file: contents/07_conclusion.tex
\section{Conclusion}

In this paper, we theoretically study a regularized projection operator and its contraction property. Building on this analysis, we introduce an RP-VI algorithm and its sample-based extension, PRQ, which features an inner–outer loop structure consisting of an inner convex optimization step and an outer value iteration. Our main theoretical result establishes finite-time, non-asymptotic convergence of PRQ under both i.i.d.\ and Markovian sampling settings. Through empirical evaluations, we demonstrate that both the regularization mechanism and the periodic structure are essential for achieving stable training and convergence in practice.

\newpage

%% file: contents/XX_appendix.tex
\begin{center}
{\Large \bf Appendices}
\end{center}

\section{Notations}

$\R$: set of real numbers; $\R^h$ : set of $h$-dimensional real-valued vectors; $\R^{m\times n}$ : set of $m\times n$ dimensional matrices; $ \mA \preceq \mB $ for $\mA,\mB\in \R^{h\times h}$: $\mB-\mA$ is a positive semi-definite matrix; $[\mA]_{ij}$ for $\mA\in\R^{m\times n}$, $1\leq i\leq m$ and $1\leq j \leq n$ : $i$-th row and $j$-th column element of matrix $\mA$; $[\vv]_i$ for $\vv\in\R^h$ and $1\leq i \leq h$: $i$-th element of $h$-dimensional vector $\vv$; $\left\|\vv\right\|_{\infty}$ for $\vv\in\R^h$ : infinity norm of a vector, i.e., $\max_{i\in[h]}|[\vv]_i|$; $\left\|\mA\right\|_{\infty}$ for $\mA\in\R^{h\times n}$ : infinity norm of a matrix, i.e., $\left\|\mA\right\|_{\infty}=\max_{1\leq i\leq h}\sum_{j=1}^n|[\mA]_{ij}|$.
Moreover, for notational simplicity, we use $\mPi_{\vtheta}$ and $\mPi_{\mPhi\vtheta}$ interchangeably to denote the greedy policy with respect to the value function $\mPhi\vtheta$.

\section{Organization}

The Appendix is organized as follows.
\begin{enumerate}[
  labelsep=0.5em,
  align=left,
  leftmargin=5em,
  widest=Section~\ref{app:sec:omitted_proof}:
]
\item[Section~\ref{app:aux_prelim}:] Auxiliary preliminaries on differential and optimization methods.
  \item[Section~\ref{app:sec:constants}:] Summary of constants used throughout the paper.
  \item[Section~\ref{app:sec:omitted_proof}:] Proofs omitted from the main text.
  \item[Section~\ref{app:sec:property_loss_ftn}:] Properties on the loss function. The derived properties will be used in both the analysis of i.i.d. and Markovian observation model.
  \item[Section~\ref{app:iidproof}:] Proof for i.i.d. observation model.
  \item[Section~\ref{app:sec:markov}:] Proof for Markovian observation model.
\end{enumerate}

\section{Auxiliary preliminaries}\label{app:aux_prelim}

\subsection{Differential methods}
\begin{definition}[Locally Lipschitz function]
 A function $\varphi:\R^h\to\R$ is said to be locally Lipschitz if for a bounded subset $B\subset\R^h$, there exists a positive real number $K$ such that
 \begin{align*}
     |\varphi(\vx_1)-\varphi(\vx_2)|\leq K ||\vx_1-\vx_2||_2,\quad \forall \vx_1,\vx_2\in B.
 \end{align*}
\end{definition}

\begin{definition}[Generalized directional derivative~\cite{clarke1981generalized}]
Let $\varphi:\R^h \to \R$. The generalized directional derivative of $\varphi$ at $\vx\in\R^h$ in direction $\vv\in\R^h$, denoted $\varphi^{\circ}(\vx;\vv)$ is given by 
\begin{align*}
    \varphi^{\circ}(\vx;\vv) =\limsup_{\substack{\vy \to \vx \\ \lambda \downarrow 0}} \frac{\varphi(\vy+\lambda \vv)-\varphi(\vy)}{\lambda}
\end{align*}
\end{definition}

\begin{definition}[Generalized gradient~\cite{clarke1976new}]
Consider a locally Lipschitz function $\varphi:\R^h\to\R$. The generalized gradient of $\varphi$ at $\vx$, denoted $\partial \varphi(\vx)$ is defined to be the subdifferential of the convex function $\varphi^{\circ}(\vx;\cdot)$ at $0$. Thus, an element $\vxi$ of $\R^h$ belongs to $\partial \varphi(\vx)$ if and only if for all $\vv\in\R^h$, 
\begin{align*}
    \varphi^{\circ}(\vx;\vv) \geq \vv^{\top}\vxi.
\end{align*}
\end{definition}

\begin{lemma}[Proposition 1.4 in~\cite{clarke1975generalized}]\label{lem:clarke_convex_hull}
Suppose $\varphi:\R^h\to\R$ is a locally Lipschitz function. Then, the following holds:
\begin{align}
 \partial \varphi(\vx)= \mathrm{conv}\left( \left\{ \lim_{k\to\infty} \nabla \varphi(\vv_k) :  \{\vv_{k}\}_{k=0}^{\infty} \;\text{such that $\vv_{k}\to\vx$, each $\varphi(\vv_k)$ is differentiable and $\lim_{k\to\infty} \nabla \varphi(\vv_k)$ exists.} \right\} \right). \label{subdifferential_convex_hull}   
\end{align}
\end{lemma}

\begin{lemma}\label{lem:subdifferential_f_cvx_hull_derivation}
Consider the function $f$ in~(\ref{eq:mspbe}). The subdifferential of $f$ at $\vtheta\in\R^h$ can be expressed as
\begin{align*}
    \partial \! f(\vtheta)\!  \subseteq & \mathrm{conv}\{ \!(\gamma\mP\mPi_{\beta}\mPhi \!-\!\mPhi)^{\top}\!\mD\mGamma(\gT\mPhi\vtheta \!-\!\mPhi\vtheta) \! \mid \! \beta \!\in\! \Lambda(\vtheta)\! \}
\end{align*}
\end{lemma}
\begin{proof}
Let us check that $f(\vtheta)$ is a locally Lipschitz function to apply Lemma~\ref{lem:clarke_convex_hull}. Observe that the function $f(\vtheta)$ can be written as a composition of weighted squared norm $||\cdot||^2_{\mD}$ and the map $\vtheta\mapsto\mGamma\gT\mPhi\vtheta-\mPhi\vtheta$. Both functions are Lipschitz, and therefore the objective function $f(\vtheta)$ becomes a locally Lipschitz function. Now, we can express the subdifferential of $f(\vtheta)$ as a convex hull of gradients as in~(\ref{subdifferential_convex_hull}). The possible choice of sequences $\{\vv_k\}_{k=0}^{\infty}$ such that $\vv_k\to\vtheta$ and $\lim_{k\to\infty} \nabla  f(\vv_k)$ exists is to choose $\vv_k\in S_{\beta}$ where $S_{\beta}=\{ \vx \in \R^h : |\arg\max_{a\in \gA} \vphi(s,a)^{\top} \vx|=1 , \quad \beta(s) = \arg\max_{a\in \gA} \vphi(s,a)^{\top} \vx  \}$ for $k\geq N$ and some $N\in\sN$ and $\vv_k\neq \vtheta$. The result follows by applying the chain rule at points of differentiability of the Lipschitz function. Since a Lipschitz function is differentiable almost everywhere, the set of points where the derivative fails to exist has Lebesgue measure zero and can therefore be excluded~\cite{clarke1975generalized}.
\end{proof}

\subsection{Optimization methods}\label{app:sec:optimization}

\begin{definition}[\cite{nesterov2018lectures}]\label{def:strongly_cvx_smooth}
The continuously differentiable function $\varphi:\R^h\to\R$ is $\mu$-strongly convex if there exists a constant $\mu>0$ such that
\begin{align*}
    \varphi(\vtheta^{\prime}) \geq \varphi(\vtheta)+\nabla \varphi(\vtheta)^{\top}(\vtheta^{\prime}-\vtheta)+\frac{\mu}{2}\|\vtheta-\vtheta^{\prime}\|^2_2.
\end{align*}
$\varphi$ is said to be $l$-smooth if 
\begin{align*}
    \left\|\nabla \varphi(\vtheta)-\nabla \varphi(\vtheta^{\prime}) \right\|_2 \leq l \left\|\vtheta-\vtheta^{\prime} \right\|_2.
\end{align*}
\end{definition}

For a twice continuously differentiable function $\varphi$ that is
$\mu$-strongly convex and $l$-smooth, the Hessian satisfies
\begin{align*}
\mu \mI \;\preceq\; \nabla^2 \varphi(\vtheta) \;\preceq\; l \mI,
\quad \forall \vtheta \in \R^h,    
\end{align*}
and consequently, all eigenvalues of $\nabla^2 \varphi(\vtheta)$ are lower bounded
by $\mu$ and upper bounded by $l$.

\section{Constants used throughout the proof}\label{app:sec:constants}

Before proceeding, we introduce several constants to simplify the notation:

    \begin{align}
        l_{V_1} :=& 2\tau_{\max}(1+\eta),\quad l_{V_2}:=\max\{2\tau_{\max} \gamma,1\}, \quad l_{V_3}:=  \tau_{\max} \left( R_{\max}+(1+\gamma+\eta)\left\| \vtheta^*_{\eta} \right\|_2 \right), \label{constant:l_V}\\
            D_1:=& \left( l_{\eta}(6l_{V_1}+l_{V_3})+\kappa l_{V_1}\gamma \right) ,\quad D_2 :=  \kappa \left(4l_{V_1}(1+l_{\eta})  \right), \quad D_3  :=\kappa l_{V_1}+l_{\eta}l_{V_2}. \label{constant:D}\\
    g_{1,\eta} :=& 16+16\eta, \quad g_{2,\eta} :=  (42 +32\eta)\gamma^2 ,\quad g_{3,\eta} :=  32(1+\eta)\gamma^2\left\|\mPhi\vtheta^*_{\eta} \right\|^2_{\infty}+(16+16\eta)R_{\max}^2+8\sigma_{\eta}^2, \label{constant:g}\\
        \gE_1 := &     \left( D_1+\frac{l_{\eta}}{2} \right) g_{2,\eta} + D_3, \quad 
        \gE_2 :=  \left( D_1+\frac{l_{\eta}}{2} \right) g_{3,\eta}+ 2  l_{\eta}l_{V_3} \label{constant:gE}\\
        \sigma_{\eta}^2:=& \max_{(s,a)\in\gS\times\gA} \left( \E\left[\left\| \left(r(s,a,s^{\prime})+\gamma\max_{u\in\gA}  \vphi(s^{\prime},u)^{\top}\vtheta^*_{\eta} -\E\left[(r(s,a,\tilde{s})+\gamma \max_{u\in\gA}\vphi(\tilde{s},u  )^{\top}\vtheta^*_{\eta})\right]\right)\vphi(s,a) \right\|_2^2 \middle|s,a \right] \right) \label{def:sigma_eta}
    \end{align}
The constants introduced in~(\ref{constant:l_V}) are utilized in Lemma~\ref{lem:poisson-property}, whereas those defined in~(\ref{constant:D}) appear in Lemma~\ref{lem:cross-term-poisson}. The constant specified in~(\ref{constant:g}) and~(\ref{def:sigma_eta}) are used in Lemma~\ref{lem:g-bound}, and the constants in~(\ref{constant:gE}) are employed in Proposition~\ref{prop:descent-lemma-final} in the Appendix.

\section{Omitted proofs in the main manuscript}\label{app:sec:omitted_proof}

\subsection{Proof of Remark~\ref{lem:eta-condition-ineq:contraction}}\label{subsec:lem:eta-condition-ineq:contraction}

\begin{lemma}[Lemma 3.3 in~\citet{limregularized}]\label{lem:eta-projection-condition}
 For $ \eta > \gamma \|\mPhi^{\top} \mD\|_{\infty}\|\mPhi\|_{\infty}+\|\mPhi^{\top}\mD\mPhi\|_{\infty} $, we have $\gamma\left\|\mGamma_{\eta}\right\|_{\infty}<1$.
\end{lemma}

\begin{proof}
If $\|\mPhi\|_{\infty}\leq 1$, then $\left\|\vphi(s,a)\right\|_{\infty}\leq 1$ for all $(s,a)\in\gS\times\gA$. Then,
\begin{align*}
    \left\|\mPhi^{\top}\mD\right\|_{\infty} = \left\| \begin{bmatrix}
       d(1,1) \vphi(1,1) &  \cdots  &d(|\gS|,|\gA|) \vphi(|\gS|,|\gA|) 
    \end{bmatrix} \right\|_{\infty} =\max_{i\in[h]} \sum_{(s,a)\in\gS\times\gA}d(s,a)|[\vphi(s,a)]_i| \leq 1.
\end{align*}    
Therefore, from Lemma~\ref{lem:eta-projection-condition} in the Appendix, $\eta>2$ is enough.
\end{proof}

\subsection{Proof of Lemma~\ref{lem:dp-contraction}}\label{subsec:lem:dp-contraction}

\begin{proof}
We have 
\begin{align*}
    \mPhi\vtheta_{k+1} =& \mGamma_{\eta} \gT \mPhi\vtheta_k =  \mGamma_{\eta} (\mR+\gamma \mP\mPi_{\mPhi\vtheta_k}\mPhi\vtheta_k). 
\end{align*}
The above equation can be re-written noting that $\vtheta^*_{\eta}$ is the solution of~(\ref{eq:rpbe}):
\begin{align*}
    \mPhi\vtheta_{k+1}-\mPhi\vtheta^*_{\eta} = \gamma\mGamma_{\eta}\mP(\mPi_{\mPhi\vtheta_k}\mPhi\vtheta_k-\mPi_{\mPhi\vtheta^*_{\eta}}\mPhi\vtheta^*_{\eta}) 
\end{align*}
Taking the infinity norm on both sides,
\begin{align*}
    \left\| \mPhi\vtheta_{k+1}-\mPhi\vtheta^*_{\eta} \right\|_{\infty} \leq & \gamma \left\| \mGamma_{\eta}\mP\right\|_{\infty} \left\| \mPi_{\vtheta_k}\mPhi\vtheta_k-\mPi_{\vtheta^*_{\eta}}\mPhi\vtheta^*_{\eta} \right\|_{\infty}\\
    \leq & \gamma \left\| \mGamma_{\eta}\mP\right\|_{\infty} \left\|\mPhi\vtheta_k-\mPhi\vtheta^*_{\eta}\right\|_{\infty}\\
    \leq & \left( \gamma \left\| \mGamma_{\eta}\mP\right\|_{\infty}  \right)^{k+1} \left\|\mPhi\vtheta_0-\mPhi\vtheta^*_{\eta}\right\|_{\infty}
\end{align*}

This gives the desired result.

\end{proof}

\subsection{Proof of Proposition~\ref{prop:outer-loop-decomposition}}\label{app:sec:prop:outer-loop-decomposition}

\begin{proof}
    We have
    \begin{equation*}
    \begin{aligned}
        &\E\left[\left\| \mPhi\vtheta_{t,K}- \mPhi\vtheta^*_{\eta} \right\|^2_{\infty} \right] \\
        \leq & (1+\delta) \E\left[ \left\|\mPhi \vtheta_{t,K}- \mGamma_{\eta}\gT\mPhi \vtheta_{t-1,K}\right\|^2_{\infty} \right]\\& + (1+\delta^{-1})\E\left[  \left\|\mGamma_{\eta}\gT\mPhi \vtheta_{t-1,K}-\mPhi\vtheta^*_{\eta} \right\|^2_{\infty}\right]\\
        =&  (1+\delta) \E\left[ \left\|\mPhi \vtheta_{t,K}- \mGamma_{\eta}\gT\mPhi \vtheta_{t-1,K}\right\|^2_{\infty} \right]\\& + (1+\delta^{-1})\E\left[ \left\|\mGamma_{\eta}\gT\mPhi \vtheta_{t-1,K}-\mGamma_{\eta}\gT\mPhi\vtheta^*_{\eta} \right\|^2_{\infty}\right] \\
        \leq &  (1+\delta)\E\left[\left\| \mPhi\vtheta_{t,K}-\mGamma_{\eta}\gT\mPhi\vtheta_{t-1,K} \right\|_{\infty}^2\right]\\& + \gamma^2 \left\| \mGamma_\eta \right\|^2_\infty (1+\delta^{-1}) \E\left[\left\| \mPhi\vtheta_{t-1,K}-\mPhi\vtheta^*_{\eta} \right\|^2_{\infty}  \right]\\
        \leq & (1+\delta)\E\left[\left\|\mPhi\vtheta_{t,K}-\mPhi(\mPhi^{\top}\mD\mPhi+\eta\mI)^{-1}\mPhi^{\top}\mD\gT\mPhi\vtheta_{t-1,K}\right\|^2_{\infty} \right] \\
        &+ \gamma^2 \left\| \mGamma_\eta \right\|^2_\infty (1+\delta^{-1}) \E\left[\left\| \mPhi\vtheta_{t-1,K}-\mPhi\vtheta^*_{\eta} \right\|^2_{\infty}  \right]\\
        \leq & \frac{2 (1+\delta)}{\mu_{\eta}}\left(\E\left[ L_{\eta}(\vtheta_{t,K},\vtheta_{t-1,K})-L_{\eta}(\vtheta^*(\vtheta_{t-1,K}),\vtheta_{t-1,K}) \right] \right) \\& + \gamma^2  \left\| \mGamma_\eta \right\|^2_\infty (1+\delta^{-1})  \E\left[\left\| \mPhi\vtheta_{t-1,K}-\mPhi\vtheta^*_{\eta} \right\|^2_{\infty}  \right].
    \end{aligned}
    \end{equation*}

        The first inequality follows from the relation $(a+b)^2\leq (1+\delta)a^2+(1+\delta^{-1})b^2$. The first equality follows from the fact that $\vtheta^*_\eta$ is the unique fixed point of (\ref{eq:rpbe}). The second inequality follows from Lemma~\ref{lem:mGamma<1}. The last inequality follows from Corollary~\ref{cor:quadratic-growth}. This concludes the proof.
\end{proof}

Next, let us define the following set:
\begin{align*}
    \gF_{t,k}:= \left\{ (s_{t,j},a_{t,j})_{j=0}^k , \vtheta_{t,0} \right\}.
\end{align*}

\begin{lemma} \label{lem:g-bound}
For $t\in\sN$ and $1\leq k \leq K$, we have
 \begin{equation*}
    \E\left[\left\| g(\vtheta_{t,k},\vtheta_{t-1,K};o_{t,k}) \right\|^2_2 \middle|\gF_{t,k} \right] \leq 10\gamma^2\left\| \mPhi(\vtheta_{t-1,K}-\vtheta^*_{\eta}) \right\|^2_{\infty}+(16+16\eta)L_{\eta}(\vtheta_{t,k},\vtheta_{t-1,K})+8\sigma_{\eta}^2.
\end{equation*}
and
 \begin{align*}
     \E\left[\left\| g(\vtheta_{t,k},\vtheta_{t-1,K};o_{t,k}) \right\|^2_2 \middle|\gF_{t,k} \right] \leq &g_{1,\eta} (L_{\eta}(\vtheta_{t,k},\vtheta_{t-1,K})-L(\vtheta^*(\vtheta_{t-1,K}),\vtheta_{t-1,K} ))\\
 &+ g_{2,\eta}  \left\| \mPhi(\vtheta_{t-1,K}-\vtheta^*_{\eta}) \right\|^2_{\infty}\\
 &+g_{3,\eta}.
\end{align*}
\end{lemma}

\begin{proof}
For simplicity of the proof, let us denote $r_{t,k}=r(s_{t,k},a_{t,k},s^{\prime}_{t,k})$ and $\vphi_{t,k}=\vphi(s_{t,k},a_{t,k})$. We have
\begin{align}
  &\E\left[ \left\| g(\vtheta_{t,k};\vtheta_{t-1,K};o_{t,k}) \right\|_2^2 \middle|\gF_{t,k}\right] \nonumber\\
  =& \E\left[\left\|\left( r_{t,k}+\gamma\max_{a\in\gA}\vphi(s^{\prime},a)^{\top}\vtheta_{t-1,K}-\vphi_{t,k}^{\top}\vtheta_{t,k} \right)(-\vphi_{t,k}) + \eta \vtheta_{t,k}\right\|_2^2 \middle|\gF_{t,k}\right] \nonumber \\
\leq & 2 \E \left[ \left\|\gamma \left(\max_{u\in\gA}\vphi(s^{\prime}_{t,k},u)^{\top}\vtheta_{t-1,K}-\max_{u\in\gA}\vphi(s^{\prime}_{t,k},u)^{\top}\vtheta^*_{\eta}\right)\vphi_{t,k} \right\|^2_2 \middle|\gF_{t,k} \right] \label{ineq:var-1} \\
&+2\E\left[ \left\| \left( r_{t,k}+\gamma\max_{a\in\gA}\vphi(s^{\prime}_{t,k},a)^{\top}\vtheta^*_{\eta}-\vphi_{t,k})^{\top}\vtheta_{t,k} \right) (-\vphi_{t,k})+\eta\vtheta_{t,k}    \right\|^2_2 \middle|\gF_{t,k}\right] .\label{ineq:var-2} 
\end{align}
The first inequality follows from the relation $||\va+\vb||^2_2\leq 2||\va||^2_2+2||\vb||^2_2$ for any $\va,\vb\in\R^{d}$. We will bound each term in~(\ref{ineq:var-1}) and~(\ref{ineq:var-2}). 

Let us first bound the term in~(\ref{ineq:var-1}):
\begin{align}
& \E\left[\left\|\gamma \left(\max_{u\in\gA}\vphi(s^{\prime}_{t,k},u)^{\top}\vtheta_{t-1,K}-\max_{u\in\gA}\vphi(s^{\prime}_{t,k},u)^{\top}\vtheta^*_{\eta}\right)\vphi_{t,k}\right\|^2_2  \middle | \gF_{t,k} \right]   \nonumber\\
\leq & \gamma^2 \E\left[\left|\max_{u\in\gA}\vphi(s^{\prime}_{t,k},u)^{\top}\vtheta_{t-1,K}-\max_{u\in\gA}\vphi(s^{\prime}_{t,k},u)^{\top}\vtheta^*_{\eta} \right|^2 \left\|\vphi_{t,k} \right\|^2_2 \middle | \gF_{t,k} \right] \nonumber\\
   \leq & \gamma^2 \E\left[\left(  \max_{u\in\gA}\left|\vphi(s^{\prime}_{t,k},u)^{\top}(\vtheta_{t-1,K}-\vtheta^*_{\eta}) \right|\right)^2\left\| \vphi_{t,k} \right\|^2_2 \middle | \gF_{t,k}\right] \nonumber\\
   \leq & \gamma^2 \E\left[ \left\|\mPhi(\vtheta_{t-1,K}-\vtheta^*_{\eta} )\right\|^2_{\infty} \left\|\vphi_{t,k}\right\|^2_2 \middle | \gF_{t,k} \right] \nonumber \\
    \leq & \gamma^2  \left\|\mPhi(\vtheta_{t-1,K}-\vtheta^*_{\eta} )\right\|^2_{\infty}.\label{ineq:var:max-contraction}
\end{align}
The second inequality follows from the non-expansiveness of the max-operator. The third inequality follows from $\max_{u\in\gA}|\vphi(s^{\prime}_{t,k},u)^{\top}\vtheta|\leq \left\| \mPhi\vtheta\right\|_{\infty}=\max_{(s,u)\in\gS\times \gA}|\vphi(s,u)^{\top}\vtheta|$ for any $\vtheta\in\R^d$.

Now, the term in~(\ref{ineq:var-2}) can be bounded as follows:
\begin{align}
 &\E\left[ \left\| \left( r_{t,k}+\gamma\max_{u\in\gA}\vphi(s^{\prime}_{t,k},u)^{\top}\vtheta^*_{\eta}-\vphi_{t,k}^{\top}\vtheta_{t,k} \right) (-\vphi_{t,k})+\eta\vtheta_{t,k}    \right\|^2_2 \middle|\gF_{t,k}\right]  \nonumber \\
 \leq & 2\E\left[ \left\| \left( \E\left[(r(s_{t,k},a_{t,k},\tilde{s})+\gamma\max_{u\in\gA}\vphi(\tilde{s},u)^{\top}\vtheta_{t-1,K})\middle|\gF_{t,k}\right] - \vphi_{t,k}^{\top}\vtheta_{t,k} \right)(-\vphi_{t,k})+\eta\vtheta_{t,k} \right\|^2_2  \middle|\gF_{t,k}\right] \label{ineq:var-dcomposition-1}\\
 &+2 \E\left[\left\| \left(r_{t,k}+\gamma\max_{u\in\gA}  \vphi(s^{\prime}_{t,k},u)^{\top}\vtheta^*_{\eta}-\E\left[(r(s_{t,k},a_{t,k},\tilde{s})+\gamma \max_{u\in\gA}\vphi(\tilde{s},u  )^{\top}\vtheta_{t-1,K})\right]\right)\vphi_{t,k} \right\|_2^2\middle|\gF_{t,k} \right] \nonumber 
\end{align}
 The first inequality again follows from the relation $||\va+\vb||^2_2\leq 2||\va||^2_2+2||\vb||^2_2$.

We note that the term in~(\ref{ineq:var-dcomposition-1}) can be bounded as follows:
\begin{align*}
    &  \E\left[ \left\| \left( \E\left[(r(s_{t,k},a_{t,k},\tilde{s})+\gamma\max_{u\in\gA}\vphi(\tilde{s},u)^{\top}\vtheta_{t-1,K})\right] - \vphi_{t,k}^{\top}\vtheta_{t,k} \right)(-\vphi_{t,k})+\eta\vtheta_{t,k} \right\|^2_2 \middle|\gF_{t,k}\right] \\
    \leq & 2\E\left[ \left\| \E\left[(r(s_{t,k},a_{t,k},\tilde{s})+\gamma\max_{u\in\gA}\vphi(\tilde{s},u)^{\top}\vtheta_{t-1,K})\right] - \vphi_{t,k}^{\top}\vtheta_{t,k} \right\|^2_2 + \eta^2 \left\| \vtheta_{t,k}\right\|^2_2 \middle|\gF_{t,k}\right]\\
    \leq & (4+4\eta)L_{\eta}(\vtheta_{t,k},\vtheta_{t-1,K}).
\end{align*}
The last inequality follows from the definition of $L_{\eta}(\cdot,\cdot)$ in~(\ref{eq:reg-mspbe}). Now, applying this result to~(\ref{ineq:var-dcomposition-1}), we get
\begin{align}
    &\E\left[ \left\| \left( r_{t,k}+\gamma\max_{u\in\gA}\vphi(s^{\prime}_{t,k},u)^{\top}\vtheta^*_{\eta}-\vphi_{t,k}^{\top}\vtheta_{t,k} \right) (-\vphi_{t,k})+\eta\vtheta_{t,k}    \right\|^2_2 \middle|\gF_{t,k}\right]  \nonumber \\
    \leq & (8+8\eta)L_{\eta}(\vtheta_{t,k},\vtheta_{t-1,K}) \nonumber\\
    &+2 \E\left[\left\| \left(r_{t,k}+\gamma\max_{u\in\gA}  \vphi(s^{\prime}_{t,k},u)^{\top}\vtheta^*_{\eta}-\E\left[(r(s_{t,k},a_{t,k},\tilde{s})+\gamma \max_{u\in\gA}\vphi(\tilde{s},u  )^{\top}\vtheta_{t-1,K})\right]\right)\vphi_{t,k} \right\|_2^2\middle|\gF_{t,k}\right] \nonumber\\
    \leq & (8+8\eta)L_{\eta}(\vtheta_{t,k},\vtheta_{t-1,K}) \nonumber \\
    &+4\E\left[\left\|  \left(r_{t,k}-\E\left[r(s_{t,k},a_{t,k},\tilde{s}) \right] +\gamma\max_{u\in\gA} \vphi(s_{t,k}^{\prime},u)^{\top} \vtheta^*_{\eta}-\gamma \E\left[\max_{u\in\gA}\vphi(\tilde{s},u)^{\top}\vtheta^*_{\eta}\right]\right)\vphi_{t,k}   \right\|_2^2 \middle|\gF_{t,k} \right] \nonumber \\
 &+4\E\left[  \left\|\left(\gamma \E\left[ \max_{u\in\gA}\vphi(s^{\prime}_{t,k},u)^{\top}\vtheta^*_{\eta}-\gamma\max_{u\in\gA}\vphi(\tilde{s},u)^{\top}\vtheta_{t-1,K}\right]\right)\vphi_{t,k} \right\|_2^2\middle|\gF_{t,k}\right] \nonumber\\
 \leq &(8+8\eta)L_{\eta}(\vtheta_{t,k},\vtheta_{t-1,K})+4\sigma_{\eta}^2+4\gamma^2\left\| \mPhi(\vtheta_{t-1,K}-\vtheta^*_{\eta})\right\|_{\infty}^2. \label{ineq:var-L}
\end{align}
The second inequality follows from the definition of $L_{\eta}(\vtheta_{t,k},\vtheta_{t-1,K})$ in~(\ref{eq:reg-mspbe}). The last inequality follows from the same logic in~(\ref{ineq:var:max-contraction}).

Now applying the bounds in~(\ref{ineq:var:max-contraction}) and~(\ref{ineq:var-L})
 to~(\ref{ineq:var-1}) and~(\ref{ineq:var-2}), respectively, we get
\begin{align*}
    \E\left[\left\| g(\vtheta_{t,k},\vtheta_{t-1,K};o_{t,k}) \right\|^2_2 \middle|\gF_{t,k} \right] \leq &10\gamma^2\left\| \mPhi(\vtheta_{t-1,K}-\vtheta^*_{\eta}) \right\|^2_{\infty}\\
    &+(16+16\eta)L_{\eta}(\vtheta_{t,k},\vtheta_{t-1,K})+8\sigma_{\eta}^2.
\end{align*}
This completes the proof of the first statement.

The second statement follows from simple decomposition:
\begin{align*}
        \E\left[\left\| g(\vtheta_{t,k},\vtheta_{t-1,K};o_{t,k})\right\|^2_2 \middle|\gF_{t,k} \right] \leq &10\gamma^2\left\| \mPhi(\vtheta_{t-1,K}-\vtheta^*_{\eta}) \right\|^2_{\infty}
    +(16+16\eta)L_{\eta}(\vtheta_{t,k},\vtheta_{t-1,K})+8\sigma_{\eta}^2\\
    \leq & 10\gamma^2\left\| \mPhi(\vtheta_{t-1,K}-\vtheta^*_{\eta}) \right\|^2_{\infty}\\
    &+(16+16\eta)(L_{\eta}(\vtheta_{t,k},\vtheta_{t-1,K})-L_{\eta}(\vtheta^*(\vtheta_{t-1,K}),\vtheta_{t-1,K} ))\\
    &+(16+16\eta)L_{\eta}(\vtheta^*(\vtheta_{t-1,K}),\vtheta_{t-1,K})+8\sigma_{\eta}^2\\
    \leq & 10\gamma^2\left\| \mPhi(\vtheta_{t-1,K}-\vtheta^*_{\eta}) \right\|^2_{\infty}\\
    &+(16+16\eta)(L_{\eta}(\vtheta_{t,k},\vtheta_{t-1,K})-L_{\eta}(\vtheta^*(\vtheta_{t-1,K}),\vtheta_{t-1,K} ))\\
    &+(16+16\eta)\left( R_{\max}^2+2\gamma^2\left\| \mPhi(\vtheta_{t-1,K}-\vtheta^*_{\eta}) \right\|^2_{\infty}+2\gamma^2\left\| \mPhi\vtheta^*_{\eta}\right\|^2_{\infty} \right)+8\sigma_{\eta}^2\\
 =& (16+16\eta)(L_{\eta}(\vtheta_{t,k},\vtheta_{t-1,K})-L_{\eta}(\vtheta^*(\vtheta_{t-1,K}),\vtheta_{t-1,K} ))\\
 &+  (42 +32\eta)\gamma^2  \left\| \mPhi(\vtheta_{t-1,K}-\vtheta^*_{\eta}) \right\|^2_{\infty}\\
 &+ 32(1+\eta)\gamma^2\left\|\mPhi\vtheta^*_{\eta} \right\|^2_{\infty}+(16+16\eta)R_{\max}^2+8\sigma_{\eta}^2.
\end{align*}
The last inequality follows from Lemma~\ref{lem:minL-bound}.
\end{proof}

The following lemma bounds the inner loop loss in terms of the error of previous final iterate:

\begin{lemma}\label{lem:minL-bound}
         For any $\vtheta^{\prime}\in\R^{h}$ the following holds:
    \begin{align*}
        L_{\eta} ( \vtheta^*(\vtheta^{\prime}),\vtheta^{\prime})\leq   R_{\max}^2 + 2\gamma^2 \left\|\mPhi(\vtheta^{\prime}-\vtheta^*_{\eta}) \right\|^2_{\infty} + 2\gamma^2\left\| \mPhi\vtheta^*_{\eta} \right\|^2_{\infty}
    \end{align*}
\end{lemma}
\begin{proof}
By the definition of $\vtheta^*(\vtheta^{\prime})$ as the minimizer of $L_{\eta}(\cdot, \vtheta^{\prime})$, we have $L_{\eta}(\vtheta^*(\vtheta^{\prime}), \vtheta^{\prime}) \leq L_{\eta}(\bm{0}, \vtheta^{\prime})$.
Plugging in the zero vector, we have
    \begin{align*}
          L_{\eta}(\bm{0},\vtheta^{\prime}) =& \frac{1}{2}\left\|  \mR+\gamma \mP\mPi_{\vtheta^{\prime}}\mPhi\vtheta^{\prime} \right\|^2_{\mD}  \\
          \leq & R_{\max}^2+\gamma^2\left\|\mP\mPi_{\vtheta^{\prime}}\mPhi\vtheta^{\prime}-\mP\mPi_{\vtheta^*_{\eta}}\mPhi\vtheta^*_{\eta} + \mP\mPi_{\vtheta^*_{\eta}}\mPhi\vtheta^*_{\eta} \right\|^2_{\mD} \\
          \leq & R_{\max}^2 + 2\gamma^2 \left\|\mPhi(\vtheta^{\prime}-\vtheta^*_{\eta}) \right\|^2_{\infty} + 2\gamma^2\left\| \mPhi\vtheta^*_{\eta} \right\|^2_{\infty}
    \end{align*}
    This completes the proof.
\end{proof}

\begin{corollary}\label{cor:quadratic-growth}
    We have
    \begin{align*}
       L_{\eta}(\vtheta,\vtheta^{\prime})-L_{\eta}(\vtheta^*(\vtheta^{\prime}),\vtheta^{\prime}) \geq \frac{\mu_{\eta}}{2}\left\|\mGamma_{\eta}\gT\mPhi\vtheta^{\prime}-\mPhi\vtheta \right\|^2_{\infty}.
    \end{align*}
\end{corollary}

\begin{proof}
The quadratic growth condition in Lemma~\ref{quadraticgrowth} implies that 
\begin{align*}
   L_{\eta}(\vtheta,\vtheta^{\prime})-L_{\eta}(\vtheta^*(\vtheta^{\prime}),\vtheta^{\prime}) \geq & \frac{\mu_{\eta}}{2}\left\|\vtheta^*(\vtheta^{\prime})-\vtheta\right\|^2_2 \\
=& \frac{\mu_{\eta}}{2}\left\|(\mPhi^{\top}\mD\mPhi+\eta\mI)^{-1}\mPhi^{\top}\mD\gT\mPhi\vtheta^{\prime}-\vtheta \right\|^2_2\\
\geq & \frac{\mu_{\eta}}{2}\left\|(\mPhi^{\top}\mD\mPhi+\eta\mI)^{-1}\mPhi^{\top}\mD\gT\mPhi\vtheta^{\prime}-\vtheta \right\|^2_{\infty}\\
\geq & \frac{\mu_{\eta}}{2}\left\|\mPhi(\mPhi^{\top}\mD\mPhi+\eta\mI)^{-1}\mPhi^{\top}\mD\gT\mPhi\vtheta^{\prime}-\mPhi\vtheta \right\|^2_{\infty}
\end{align*}
The first equality follows from the definition of $\vtheta^*(\vtheta^{\prime})$ in~(\ref{def:theta^*}). The second inequality follows from the vector norm inequality $\|\cdot\|_{\infty}\leq \|\cdot\|_2$, and the last inequality follows from Assumption~\ref{assmp:feature_matrix}. 
\end{proof}

\begin{lemma}\label{lem:matrix-boudn-1}
For $\vtheta\in\R^h$, we have
    \begin{align*} 
        \left\|\gamma\mP\mPi_{\vtheta}-\mI \right\|_{\infty}\leq 2.
    \end{align*}
\end{lemma}

\begin{proof}
Note that we have
    
    \begin{align*}
    1-\gamma [\mP\mPi_{\vtheta}]_{ii} +\gamma\sum_{j\neq i} |[\mP\mPi_{\vtheta}]_{ij}|\leq 2.
    \end{align*}
    This completes the proof.
\end{proof}

\begin{lemma}\label{lem:phi-lipschitz}
    For $\vx, \vy,\vtheta,\vtheta^{\prime}\in\R^h$ and $(s,a)\in\gS\times\gA$, we have
    \begin{align*}
       \left\|  \bar{g}(\vx,\vtheta^{\prime};s,a)-\bar{g}(\vy,\vtheta^{\prime};s,a) \right\|_2 \leq  \left( 1+ \eta\right) \left\| \vx - \vy  \right\|_2.
    \end{align*}
    Moreover, we have
    \begin{align*}
               \left\|  \bar{g}(\vx,\vtheta;s,a)-\bar{g}(\vx,\vtheta^{\prime};s,a) \right\|_2 \leq &   \gamma \left\| \mPhi\vtheta-\mPhi\vtheta^{\prime} \right\|_{\infty}, \\
            \left\|  \nabla L_{\eta}(\vx,\vtheta)-\nabla L_{\eta}(\vx,\vtheta^{\prime}) \right\|_2 \leq &   \gamma \left\| \mPhi\vtheta-\mPhi\vtheta^{\prime} \right\|_{\infty}.
    \end{align*}
\end{lemma}

\begin{proof}
    From the definition of $\bar{g}(\cdot)$ in~(\ref{def:bar_g}), we have
    \begin{align*}
        \left\|  \bar{g}(\vx,\vtheta^{\prime};s,a)-\bar{g}(\vy,\vtheta^{\prime};s,a) \right\|_2 =&  \left\| -\vphi(s,a)^{\top}(\vx-\vy)\vphi(s,a)+\eta(\vx-\vy)  \right\|_2\\
        \leq & \left( 1+ \eta\right) \left\| \vx - \vy  \right\|_2
    \end{align*}
    The last line follows from the boundedness of the feature vector.
    
    Now, the second statement follows from 
    \begin{align}
                &\left\|  \bar{g}(\vx,\vtheta^{\prime};s,a)-\bar{g}(\vx,\vtheta;s,a) \right\|_2 \nonumber\\
                =& \left\| \gamma  \vphi(s,a)\sum_{s^{\prime}\in\gS}\gP(s^{\prime}\mid s,a) \left(\max_{u\in\gA} \vphi(s^{\prime},u)^{\top}\vtheta-\max_{u\in\gA}\vphi(s^{\prime},u)^{\top}\vtheta^{\prime}\right) \right\|_2 \nonumber\\
                \leq & \gamma \sum_{s^{\prime}\in\gS}\gP(s^{\prime}\mid s,a)\left|\max_{u\in\gA}\vphi(s^{\prime},u)^{\top}\vtheta-\max_{u\in\gA}\vphi(s^{\prime},u)^{\top}\vtheta^{\prime} \right| \nonumber\\
                \leq & \gamma \sum_{s^{\prime}\in\gS}\gP(s^{\prime}\mid s,a)\left| \max_{u\in\gA}\vphi(s^{\prime},u)^{\top}(\vtheta-\vtheta^{\prime}) \right| \nonumber \\
                \leq & \gamma \left\| \mPhi\vtheta-\mPhi\vtheta^{\prime} \right\|_{\infty}. \label{ineq:barg-lipschitz}
    \end{align}    
    The first inequality follows from the non-expansiveness of the max-operator. The last inequality follows from the definition of the infinity norm. 
    
    The same logic holds for the Lipschitzness of $\nabla L_{\eta}$ with respect to its second argument:
    \begin{align*}
        \left\| \nabla L_{\eta}( \vx,\vtheta)-\nabla L_{\eta}(\vx,\vtheta^{\prime})\right\|_2 =& \left\| \sum_{(s,a)\in\gS\times\gA}d(s,a) \left(\bar{g}(\vx,\vtheta^{\prime};s,a)-\bar{g}(\vx,\vtheta;s,a) \right)\right\|_2\\
        \leq & \sum_{(s,a)\in \gS\times\gA}d(s,a)\left\| \bar{g}(\vx,\vtheta^{\prime};s,a)-\bar{g}(\vx,\vtheta;s,a) \right\|_2\\
        \leq & \gamma \left\|\mPhi\vtheta-\mPhi\vtheta^{\prime}\right\|_{\infty}.
    \end{align*}
    The last line follows from~(\ref{ineq:barg-lipschitz}). This completes the proof.
\end{proof}

\section{Geometry of the Inner-Loop Objective}\label{app:sec:property_loss_ftn}

This section provides properties on the geometry of the inner-loop objective. We adopt the standard optimization framework~\cite{nesterov2018lectures}.

\begin{lemma}[Strong convexity and smoothness]\label{lem:strong_convexity}
For fixed $\vtheta'\in\mathbb{R}^h$, the function $L_{\eta}(\vtheta,\vtheta')$ is $\mu_\eta$-strongly convex in $\vtheta$, where $\mu_\eta=\lambda_{\min}(\mPhi^\top \mD \mPhi)+\eta$ and $l_{\eta}$-smooth where $l_{\eta}=\lambda_{\max}(\mPhi^\top \mD \mPhi)+\eta$.
\end{lemma}

\begin{proof}
The derivative of $L_{\eta}(\vtheta,\vtheta')$ with respect to $\vtheta$ is
\begin{align*}
\nabla L_{\eta}(\vtheta,\vtheta') 
&= \mathbb{E}_{s,a}\Big[ \Big(\mathbb{E}_{s'}\big[r(s,a,s')+\gamma\max_{u\in\mathcal{A}}\vphi(s',u)^\top\vtheta' - \vphi(s,a)^\top\vtheta\big]\Big)(-\vphi(s,a)) + \eta \vtheta \Big].
\end{align*}
The second-order derivative is
\begin{align*}
\nabla^2 L_{\eta}(\vtheta,\vtheta')
= \sum_{(s,a)\in\mathcal{S}\times\mathcal{A}} d(s,a)\big(\vphi(s,a)\vphi(s,a)^\top+\eta \mI\big)
= \mPhi^\top \mD \mPhi + \eta \mI.
\end{align*}

Since $\mPhi^\top \mD \mPhi$ is positive semidefinite, all eigenvalues of $\nabla^2 L_\eta$ are bounded below by $\lambda_{\min}(\mPhi^\top \mD\mPhi)+\eta=\mu_\eta>0$. Hence $L_{\eta}(\cdot,\theta')$ is $\mu_\eta$-strongly convex. The smoothness also follows from the definition in Definition~\ref{def:strongly_cvx_smooth}.
\end{proof}

\begin{lemma}[Descent lemma]\label{lem:smooth}
Fix $\vtheta'\in\mathbb{R}^h$ and let $l_\eta=\lambda_{\max}(\mPhi^\top \mD \mPhi)+\eta$.
Then for any $\vtheta,\vDelta\in\mathbb{R}^h$ and $\alpha > 0$:

\[
   L_{\eta}(\vtheta-\alpha \vDelta,\vtheta')
   \le
   L_{\eta}(\vtheta,\vtheta') - \alpha\,\nabla L_{\eta}(\vtheta,\vtheta')^\top \vDelta
   + \frac{l_\eta}{2}\,\alpha^2\|\vDelta\|^2.
\]
\end{lemma}
\begin{proof}
Since $L_\eta(\cdot,\vtheta')$ is $l_\eta$-smooth in $\vtheta$, its gradient is $l_\eta$-Lipschitz. From the definition of smoothness in Definition~\ref{def:strongly_cvx_smooth}, for any $\vtheta, \vDelta $ and any $\alpha >0$,
\[
  L_{\eta}(\vtheta-\alpha \vDelta,\vtheta')
  \le
  L_{\eta}(\vtheta,\vtheta') + \nabla L_{\eta}(\vtheta,\vtheta')^\top\big((\vtheta-\alpha \vDelta)-\vtheta\big)
  + \frac{l_\eta}{2}\,\|\vtheta-\alpha \vDelta-\vtheta\|^2,
\]
which simplifies to
\[
  L_{\eta}(\vtheta-\alpha \vDelta,\vtheta')
  \le
  L_{\eta}(\vtheta,\vtheta') - \alpha\,\nabla L_{\eta}(\vtheta,\vtheta')^\top \vDelta
  + \frac{l_\eta}{2}\,\alpha^2\|\vDelta\|^2.
\]

This completes the proof.
\end{proof}

The definitions of strong convexity and smoothness are provided in Section~\ref{app:sec:optimization} of the Appendix. The following properties will be useful throughout the paper:

\begin{lemma}[Theorem 2 in~\citet{karimi2016linear}]\label{lem:gdgap}
For fixed $\vtheta'\in\mathbb{R}^h$ , $\vtheta^*(\vtheta') \;=\; \arg\min_{\vtheta\in\mathbb{R}^h} L_{\eta}(\vtheta,\vtheta')$ and any $\vtheta\in\mathbb{R}^h$, 
\[
   \|\nabla L_{\eta}(\vtheta,\vtheta')\|^2 \ge 2\mu_\eta\big(L_{\eta}(\vtheta,\vtheta')-L_{\eta}(\vtheta^*(\vtheta'),\vtheta')\big).
\]
\end{lemma}

\begin{lemma}
\label{quadraticgrowth}
For fixed $\vtheta'\in\mathbb{R}^h$ and any $\vtheta\in\mathbb{R}^h$, 
\[
   L_{\eta}(\vtheta,\vtheta')-L_{\eta}(\vtheta^*(\vtheta'),\vtheta') \le \frac{l_\eta}{2}\,\|\vtheta-\vtheta^*(\vtheta')\|^2.
\]
\end{lemma}

\begin{proof}
By Lemma~\ref{lem:smooth} (the $l_\eta$-smoothness of $L_{\eta}(\cdot,\vtheta')$), for any $\vx,\vy\in\mathbb{R}^h$,
\[
L_{\eta}(\vy,\vtheta') \le L_{\eta}(\vx,\vtheta') + \nabla L_{\eta}(\vx,\vtheta')^\top (\vy-\vx) + \frac{l_\eta}{2}\,\|\vy-\vx\|^2.
\]
Apply this with $\vx=\vtheta^*(\vtheta')$ and $\vy=\vtheta$. Since $\vtheta^*(\vtheta')$ minimizes $L_{\eta}(\cdot,\vtheta')$, we have $\nabla L_{\eta}(\vtheta^*(\vtheta'),\vtheta')=0$, hence
\[
L_{\eta}(\vtheta,\vtheta') - L_{\eta}(\vtheta^*(\vtheta'),\vtheta')
\le \frac{l_\eta}{2}\,\|\vtheta-\vtheta^*(\vtheta')\|^2.
\]
\end{proof}

\begin{lemma}[Lipschitz property]\label{lem:theta-star-lipschitz}
For any $\vtheta\in\mathbb{R}^h$, 
\[
   \|\vtheta^*(\vtheta)-\vtheta^*_\eta\|_2 \le \frac{\gamma}{\mu_\eta}\,\|\mPhi(\vtheta-\vtheta^*_\eta)\|_\infty.
\]
\end{lemma}

\begin{proof}
    We have
    \begin{align*}
    &\left\| \vtheta^*(\vtheta)-\vtheta^*_{\eta} \right\|_2 \\=& \left\| \gamma(\mPhi^{\top}\mD\mPhi+\eta\mI)^{-1} \mPhi^{\top}\mD\mP \left( \mPi_{\vtheta}\mPhi\vtheta - \mPi_{\vtheta^*_\eta}\mPhi\vtheta^*_{\eta} \right) \right\|_2\\
    \leq & \gamma \left\| (\mPhi^{\top}\mD\mPhi+\eta\mI)^{-1} \right\|_2 \left\| \sum_{(s,a)\in\gS\times\gA}d(s,a)\vphi(s,a) \left(\sum_{s^{\prime}\in\gS} \gP(s^{\prime}\mid s,a) \left(\max_{u\in\gA}\vphi(s^{\prime},u)^{\top}\vtheta-\max_{u\in\gA}\vphi(s^{\prime},u)^{\top}\vtheta^*_{\eta}\right)  \right)\right\|_2\\
    \leq & \frac{\gamma}{\mu_{\eta}} \left( \sum_{(s,a)\in\gS\times\gA} d(s,a) \left\|\vphi(s,a) \right\|_2 \sum_{s^{\prime}\in\gS} \gP(s^{\prime}\mid s,a)\left| \max_{u\in\gA}\vphi(s^{\prime},u)^{\top}\vtheta-\max_{u\in\gA}\vphi(s^{\prime},u)^{\top}\vtheta^*_{\eta}   \right|\right)\\
    \leq & \frac{\gamma}{\mu_{\eta}} \left( \sum_{(s,a)\in\gS\times\gA} d(s,a)\sum_{s^{\prime}\in\gS} \gP(s^{\prime}\mid s,a) \left\| \mPhi\vtheta-\mPhi\vtheta^*_{\eta} \right\|_{\infty}\right)\\
    \leq &\frac{\gamma}{\mu_{\eta}} \left\| \mPhi\vtheta-\mPhi\vtheta^*_{\eta} \right\|_{\infty}.
\end{align*}
This completes the proof.
\end{proof}

\section{Analysis and proof for i.i.d observation model}
\label{app:iidproof}


Our goal is to establish an $\epsilon$–accurate error guarantee of the form in i.i.d. observation model,
\[
\mathbb{E}\big[\|\mPhi(\vtheta_{t,K}-\vtheta^*_{\eta})\|_{\infty}^2\big]\le \epsilon.
\]

To that end, we analyze the geometry of the inner-loop objective $L_{\eta}(\vtheta,\vtheta')$, collecting strong convexity, smoothness, gradient–gap, and related Lipschitz properties that will serve as our basic tools (Section~\ref{app:sec:property_loss_ftn}). We then derive a finite-time bound by viewing the inner loop as stochastic gradient descent on a strongly convex and smooth objective under the i.i.d.\ sampling assumption (Section~\ref{subsec:inner-iid}). This analysis yields a single linear recursion, whose solution leads to our main result (Theorem~\ref{thm:epsilon-accuracy-main}), showing that, with appropriate choices of the step size, inner-loop length, and number of outer iterations, the desired $\epsilon$-accuracy is achieved.

\subsection{Finite Time Error Analysis (i.i.d)}
\label{subsec:inner-iid}

\begin{lemma}
\label{lem:1stepk}
Suppose the step size $\alpha \le \frac{3\mu_\eta}{l_\eta g_{1,\eta}}$. Then for each inner iteration $k$,
\begin{align*}
\mathbb{E}\!\left[
L_\eta(\vtheta_{t,k+1},\vtheta_{t-1,K}) - L_\eta(\vtheta^*(\vtheta_{t-1,K}),\vtheta_{t-1,K})
\,\right]
&\le
\left(1 - \tfrac{\mu_\eta}{2}\alpha\right) \\[-2pt]
&\qquad\times
\mathbb{E}\!\left[L_\eta(\vtheta_{t,k},\vtheta_{t-1,K})
- L_\eta(\vtheta^*(\vtheta_{t-1,K}),\vtheta_{t-1,K})\right]\\[2pt]
&\quad + \frac{l_\eta}{2}\alpha^2
\Big[
g_{2,\eta}\mathbb{E}\!\left[\big\|\mPhi(\vtheta_{t-1,K}-\vtheta^*_\eta)\big\|_\infty^2\right]
+ g_{3,\eta}
\Big].
\end{align*}

\end{lemma}
\begin{proof}
Fix $t$ and $k$, and condition on $(\vtheta_{t,k},\vtheta_{t-1,K})$.  
Apply Lemma~\ref{lem:smooth} with $g=g(\vtheta_{t,k},\vtheta_{t-1,K};o_{t,k})$ and stepsize $\alpha > 0$:
\[
L_\eta(\vtheta_{t,k+1},\vtheta_{t-1,K})
\le
L_{\eta}(\vtheta_{t,k},\vtheta_{t-1,K})
-\alpha\,\nabla L_{\eta}(\vtheta_{t,k},\vtheta_{t-1,K})^\top g(\vtheta_{t,k},\vtheta_{t-1,K};o_{t,k})
+\frac{l_\eta}{2}\alpha^2\|g(\vtheta_{t,k},\vtheta_{t-1,K};o_{t,k})\|^2.
\]

Taking conditional expectation given $(\vtheta_{t,k},\vtheta_{t-1,K})$ and using Lemma~\ref{lem:g-bound}, we have,
\[
\begin{aligned}
\mathbb{E}\!\left[g(\vtheta_{t,k},\vtheta_{t-1,K};o_{t,k})
   \,\middle|\,\vtheta_{t,k},\vtheta_{t-1,K}\right]
   &= \nabla L_{\eta}(\vtheta_{t,k},\vtheta_{t-1,K}), \\[6pt]
\mathbb{E}\!\left[\|g(\vtheta_{t,k},\vtheta_{t-1,K};o_{t,k})\|^2
   \,\middle|\,\vtheta_{t,k},\vtheta_{t-1,K}\right]
   &\le g_{1,\eta} (L_{\eta}(\vtheta_{t,k},\vtheta_{t-1,K})-L_\eta(\vtheta^*(\vtheta_{t-1,K}),\vtheta_{t-1,K} ))\\
 &+ g_{2,\eta}  \left\| \mPhi(\vtheta_{t-1,K}-\vtheta^*_{\eta}) \right\|^2_{\infty}\\
 &+g_{3,\eta}.
\end{aligned}
\]

\noindent Thus, we obtain

\begin{align*}
\mathbb{E}\!\left[L_\eta(\vtheta_{t,k+1},\vtheta_{t-1,K})
\,\middle|\,\vtheta_{t,k},\vtheta_{t-1,K}\right]
&\le
L_{\eta}(\vtheta_{t,k},\vtheta_{t-1,K})
-\alpha\,\|\nabla L_{\eta}(\vtheta_{t,k},\vtheta_{t-1,K})\|^2 \\[4pt]
&\quad
+\frac{l_\eta}{2}\alpha^2
\Big[
g_{1,\eta}\big(L_{\eta}(\vtheta_{t,k},\vtheta_{t-1,K})
- L_\eta(\vtheta^*(\vtheta_{t-1,K}),\vtheta_{t-1,K})\big) \\[2pt]
&\qquad\quad
+\, g_{2,\eta}\big\|\mPhi(\vtheta_{t-1,K}-\vtheta^*_\eta)\big\|_\infty^2
+ g_{3,\eta}
\Big].
\end{align*}

\noindent Using Lemma~\ref{lem:gdgap}
\[
\|\nabla L_{\eta}(\vtheta_{t,k},\vtheta_{t-1,K})\|^2
\ge 2\mu_\eta\!\left(L_{\eta}(\vtheta_{t,k},\vtheta_{t-1,K})
- L_\eta(\vtheta^*(\vtheta_{t-1,K}),\vtheta_{t-1,K})\right),
\]
we obtain
\begin{align*}
\mathbb{E}\!\left[
L_\eta(\vtheta_{t,k+1},\vtheta_{t-1,K}) - L_\eta(\vtheta^*(\vtheta_{t-1,K}),\vtheta_{t-1,K})
\,\middle|\,\vtheta_{t,k},\vtheta_{t-1,K}\right]
&\le
\Big(1-2\mu_\eta\alpha+\tfrac{l_\eta}{2}\alpha^2 g_{1,\eta}\Big) \\[-2pt]
&\qquad\times
\big(L_{\eta}(\vtheta_{t,k},\vtheta_{t-1,K})
- L_\eta(\vtheta^*(\vtheta_{t-1,K}),\vtheta_{t-1,K})\big) \\[2pt]
&\quad + \frac{l_\eta}{2}\alpha^2
\Big[
g_{2,\eta}\big\|\mPhi(\vtheta_{t-1,K}-\vtheta^*_\eta)\big\|_\infty^2
+ g_{3,\eta}
\Big].
\end{align*}

\noindent For $0 < \alpha \le  \frac{3\mu_\eta}{l_\eta g_{1,\eta}},$ one can replace the quadratic rate term by a linear bound:
\[
1 - 2\mu_\eta \alpha + \frac{l_\eta}{2} g_{1,\eta} \alpha^2
\;\le\;
1 - \frac{\mu_\eta}{2}\alpha.
\]

\noindent Thus,
\begin{align*}
\mathbb{E}\!\left[
L_\eta(\vtheta_{t,k+1},\vtheta_{t-1,K}) - L_\eta(\vtheta^*(\vtheta_{t-1,K}),\vtheta_{t-1,K})
\,\middle|\,\vtheta_{t,k},\vtheta_{t-1,K}\right]
&\le
\left(1 - \tfrac{\mu_\eta}{2}\alpha\right) \\[-2pt]
&\qquad\times
\big(L_{\eta}(\vtheta_{t,k},\vtheta_{t-1,K})
- L_\eta(\vtheta^*(\vtheta_{t-1,K}),\vtheta_{t-1,K})\big) \\[2pt]
&\quad + \frac{l_\eta}{2}\alpha^2
\Big[
g_{2,\eta}\big\|\mPhi(\vtheta_{t-1,K}-\vtheta^*_\eta)\big\|_\infty^2
+ g_{3,\eta}
\Big].
\end{align*}

\noindent Finally, take total expectation on $(\vtheta_{t,k},\vtheta_{t-1,K})$ to conclude the claim.
\end{proof}

\begin{lemma}
\label{lem:k-stepinnerbound}
Suppose
the one-step recursion (Lemma~\ref{lem:1stepk} with $ \alpha \le  \frac{3\mu_\eta}{l_\eta g_{1,\eta}}$) holds. Then
\begin{align}
   \E\!\Bigl[L_{\eta}(\vtheta_{t,k},\vtheta_{t-1,K})
   - L_\eta(\vtheta^*(\vtheta_{t-1,K}),\vtheta_{t-1,K})\Bigr]
   \le&
   \Bigl(1-\tfrac{\mu_\eta}{2}\alpha\Bigr)^k
   \E\!\Bigl[L_\eta(\vtheta_{t,0},\vtheta_{t-1,K})
   - L_\eta(\vtheta^*(\vtheta_{t-1,K}),\vtheta_{t-1,K})\Bigr] \nonumber\\
   &+ \frac{l_\eta}{\mu_\eta}\,\alpha\,
\Big[
g_{2,\eta}\,\E\!\left[\big\|\mPhi(\vtheta_{t-1,K}-\vtheta^*_\eta)\big\|_\infty^2\right]
+ g_{3,\eta}
\Big].
\end{align}
\end{lemma}

\begin{proof} Let 
\[
x_k := \E\!\left[L_{\eta}(\vtheta_{t,k},\vtheta_{t-1,K}) - L_\eta(\vtheta^*(\vtheta_{t-1,K}),\vtheta_{t-1,K})\right].
\]
From Lemma~\ref{lem:1stepk}, with the stepsize condition \(\alpha \le  \tfrac{3\mu_\eta}{l_\eta g_{1,\eta}} \),
the recursion becomes
\[
x_{k+1}
\;\le\;
\Bigl(1-\tfrac{\mu_\eta}{2}\alpha\Bigr)\,x_k
\;+\;
\frac{l_\eta}{2}\alpha^2\left[
g_{2,\eta}\,\E\!\left[\big\|\mPhi(\vtheta_{t-1,K}-\vtheta^*_\eta)\big\|_\infty^2\right]
+ g_{3,\eta}
\right].
\]

By induction, this yields
\[
x_k
\;\le\;
\Bigl(1-\tfrac{\mu_\eta}{2}\alpha\Bigr)^{k} x_0
+
\frac{l_\eta}{2}\alpha^2
\left[
g_{2,\eta}\,\E\!\left[\big\|\mPhi(\vtheta_{t-1,K}-\vtheta^*_\eta)\big\|_\infty^2\right]
+ g_{3,\eta}
\right]
\sum_{i=0}^{k-1}
\Bigl(1-\tfrac{\mu_\eta}{2}\alpha\Bigr)^{i}.
\]

Since the sum is geometric and bounded by the infinite series,
\[
\sum_{i=0}^{k-1}\Bigl(1-\tfrac{\mu_\eta}{2}\alpha\Bigr)^{i}
\;\le\;
\sum_{i=0}^{\infty}\Bigl(1-\tfrac{\mu_\eta}{2}\alpha\Bigr)^{i}
= \frac{1}{\tfrac{\mu_\eta}{2}\alpha}
= \frac{2}{\mu_\eta\,\alpha},
\]
we obtain the bound
\[
x_k
\;\le\;
\Bigl(1-\tfrac{\mu_\eta}{2}\alpha\Bigr)^{k} x_0
\;+\;
\frac{l_\eta}{\mu_\eta}\,\alpha\,
\Big[
g_{2,\eta}\,\E\!\left[\big\|\mPhi(\vtheta_{t-1,K}-\vtheta^*_\eta)\big\|_\infty^2\right]
+ g_{3,\eta}
\Big].
\]

\end{proof}

\begin{lemma}
\label{lem:init_lossgap}
The following lemma holds.
    \[
   L_\eta(\vtheta_{t,0},\vtheta_{t-1,K})
   - L_\eta(\vtheta^*(\vtheta_{t-1,K}),\vtheta_{t-1,K})
   \;\leq\;
   R_{\max}^2+8\left\|\mPhi(\vtheta_{t-1,K}-\vtheta^*_{\eta}) \right\|^2_{\infty}+8 \left\| \mPhi\vtheta^*_{\eta}\right\|^2_{\infty}
    \]
\end{lemma}

\begin{proof}
    We have
    \begin{align*}
        & L_\eta(\vtheta_{t-1,K},\vtheta_{t-1,K})-L_\eta(\vtheta^*(\vtheta_{t-1,K}),\vtheta_{t-1,K})  \\
        \leq & \frac{1}{2} \left\|\mR+\gamma\mP\mPi_{\vtheta_{t-1,K}}\mPhi\vtheta_{t-1,K} -\mPhi\vtheta_{t-1,K}\right\|^2_{\mD}\\
        \leq & R_{\max}^2 +\left\| \gamma\mP\mPi_{\vtheta_{t-1,K}}\mPhi\vtheta_{t-1,K}-\mPhi\vtheta_{t-1,K} \right\|^2_{\infty}\\
        \leq &R_{\max}^2+2\left\| \gamma \mP\mPi_{\vtheta_{t-1,K}}\mPhi(\vtheta_{t-1,K}-\vtheta^*_{\eta})-\mPhi(\vtheta_{t-1,K}-\vtheta^*_{\eta})\right\|^2_{\infty}+2\left\| \gamma\mP\mPi_{\vtheta_{t-1,K}}\mPhi\vtheta_{\eta}^*-\mPhi\vtheta^*_{\eta} \right\|^2_{\infty}\\
        \leq & R_{\max}^2+8\left\|\mPhi(\vtheta_{t-1,K}-\vtheta^*_{\eta}) \right\|^2_{\infty}+8 \left\| \mPhi\vtheta^*_{\eta}\right\|^2_{\infty}
    \end{align*}
    The first inequality follows from the definition of $L_\eta(\cdot,\cdot)$ in~(\ref{eq:reg-mspbe}). The second and third inequalities follow from the relation $||\va+\vb||^2_\infty \leq 2||\va||^2_\infty+2||\vb||^2_\infty$ for $\va,\vb\in\R^d$. The last line follows from Lemma~\ref{lem:matrix-boudn-1}. This completes the proof.
\end{proof}

\begin{lemma}[Main recursion]
\label{lem:main_recur}
Let 
\[
y_t := \E\!\left[\|\mPhi\vtheta_{t,K}-\mPhi\vtheta^*_\eta\|_\infty^2\right],
\qquad 
y_{t-1} := \E\!\left[\|\mPhi\vtheta_{t-1,K}-\mPhi\vtheta^*_\eta\|_\infty^2\right].
\]
Under the step size condition 
$\bigg( 0 < \alpha \le \min\!\left\{\frac{3\mu_\eta}{l_\eta g_{1,\eta}},\,\frac{2}{\mu_\eta}\right\} \bigg)$,
the following inequality holds:
\begin{align*}
y_t 
&\le
\Bigg[
\frac{16(1+\gamma^{2}\left\| \mGamma_\eta \right\|^2_\infty)}{\mu_{\eta}(1-\gamma^{2}\left\| \mGamma_\eta \right\|^2_\infty)}
\Bigl(1-\tfrac{\mu_\eta}{2}\alpha\Bigr)^{K}
\;+\;
\frac{2(1+\gamma^{2}\left\| \mGamma_\eta \right\|^2_\infty)\,l_\eta}{\mu_\eta^{2}(1-\gamma^{2}\left\| \mGamma_\eta \right\|^2_\infty)}\,\alpha\,g_{2,\eta}
\;+\;
\frac{1+\gamma^{2}\left\| \mGamma_\eta \right\|^2_\infty}{2}
\Bigg] y_{t-1} \nonumber\\
&\quad+\;
\frac{2(1+\gamma^{2}\left\| \mGamma_\eta \right\|^2_\infty)}{\mu_{\eta}(1-\gamma^{2}\left\| \mGamma_\eta \right\|^2_\infty)}
\Bigl(1-\tfrac{\mu_\eta}{2}\alpha\Bigr)^{K} R_{\max}^2
\;+\;
\frac{16(1+\gamma^{2}\left\| \mGamma_\eta \right\|^2_\infty)}{\mu_{\eta}(1-\gamma^{2}\left\| \mGamma_\eta \right\|^2_\infty)}
\Bigl(1-\tfrac{\mu_\eta}{2}\alpha\Bigr)^{K}
\big\|\mPhi\vtheta^*_\eta\big\|_\infty^2\nonumber\\
&\quad+\;
\frac{2(1+\gamma^{2}\left\| \mGamma_\eta \right\|^2_\infty)\,l_\eta}{\mu_\eta^{2}(1-\gamma^{2}\left\| \mGamma_\eta \right\|^2_\infty)}\,\alpha\,g_{3,\eta}.
\end{align*}
\end{lemma}

\begin{proof}
Let 
\[
y_t := \E\!\left[\|\mPhi\vtheta_{t,K}-\mPhi\vtheta^*_\eta\|_\infty^2\right],
\qquad 
y_{t-1} := \E\!\left[\|\mPhi\vtheta_{t-1,K}-\mPhi\vtheta^*_\eta\|_\infty^2\right].
\]

\noindent From Proposition~\ref{prop:outer-loop-decomposition},

\begin{align*}
y_t 
&\le \frac{2 (1+\delta)}{\mu_{\eta}}\left(\E\left[ L_{\eta}(\vtheta_{t,K},\vtheta_{t-1,K})-L_\eta(\vtheta^*(\vtheta_{t-1,K}),\vtheta_{t-1,K}) \right] \right)+ \gamma^2\left\| \mGamma_\eta \right\|^2_\infty (1+\delta^{-1})y_{t-1}.
\end{align*}

\noindent With \(\displaystyle \delta=\frac{2\gamma^{2}\left\| \mGamma_\eta \right\|^2_\infty}{1-\gamma^{2}\left\| \mGamma_\eta \right\|^2_\infty}\), we have
\[
1+\delta=\frac{1+\gamma^{2}\left\| \mGamma_\eta \right\|^2_\infty}{1-\gamma^{2}\left\| \mGamma_\eta \right\|^2_\infty},
\qquad
\gamma^{2}\left\| \mGamma_\eta \right\|^2_\infty\bigl(1+\delta^{-1}\bigr)=\frac{1+\gamma^{2}\left\| \mGamma_\eta \right\|^2_\infty}{2}.
\]
Hence,
\[
y_t 
\;\le\;
\frac{2(1+\gamma^{2}\left\| \mGamma_\eta \right\|^2_\infty)}{\mu_{\eta}(1-\gamma^{2}\left\| \mGamma_\eta \right\|^2_\infty)}\,
\E\!\left[ L_{\eta}(\vtheta_{t,k},\vtheta_{t-1,K})-L_\eta(\vtheta^*(\vtheta_{t-1,K}),\vtheta_{t-1,K}) \right]
\;+\;
\frac{1+\gamma^{2}\left\| \mGamma_\eta \right\|^2_\infty}{2}\,y_{t-1}.
\]

\noindent Using Lemma~\ref{lem:k-stepinnerbound}, replacing 
\(\E\!\Bigl[L_{\eta}(\vtheta_{t,K},\vtheta_{t-1,K})
   - L_\eta(\vtheta^*(\vtheta_{t-1,K}),\vtheta_{t-1,K})\Bigr]\) gives
\begin{align}
y_t 
&\le
\frac{2(1+\gamma^{2}\left\| \mGamma_\eta \right\|^2_\infty)}{\mu_{\eta}(1-\gamma^{2}\left\| \mGamma_\eta \right\|^2_\infty)}
\Bigg(
\Bigl(1-\tfrac{\mu_\eta}{2}\alpha\Bigr)^{K}
\E\!\Bigl[
L_\eta(\vtheta_{t,0},\vtheta_{t-1,K})
- L_\eta(\vtheta^*(\vtheta_{t-1,K}),\vtheta_{t-1,K})
\Bigr] 
+
\frac{l_\eta}{\mu_\eta}\,\alpha
\Big[
g_{2,\eta}\,y_{t-1}
+ g_{3,\eta}
\Big]
\Bigg)\nonumber\\
&+
\frac{1+\gamma^{2}\left\| \mGamma_\eta \right\|^2_\infty}{2}\,y_{t-1}.
\label{eq:yt-after-kstep}
\end{align}

\noindent Next, applying Lemma~\ref{lem:init_lossgap} yields
\begin{align}
y_t 
&\le
\frac{2(1+\gamma^{2}\left\| \mGamma_\eta \right\|^2_\infty)}{\mu_{\eta}(1-\gamma^{2}\left\| \mGamma_\eta \right\|^2_\infty)}
\Bigg(
\Bigl(1-\tfrac{\mu_\eta}{2}\alpha\Bigr)^{K}
\Big[
8\,y_{t-1}
+ R_{\max}^2
+ 8\,\big\|\mPhi\vtheta^*_\eta\big\|_\infty^2
\Big]
+
\frac{l_\eta}{\mu_\eta}\,\alpha
\Big[
g_{2,\eta}\,y_{t-1}
+ g_{3,\eta}
\Big]
\Bigg)
+
\frac{1+\gamma^{2}\left\| \mGamma_\eta \right\|^2_\infty}{2}\,y_{t-1} \nonumber\\[6pt]
&=
\Bigg[
\frac{16(1+\gamma^{2}\left\| \mGamma_\eta \right\|^2_\infty)}{\mu_{\eta}(1-\gamma^{2}\left\| \mGamma_\eta \right\|^2_\infty)}
\Bigl(1-\tfrac{\mu_\eta}{2}\alpha\Bigr)^{K}
\;+\;
\frac{2(1+\gamma^{2}\left\| \mGamma_\eta \right\|^2_\infty)\,l_\eta}{\mu_\eta^{2}(1-\gamma^{2}\left\| \mGamma_\eta \right\|^2_\infty)}\,\alpha\,g_{2,\eta}
\;+\;
\frac{1+\gamma^{2}\left\| \mGamma_\eta \right\|^2_\infty}{2}
\Bigg] y_{t-1} \nonumber\\
&\quad+\;
\frac{2(1+\gamma^{2}\left\| \mGamma_\eta \right\|^2_\infty)}{\mu_{\eta}(1-\gamma^{2}\left\| \mGamma_\eta \right\|^2_\infty)}
\Bigl(1-\tfrac{\mu_\eta}{2}\alpha\Bigr)^{K} R_{\max}^2
\;+\;
\frac{16(1+\gamma^{2}\left\| \mGamma_\eta \right\|^2_\infty)}{\mu_{\eta}(1-\gamma^{2}\left\| \mGamma_\eta \right\|^2_\infty)}
\Bigl(1-\tfrac{\mu_\eta}{2}\alpha\Bigr)^{K}
\big\|\mPhi\vtheta^*_\eta\big\|_\infty^2 \nonumber\\
&\quad+\; \frac{2(1+\gamma^{2}\left\| \mGamma_\eta \right\|^2_\infty)\,l_\eta}{\mu_\eta^{2}(1-\gamma^{2}\left\| \mGamma_\eta \right\|^2_\infty)}\,\alpha\,g_{3,\eta}.
\label{eq:yt-full-chain}
\end{align}

\noindent This concludes the proof and establishes the desired result.
\end{proof}

\subsection{Proof of Theorem~\ref{thm:epsilon-accuracy-main}}
\label{thm:epsilon-accuracy}

\begin{proof} Let 
\[
y_t := \E\!\left[\|\mPhi\vtheta_{t,K}-\mPhi\vtheta^*_\eta\|_\infty^2\right],
\qquad 
y_{t-1} := \E\!\left[\|\mPhi\vtheta_{t-1,K}-\mPhi\vtheta^*_\eta\|_\infty^2\right].
\]
Fix $\displaystyle \delta=\frac{2\gamma^{2}\left\| \mGamma_\eta \right\|^2_\infty}{1-\gamma^{2}\left\| \mGamma_\eta \right\|^2_\infty}$
and assume
$\bigg( 0 < \alpha \le \min\!\left\{\frac{3\mu_\eta}{l_\eta g_{1,\eta}},\,\frac{2}{\mu_\eta}\right\} \bigg)$.

\noindent From Lemma~\ref{lem:main_recur}, we have
\begin{align}
\label{eq:mainrecur}
y_t 
&\le
\Bigg[
\frac{16(1+\gamma^{2}\left\| \mGamma_\eta \right\|^2_\infty)}{\mu_{\eta}(1-\gamma^{2}\left\| \mGamma_\eta \right\|^2_\infty)}
\Bigl(1-\tfrac{\mu_\eta}{2}\alpha\Bigr)^{K}
+
\frac{2(1+\gamma^{2}\left\| \mGamma_\eta \right\|^2_\infty)\,l_\eta}{\mu_\eta^{2}(1-\gamma^{2}\left\| \mGamma_\eta \right\|^2_\infty)}\,\alpha\,g_{2,\eta}
+
\frac{1+\gamma^{2}\left\| \mGamma_\eta \right\|^2_\infty}{2}
\Bigg] y_{t-1} \nonumber\\
&\quad+\;
\frac{2(1+\gamma^{2}\left\| \mGamma_\eta \right\|^2_\infty)}{\mu_{\eta}(1-\gamma^{2}\left\| \mGamma_\eta \right\|^2_\infty)}
\Bigl(1-\tfrac{\mu_\eta}{2}\alpha\Bigr)^{K} R_{\max}^2
+\frac{16(1+\gamma^{2}\left\| \mGamma_\eta \right\|^2_\infty)}{\mu_{\eta}(1-\gamma^{2}\left\| \mGamma_\eta \right\|^2_\infty)}
\Bigl(1-\tfrac{\mu_\eta}{2}\alpha\Bigr)^{K}
\|\mPhi\vtheta^*_\eta\|_\infty^2 \nonumber \\
&\quad+\;\frac{2(1+\gamma^{2}\left\| \mGamma_\eta \right\|^2_\infty)\,l_\eta}{\mu_\eta^{2}(1-\gamma^{2}\left\| \mGamma_\eta \right\|^2_\infty)}\,\alpha\,g_{3,\eta}.
\end{align}

Let us define, for convenience,
\begin{equation}
\label{eq:residue}
\gE_{K,\alpha}
:=
\frac{2(1+\gamma^{2}\left\| \mGamma_\eta \right\|^2_\infty)}{\mu_{\eta}(1-\gamma^{2}\left\| \mGamma_\eta \right\|^2_\infty)}
\Bigl(1-\tfrac{\mu_\eta}{2}\alpha\Bigr)^{K}
\!\Big(R_{\max}^2 + 8\|\mPhi\vtheta^*_\eta\|_\infty^2\Big)
+\frac{2(1+\gamma^{2}\left\| \mGamma_\eta \right\|^2_\infty)\,l_\eta}{\mu_\eta^{2}(1-\gamma^{2}\left\| \mGamma_\eta \right\|^2_\infty)}\,\alpha\,g_{3,\eta}.
\end{equation}

\noindent so that the recursion can be compactly written as
\begin{equation}
\label{eq:contraction-y}
y_t \;\le\;
\Bigg[
\frac{16(1+\gamma^{2}\left\| \mGamma_\eta \right\|^2_\infty)}{\mu_{\eta}(1-\gamma^{2}\left\| \mGamma_\eta \right\|^2_\infty)}
\Bigl(1-\tfrac{\mu_\eta}{2}\alpha\Bigr)^{K}
+
\frac{2(1+\gamma^{2}\left\| \mGamma_\eta \right\|^2_\infty)\,l_\eta}{\mu_\eta^{2}(1-\gamma^{2}\left\| \mGamma_\eta \right\|^2_\infty)}\,\alpha\,g_{2,\eta}
+
\frac{1+\gamma^{2}\left\| \mGamma_\eta \right\|^2_\infty}{2}
\Bigg] y_{t-1}
\;+\;
\gE_{K,\alpha}.
\end{equation}

We make the coefficient of $y_{t-1}$ in (\ref{eq:mainrecur}) strictly smaller than $1$ by choosing $K$ large enough. It suffices to ensure

\[
\frac{16(1+\gamma^{2}\left\| \mGamma_\eta \right\|^2_\infty)}{\mu_{\eta}(1-\gamma^{2}\left\| \mGamma_\eta \right\|^2_\infty)}
\Bigl(1-\tfrac{\mu_\eta}{2}\alpha\Bigr)^{K}
\;+\;
\frac{2(1+\gamma^{2}\left\| \mGamma_\eta \right\|^2_\infty)\,l_\eta}{\mu_\eta^{2}(1-\gamma^{2}\left\| \mGamma_\eta \right\|^2_\infty)}\,\alpha\,g_{2,\eta}
\;+\;
\frac{1+\gamma^{2}\left\| \mGamma_\eta \right\|^2_\infty}{2}
\;\le\;
1-\frac{1-\gamma^{2}\left\| \mGamma_\eta \right\|^2_\infty}{4}  = \frac{3+\gamma^{2}\left\| \mGamma_\eta \right\|^2_\infty}{4}.
\]

To guarantee this bound, it suffices to allocate half of the available margin 
$\tfrac{1-\gamma^2 \left\| \mGamma_\eta \right\|^2_\infty}{4}$ to each of the first two terms,
\[
\begin{cases}
\displaystyle 
\frac{16(1+\gamma^{2}\left\| \mGamma_\eta \right\|^2_\infty)}{\mu_{\eta}(1-\gamma^{2}\left\| \mGamma_\eta \right\|^2_\infty)}
\Bigl(1-\tfrac{\mu_\eta}{2}\alpha\Bigr)^{K}
\;\le\; \frac{1-\gamma^2\left\| \mGamma_\eta \right\|^2_\infty}{8}, \\[10pt]
\displaystyle 
\frac{2(1+\gamma^{2}\left\| \mGamma_\eta \right\|^2_\infty)\,l_\eta}{\mu_\eta^{2}(1-\gamma^{2}\left\| \mGamma_\eta \right\|^2_\infty)}\,\alpha\,g_{2,\eta}
\;\le\; \frac{1-\gamma^2\left\| \mGamma_\eta \right\|^2_\infty}{8}.
\end{cases}
\]
The second condition yields an explicit upper bound on $\alpha$,
\begin{equation}
\alpha \;\le\;
\frac{\mu_\eta^{2}(1-\gamma^{2}\left\| \mGamma_\eta \right\|^2_\infty)^{2}}
{16(1+\gamma^{2}\left\| \mGamma_\eta \right\|^2_\infty)\,l_\eta\,g_{2,\eta}}.\label{alphabound2}
\end{equation}

\noindent Substituting this into the first inequality then specifies the required 
lower bound on $K$,
\[
K \;\ge\;
\frac{\ln\!\Big(\tfrac{128(1+\gamma^{2}\left\| \mGamma_\eta \right\|^2_\infty)}{\mu_\eta(1-\gamma^{2}\left\| \mGamma_\eta \right\|^2_\infty)^{2}}\Big)}
{-\ln\!\Big(1-\tfrac{\mu_\eta}{2}\alpha\Big)}.
\]
These two design constraints ensure the desired contraction condition.

\noindent Using the inequality $-\ln(1-x)\ge x$ for $x\in(0,1)$, we further obtain
\begin{equation}
K
\;\ge\;
\frac{2}{\mu_\eta\,\alpha}\,
\ln\!\Big(\frac{128(1+\gamma^{2}\left\| \mGamma_\eta \right\|^2_\infty)}{\mu_\eta(1-\gamma^{2}\left\| \mGamma_\eta \right\|^2_\infty)^{2}}\Big).
\label{kbound1}
\end{equation}

\noindent Substituting the contraction condition derived above into the recursion in~(\ref{eq:contraction-y}), 
we obtain
\begin{equation}
y_t \;\le\; \Bigl(1-\tfrac{1-\gamma^{2}\left\| \mGamma_\eta \right\|^2_\infty}{4}\Bigr)\,y_{t-1} \;+\; \gE_{K,\alpha}.
\label{eq:one-step}
\end{equation}
Let \(a:=1-\tfrac{1-\gamma^{2}\left\| \mGamma_\eta \right\|^2_\infty}{4}=\tfrac{3+\gamma^{2}\left\| \mGamma_\eta \right\|^2_\infty}{4}\in(0,1)\).
Iterating (\ref{eq:one-step}) yields
\begin{align*}
y_t
\:\le\; \left(\frac{3+\gamma^2\left\| \mGamma_\eta \right\|^2_\infty}{4}\right)^t y_0 \;+\; \frac{1}{1-a}\,\gE_{K,\alpha}.
\end{align*}
Since \(1-a=\tfrac{1-\gamma^2\left\| \mGamma_\eta \right\|^2_\infty}{4}\), we conclude that
\begin{equation}
y_t
\;\le\;
\left(\frac{3+\gamma^2\left\| \mGamma_\eta \right\|^2_\infty}{4}\right)^t y_0
\;+\;
\frac{4}{1-\gamma^2\left\| \mGamma_\eta \right\|^2_\infty}\,\gE_{K,\alpha},
\label{eq:geom-bound}
\end{equation}

\noindent It remains to make the geometric term at most $\epsilon/2$, i.e.,
\[
\left(\frac{3+\gamma^2\left\| \mGamma_\eta \right\|^2_\infty}{4}\right)^t y_0 \;\le\; \frac{\epsilon}{2}.
\]
Taking logarithms gives
\[
t \;\ge\; \frac{\ln(2y_0/\epsilon)}{-\ln a}.
\]
Since $-\ln a \;\ge\; 1-a = \tfrac{1-\gamma^2\left\| \mGamma_\eta \right\|^2_\infty}{4}$, we have
\[
\frac{1}{-\ln a} \;\le\; \frac{4}{1-\gamma^2\left\| \mGamma_\eta \right\|^2_\infty}.
\]
Therefore, a sufficient condition is
\begin{equation}
t \;\ge\; \frac{4}{1-\gamma^2\left\| \mGamma_\eta \right\|^2_\infty}\,\ln\!\Bigg(\frac{2y_0}{\epsilon}\Bigg).
\label{eq:t-lb}
\end{equation}
In addition, to ensure the steady-state residue is at most $\epsilon/2$, it suffices to require
\begin{equation}
\frac{4}{1-\gamma^2\left\| \mGamma_\eta \right\|^2_\infty}\,\gE_{K,\alpha} \;\le\; \frac{\epsilon}{2}
\quad\Longleftrightarrow\quad
\gE_{K,\alpha} \;\le\; \frac{1-\gamma^2\left\| \mGamma_\eta \right\|^2_\infty}{8}\,\epsilon,
\label{eq:residue-lb}
\end{equation}
where $\gE_{K,\alpha}$ is defined in~(\ref{eq:residue}).

From (\ref{eq:residue-lb}), it suffices to make each term in (\ref{eq:residue}) smaller than \(\tfrac{1-\gamma^2\left\| \mGamma_\eta \right\|^2_\infty}{16}\epsilon\):
\[
\frac{2(1+\gamma^{2}\left\| \mGamma_\eta \right\|^2_\infty)}{\mu_{\eta}(1-\gamma^{2}\left\| \mGamma_\eta \right\|^2_\infty)}
\Bigl(1-\tfrac{\mu_\eta}{2}\alpha\Bigr)^{K}
\!\Big(R_{\max}^2 + 8\|\mPhi\vtheta^*_\eta\|_\infty^2\Big)
\;\le\;
\frac{1-\gamma^2\left\| \mGamma_\eta \right\|^2_\infty}{16}\,\epsilon,
\qquad
\]
\[
\frac{2(1+\gamma^{2}\left\| \mGamma_\eta \right\|^2_\infty)\,l_\eta}{\mu_\eta^{2}(1-\gamma^{2}\left\| \mGamma_\eta \right\|^2_\infty)}\,\alpha\,g_{3,\eta}
\;\le\;
\frac{1-\gamma^2\left\| \mGamma_\eta \right\|^2_\infty}{16}\,\epsilon.
\]
This allocation is sufficient to guarantee \(\gE_{K,\alpha}\le\tfrac{1-\gamma^2\left\| \mGamma_\eta \right\|^2_\infty}{8}\epsilon.\)

\noindent From the two sufficient inequalities above, we can derive explicit complexity bounds for \(K\) and \(\alpha\).

\medskip
\noindent\textbf{(a) Bound on \(K\).}
From the first inequality,
\[
\frac{2(1+\gamma^{2}\left\| \mGamma_\eta \right\|^2_\infty)}{\mu_{\eta}(1-\gamma^{2}\left\| \mGamma_\eta \right\|^2_\infty)}
\Bigl(1-\tfrac{\mu_\eta}{2}\alpha\Bigr)^{K}
\!\Big(R_{\max}^2 + 8\|\mPhi\vtheta^*_\eta\|_\infty^2\Big)
\;\le\;
\frac{1-\gamma^2\left\| \mGamma_\eta \right\|^2_\infty}{16}\,\epsilon.
\]
Rearranging gives
\[
\Bigl(1-\tfrac{\mu_\eta}{2}\alpha\Bigr)^{K}
\;\le\;
\frac{\mu_\eta(1-\gamma^{2}\left\| \mGamma_\eta \right\|^2_\infty)^{2}}{32(1+\gamma^{2}\left\| \mGamma_\eta \right\|^2_\infty)\big(R_{\max}^2 + 8\|\mPhi\vtheta^*_\eta\|_\infty^2\big)}\,\epsilon.
\]
Taking logarithms on both sides yields
\[
K
\;\ge\;
\frac{\ln\!\Big(\tfrac{32(1+\gamma^{2}\left\| \mGamma_\eta \right\|^2_\infty)\big(R_{\max}^2 + 8\|\mPhi\vtheta^*_\eta\|_\infty^2\big)}{\mu_\eta(1-\gamma^{2}\left\| \mGamma_\eta \right\|^2_\infty)^{2}\epsilon}\Big)}
{-\ln\!\big(1-\tfrac{\mu_\eta}{2}\alpha\big)}.
\]

Using the inequality \(-\ln(1-x)\ge x\) for \(x\in(0,1)\), we further have
\begin{equation}
\label{eq:K-bound-residue}
K
\;\ge\;
\frac{2}{\mu_\eta\,\alpha}
\ln\!\Big(
\frac{32(1+\gamma^{2}\left\| \mGamma_\eta \right\|^2_\infty)\big(R_{\max}^2 + 8\|\mPhi\vtheta^*_\eta\|_\infty^2\big)}{\mu_\eta(1-\gamma^{2}\left\| \mGamma_\eta \right\|^2_\infty)^{2}\epsilon}
\Big).
\end{equation}

\medskip
\noindent\textbf{(b) Bound on \(\alpha\).}
From the second inequality,
\[
\frac{2(1+\gamma^{2}\left\| \mGamma_\eta \right\|^2_\infty)\,l_\eta}{\mu_\eta^{2}(1-\gamma^{2}\left\| \mGamma_\eta \right\|^2_\infty)}\,\alpha\,g_{3,\eta}
\;\le\;
\frac{1-\gamma^2\left\| \mGamma_\eta \right\|^2_\infty}{16}\,\epsilon,
\]
which directly gives
\begin{equation}
\alpha
\;\le\;
\frac{\mu_\eta^{2}(1-\gamma^{2}\left\| \mGamma_\eta \right\|^2_\infty)^{2}}{32(1+\gamma^{2}\left\| \mGamma_\eta \right\|^2_\infty)\,l_\eta\,g_{3,\eta}}\,\epsilon.
\label{eq:alpha-bound-residue}
\end{equation}

\medskip
Combining (\ref{eq:K-bound-residue}) and (\ref{eq:alpha-bound-residue}), one obtains the sufficient conditions on \((\alpha,K)\) ensuring
\(\gE_{K,\alpha}\le\tfrac{1-\gamma^{2}\left\| \mGamma_\eta \right\|^2_\infty}{8}\epsilon\),
and consequently
\(y_t\le\epsilon\)
for \(t\) satisfying (\ref{eq:t-lb}).

Collecting the step-size conditions from Lemma~\ref{lem:1stepk}, (\ref{alphabound2}), and (\ref{eq:alpha-bound-residue}), define
\[
\bar\alpha_1:=\frac{2}{\mu_\eta},\qquad
\bar\alpha_2:=\frac{3\mu_\eta}{l_\eta g_{1,\eta}},\qquad
\bar\alpha_3:=\frac{\mu_\eta^{2}(1-\gamma^{2}\left\| \mGamma_\eta \right\|^2_\infty)^{2}}{16(1+\gamma^{2}\left\| \mGamma_\eta \right\|^2_\infty)\,l_\eta\,g_{2,\eta}},\qquad
\bar\alpha_4:=\frac{\mu_\eta^{2}(1-\gamma^{2}\left\| \mGamma_\eta \right\|^2_\infty)^{2}}{32(1+\gamma^{2}\left\| \mGamma_\eta \right\|^2_\infty)\,l_\eta\,g_{3,\eta}}\,\epsilon,
\]
and set
\[
{\bar\alpha} := \;\min\{\bar\alpha_1,\bar\alpha_2,\bar\alpha_3,\bar\alpha_4\}.
\]
Then it suffices to choose \(0<\alpha\le {\bar\alpha}\).
These four components correspond precisely to \(\bar\alpha_1,\bar\alpha_2,\bar\alpha_3,\bar\alpha_4\).

\medskip

\noindent Similarly, gathering the bounds on $K$ from~(\ref{kbound1}),~(\ref{eq:K-bound-residue}),
we obtain
\begin{align*}
    K
    \ge
    \max\!\Bigg\{
    \frac{2}{\mu_\eta\,\alpha}\,
    \ln\!\Bigg(
        \frac{32(1+\gamma^{2}\left\| \mGamma_\eta \right\|^2_\infty)
        \big(R_{\max}^2 + 8\|\mPhi\vtheta^*_\eta\|_\infty^2\big)}
        {\mu_\eta(1-\gamma^{2}\left\| \mGamma_\eta \right\|^2_\infty)^{2}\epsilon}     
    \Bigg), 
    \frac{2}{\mu_\eta\,\alpha}\,
    \ln\!\Bigg(\frac{128(1+\gamma^{2}\left\| \mGamma_\eta \right\|^2_\infty)}{\mu_\eta(1-\gamma^{2}\left\| \mGamma_\eta \right\|^2_\infty)^{2}}\Bigg)
    \Bigg\}
\end{align*}

Replacing $\alpha$ with its asymptotically minimal bound
$\alpha \asymp \frac{\mu_\eta^{2}(1-\gamma^{2}\left\| \mGamma_\eta \right\|^2_\infty)^{2}}{(1+\gamma^{2}\left\| \mGamma_\eta \right\|^2_\infty)\,l_\eta\,g_{3,\eta}}\;\epsilon$
gives the $\alpha$–free form
\[
K
\ge
\frac{2(1+\gamma^{2}\left\| \mGamma_\eta \right\|^2_\infty)\,l_\eta\,g_{3,\eta}}
     {\mu_\eta^{3}(1-\gamma^{2}\left\| \mGamma_\eta \right\|^2_\infty)^{2}\,\epsilon}\;
\max\!\left\{
\ln\!\frac{32(1+\gamma^{2}\left\| \mGamma_\eta \right\|^2_\infty)\big(R_{\max}^2+8\|\mPhi\vtheta^*_\eta\|_\infty^2\big)}
         {\mu_\eta(1-\gamma^{2}\left\| \mGamma_\eta \right\|^2_\infty)^{2}\epsilon},
\;
\ln\!\frac{128(1+\gamma^{2}\left\| \mGamma_\eta \right\|^2_\infty)}{\mu_\eta(1-\gamma^{2}\left\| \mGamma_\eta \right\|^2_\infty)^{2}}
\right\}
.
\]

Since $g_{3,\eta}
= 32(1+\eta)\gamma^2\|\mPhi\vtheta^*_{\eta}\|_\infty^2+(16+16\eta)R_{\max}^2+8\sigma_{\eta}^2$
depends on $\|\vtheta^*_\eta\|^2_{\infty}$, absorbing these constants into the complexity,
there exists a choice of iteration numbers of the form
\[
K \;=\; \gO\!\left(\frac{l_\eta\,\|\vtheta^*_{\eta}\|^2_2}
{\epsilon\mu_\eta^{3}(1-\gamma\left\| \mGamma_\eta \right\|_\infty)^{2}}\right),
\qquad
t \;=\; \gO\!\left(\frac{1}{1-\gamma\left\| \mGamma_\eta \right\|_\infty}\right),
\]
for which the desired accuracy guarantee holds.

\end{proof}


\section{Analysis under the Markovian observation model}\label{app:sec:markov}

In this section, we present a detailed analysis and establish the convergence rate under the Markovian observation model introduced in Section~\ref{sec:markovian_obs_model}.

\subsection{Markov chain and Poisson Equation}

For the analysis of the Markovian observation model in Section~\ref{sec:markovian_obs_model}, we introduce the so-called Poisson's equation. The Poisson equation~\citep{glynn1996liapounov} serves as a fundamental tool in the study of Markov chains and has been utilized in various works, including~\citet{haque2024stochastic}, for the analysis of stochastic approximation schemes. Following the approach of~\citet{haque2024stochastic}, we leverage this framework to establish our results.

Let $\{(S_k,A_k)\}_{k=0}^{\infty}$ be a sequence of random variables induced by the irreducible Markov chain with behavior policy $\beta$ in Section~\ref{sec:markovian_obs_model}. Then, for some functions $\varphi,\psi:\gS\times\gA\to\R$, the Poisson's equation is defined as
\begin{align*}
    \E\left[ \psi(S_{1},A_{1})  \middle | (S_0,A_0)=(s,a)\right] - \psi(s,a) = -  \varphi(s,a)
\end{align*}
 Given $\varphi$, a candidate solution for $\psi$ is $\E\left[ \sum^{\tau(\tilde{s},\tilde{a})-1}_{k=0} \varphi(S_k,A_k)\middle | (S_0,A_0)=(s,a)\right]$  where $\tau(\tilde{s},\tilde{a})=\inf\{ n \geq 1 : (S_n, A_n)=(\tilde{s},\tilde{a}) \}$ is a hitting time for some $(\tilde{s},\tilde{a})\in\gS\times\gA$.

\subsection{Main Analysis}

First, we define two key quantities used throughout the analysis.
First, let
\begin{align}
    \Bar{g}(\vtheta,\vtheta^{\prime};s,a) :=\sum_{s^{\prime}\in\gS} \gP(s^{\prime} \mid s,a )  g(\vtheta,\vtheta^{\prime};s,a,s^{\prime}), \label{def:bar_g}
\end{align}
and
\begin{equation}\label{def:poisson-sol}
    V(\vtheta,\vtheta^{\prime},s,a) = \E\left[  \sum^{\tau(\tilde{s},\tilde{a})-1}_{k=0} 
    \bar{g}(\vtheta,\vtheta^{\prime};S_{k},A_{k})-\nabla L_{\eta}(\vtheta,\vtheta^{\prime})\middle| (S_0,A_0)=(s,a)\right].
\end{equation}
 For simplicity, let us denote $\tau=\tau(\tilde{s},\tilde{a})$. With a slight abuse of notation, we define $L_{\eta}$ in~(\ref{eq:reg-mspbe}) by taking $d$ to be the stationary distribution $\mu_{\infty}$.

\begin{lemma}\label{lem:poisson-eq}
Consider the sequence of random variables $\{(S_k,A_k)\}_{k=0}^{\infty}$ induced by the Markov chain. Then, for $\vtheta,\vtheta^{\prime}\in\R^h$, the following equation holds : 
    \begin{align*}
         & V(\vtheta,\vtheta^{\prime},s,a)-\E\left[ V(\vtheta,\vtheta^{\prime},S_{1},A_{1})  \middle | (S_0,A_0)=(s,a) \right] =\bar{g}(\vtheta,\vtheta^{\prime};s,a)- \nabla L_{\eta}(\vtheta,\vtheta^{\prime}).
    \end{align*}
\end{lemma}

\begin{proof}
    From the definition of $V$ in~(\ref{def:poisson-sol}), we have
\begin{align*}
&V(\vtheta,\vtheta^{\prime},s,a)-\E\left[ V(\vtheta,\vtheta^{\prime},S_{1},A_{1})  \middle |(S_0,A_0)=(s,a) \right]\\
=&  \bar{g}(\vtheta,\vtheta^{\prime};s,a)- \nabla L_{\eta}(\vtheta,\vtheta^{\prime})+\E\left[ \bm{1}\{\tau\geq2\}\left(  \sum^{\tau-1}_{k=1} 
    \bar{g}(\vtheta,\vtheta^{\prime};S_{k},A_{k})-\nabla L_{\eta}(\vtheta,\vtheta^{\prime})\right)\middle| (S_0,A_0)=(s,a)\right]\\
    &-\E\left[ \E\left[\sum^{\tilde{\tau}-1}_{k=0}\bar{g}(\vtheta,\vtheta^{\prime};\tilde{S}_{k},\tilde{A}_{k})-\nabla L_{\eta}(\vtheta, \vtheta^{\prime})  \middle |(\tilde{S}_0,\tilde{A}_0)=(S_{1},A_{1}) \right] \middle | (S_0,A_0)=(s,a)\right] \\
    =& \bar{g}(\vtheta,\vtheta^{\prime};s,a)- \nabla L_{\eta}(\vtheta,\vtheta^{\prime}).
\end{align*}
where $\tilde{\tau}$ is the hitting time defined by a sequence of random variables $\{(\tilde{S}_k, \tilde{A}_k)\}_{k=0}^{\infty}$ induced by the Markov chain. The second equality follows from the fact that conditioned on $(\tilde{S}_0,\tilde{A}_0)=(S_1,A_1)$, $\tilde{\tau}$ follows the same law of distribution of $\tau$ for $\tau\geq 2$ and $V(\vtheta,\vtheta^{\prime},\tilde{s},\tilde{a})=0$.

\end{proof}

Now, let us provide several useful properties related to the solution of Poisson's equation, $V$:

\begin{lemma}\label{lem:V-theta-eta-bound}
    For $(s,a)\in\gS\times\gA$, we have
    \begin{align*}
        \left\| V(\vtheta^*_{\eta},\vtheta^*_{\eta},s,a) \right\|_2 \leq  \tau_{\max} \left( R_{\max}+(1+\gamma+\eta)\left\| \vtheta^*_{\eta} \right\|_2 \right).
    \end{align*}
\end{lemma}
\begin{proof}
 We have
 \begin{align*}
             \left\| V(\vtheta^*_{\eta},\vtheta^*_{\eta},s,a) \right\|_2 =& \left\| \E\left[ \sum^{\tau-1}_{k=0} \bar{g}(\vtheta^*_{\eta},\vtheta^*_{\eta};S_k,A_k) \middle | (S_0,A_0)=(s,a) \right]\right\|_2\\
             =& \E\left[ \sum^{\tau-1}_{k=0}(R_{\max}+ (1+\gamma) \left\|\vtheta^*_{\eta} \right\|_2)+ \eta\left\| \vtheta^*_{\eta} \right\|_2 \middle | (S_0,A_0)=(s,a)\right] \\
             \leq & \tau_{\max} \left( R_{\max}+(1+\gamma+\eta)\left\| \vtheta^*_{\eta} \right\|_2 \right).
 \end{align*}
The second equality follows from the definition of $\bar{g}$ in~(\ref{def:bar_g}). This completes the proof.
\end{proof}

\begin{lemma}[Properties of $V$]\label{lem:poisson-property}
For $\vx,\vy,\vtheta^{\prime}\in\R^h$ and $(s,a)\in\gS\times\gA$, we have
    \begin{align*}
       \left\| V(\vx,\vtheta^{\prime},s,a)-V(\vy,\vtheta^{\prime},s,a)  \right\|_2 \leq &  l_{V_1} \left\| \vx-\vy  \right\|_2,\\
               \left\| V(\vx,\vtheta^{\prime},s,a)-V(\vx,\vtheta,s,a)  \right\|_2 \leq & l_{V_2} \left\| \mPhi\vtheta-\mPhi\vtheta^{\prime}\right\|_{\infty},\\
                   \left\|V(\vtheta,\vtheta^{\prime},s,a) \right\|_2\leq & l_{V_1}\left\|\vtheta-\vtheta^*_{\eta} \right\|_2+l_{V_2}\left\| \mPhi\vtheta^{\prime}-\mPhi\vtheta^*_{\eta} \right\|_{\infty} +l_{V_3}.
    \end{align*}
\end{lemma}
\begin{proof}
        The definition of Poisson solution in~(\ref{def:poisson-sol}) yields
    \begin{align*}
&V(\vx,\vtheta^{\prime},s,a)-V(\vy,\vtheta^{\prime},s,a)\\
       =&\left\|  \E\left[ \sum^{\tau-1}_{k=0}\bar{g}(\vx,\vtheta^{\prime};S_k,A_k)- \bar{g}(\vy,\vtheta^{\prime};S_k,A_k)  -\nabla L_{\eta}(\vx,\vtheta^{\prime})+\nabla L_{\eta}(\vy,\vtheta^{\prime}) \middle| (S_0,A_0)=(s,a) \right] \right\|_2\\
      \leq & \E\left[ \left\| \sum^{\tau-1}_{k=0}\bar{g}(\vx,\vtheta^{\prime};S_k,A_k)- \bar{g}(\vy,\vtheta^{\prime};S_k,A_k)  -\nabla L_{\eta}(\vx,\vtheta^{\prime})+\nabla L_{\eta}(\vy,\vtheta^{\prime}) \right\|_2 \middle| (S_0,A_0)=(s,a)  \right]\\
      \leq &  \E\left[ \sum^{\tau-1}_{k=0}\left\|\bar{g}(\vx,\vtheta^{\prime};S_k,A_k)- \bar{g}(\vy,\vtheta^{\prime};S_k,A_k) \right\|_2 +\left\|\nabla L_{\eta}(\vx,\vtheta^{\prime})-\nabla L_{\eta}(\vy,\vtheta^{\prime}) \right\|_2 \middle| (S_0,A_0)=(s,a)  \right]\\
      \leq & \E\left[\tau\right] (1+\eta) \left\| \vx-\vy \right\|_2 + \E[\tau](\lambda_{\max}(\mPhi^{\top}\mD\mPhi)+\eta) \left\| \vx-\vy  \right\|_2.
    \end{align*}
    The last inequality follows from Lemma~\ref{lem:phi-lipschitz} and Lemma~\ref{lem:strong_convexity}. 
    
    The second statement follows by the same reasoning as in the preceding proof:
    \begin{align*}
            & \left\| V(\vx,\vtheta^{\prime},s,a)-V(\vx,\vtheta,s,a)  \right\|_2 \\=&  \left\| \E\left[ \sum^{\tau-1}_{k=0}\bar{g}(\vx,\vtheta^{\prime};S_k,A_k)- \bar{g}(\vx,\vtheta;S_k,A_k)  -\nabla L_{\eta}(\vx,\vtheta^{\prime})+\nabla L_{\eta}(\vx,\vtheta) \middle| (S_0,A_0)=(s,a) \right] \right\|_2\\
            \leq & \E\left[ \sum^{\tau-1}_{k=0} \left\|\bar{g}(\vx,\vtheta^{\prime};S_k,A_k)- \bar{g}(\vx,\vtheta;S_k,A_k) \right\|_2  +\left\|\nabla L_{\eta}(\vx,\vtheta^{\prime})-\nabla L_{\eta}(\vx,\vtheta) \right\|_2 \middle| (S_0,A_0)=(s,a) \right] \\
            \leq & 2\tau_{\max} \gamma \left\| \mPhi\vtheta-\mPhi\vtheta^{\prime}\right\|_{\infty}.
    \end{align*}

The last inequality follows from Lemma~\ref{lem:phi-lipschitz} in the Appendix.
    
    The last statement follows from the following:
    \begin{align*}
        \left\| V(\vtheta,\vtheta^{\prime},s,a) \right\|_2 \leq & \left\| V(\vtheta,\vtheta^{\prime},s,a)-V(\vtheta^*_{\eta},\vtheta^{*}_{\eta},s,a)\right\|_2+\left\| V(\vtheta^*_{\eta},\vtheta^{*}_{\eta},s,a) \right\|_2\\
        \leq & \left\| V(\vtheta,\vtheta^{\prime},s,a)-V(\vtheta,\vtheta^{*}_{\eta},s,a)\right\|_2+\left\| V(\vtheta,\vtheta^{*}_{\eta},s,a)-V(\vtheta^*_{\eta},\vtheta^{*}_{\eta},s,a)\right\|_2\\
        &+\left\| V(\vtheta^*_{\eta},\vtheta^{*}_{\eta},s,a) \right\|_2\\
        \leq & l_{V_1}\left\|\vtheta^{\prime}-\vtheta^*_{\eta} \right\|_2+l_{V_2}\left\| \mPhi\vtheta-\mPhi\vtheta^*_{\eta} \right\|_{\infty} +l_{V_3}.
    \end{align*}
    The first and second inequality follows from simple algebraic decomposition and triangle inequality.  The last inequality follows from the previous two results, and applying Lemma~\ref{lem:V-theta-eta-bound}. 
\end{proof}

Now, we present the descent lemma version for the Markoviain observation model:

\begin{proposition}\label{lem:descent-lemma-poisson}
For $t\in\sN$ and $1\leq k \leq K-1$, we have
    \begin{align}
    &  \E\left[ L_{\eta}(\vtheta_{t,{k+1}},\vtheta_{t-1,K} )  \middle | \gF_{t,k}\right] -L_{\eta}(\vtheta_{t,k} ,\vtheta_{t-1,K}) \nonumber \\
    \leq & -\alpha_k \nabla L_{\eta}(\vtheta_{t,k},\vtheta_{t-1,K})^{\top}\left(V(\vtheta_{t,k},\vtheta_{t,0},s_{t,k},a_{t,k})-\E\left[ V(\vtheta_{t,k},\vtheta_{t,0},S_{1},A_{1})\middle|(S_0,A_0)=(s_{t,k},a_{t,k})\right] \right)  \nonumber\\
    &-\alpha_k2\mu_{\eta}\left(L_{\eta}(\vtheta_{t,k},\vtheta_{t-1,K})- L_{\eta}(\vtheta^*(\vtheta_{t-1,K}),\vtheta_{t-1,K})\right)+\frac{1}{2}\alpha_k^2l_{\eta}\E\left[  \left\| g(\vtheta_{t,k},\vtheta_{t-1,K};o_{t,k}) \right\|^2_2\middle|\gF_{t,k} \right], \label{ineq:1}
    \end{align}
    where
    \begin{align*}
        \gF_{t,k}:=\left\{ \vtheta_{0,0}, \{ (s_{i,j},a_{i,j}) :  1 \leq i \leq t, 1 \leq j \leq k \} \right\} .
    \end{align*}
\end{proposition}

\begin{proof}
    We will bound the term the cross term $\nabla L_{\eta}(\vtheta_{t,k},\vtheta_{t-1,K})^{\top}g(\vtheta_{t,k},\vtheta_{t-1,K};o_{t,k} )$ in Lemma~\ref{lem:smooth} using the Poisson equation in Lemma~\ref{lem:poisson-eq}. Let us first observe the following simple decomposition of the cross term:
\begin{align*}
    &\nabla L_{\eta}(\vtheta_{t,k},\vtheta_{t-1,K})^{\top}g(\vtheta_{t,k},\vtheta_{t-1,K};o_{t,k} ) \\
    =& \underbrace{\nabla L_{\eta}(\vtheta_{t,k},\vtheta_{t-1,K})^{\top}\bar{g}(\vtheta_{t,k},\vtheta_{t-1,K};s_{t,k},a_{t,k})}_{I_1}\\
    &+\underbrace{\nabla L_{\eta}(\vtheta_{t,k},\vtheta_{t-1,K})^{\top}(g(\vtheta_{t,k},\vtheta_{t-1,K};o_{t,k})-\bar{g}(\vtheta_{t,k},\vtheta_{t-1,K};s_{t,k},a_{t,k}))}_{I_2}
\end{align*}

The term in $I_2$ disappears if we take the expectation with respect to $s_{t,k+1}$, therefore, our interest is to bound $I_1$. The term $I_1$ can be re-written using the Poisson equation in Lemma~\ref{lem:poisson-eq}:
\begin{align*}
&\nabla L_{\eta}(\vtheta_{t,k},\vtheta_{t-1,K})^{\top}\bar{g}(\vtheta_{t,k},\vtheta_{t-1,K};s_{t,k},a_{t,k}) \\
=&  \nabla L_{\eta}(\vtheta_{t,k},\vtheta_{t-1,K})^{\top} \left(\bar{g}(\vtheta_{t,k},\vtheta_{t-1,K};s_{t,k},a_{t,k}) - \nabla L_{\eta}(\vtheta_{t,k},\vtheta_{t,0}) \right) +  \left\| \nabla L_{\eta}(\vtheta_{t,k},\vtheta_{t-1,K})\right\|_2^2 \\
     =&\nabla L_{\eta}(\vtheta_{t,k},\vtheta_{t-1,K})^{\top} \left(V(\vtheta_{t,k},\vtheta_{t,0},s_{t,k},a_{t,k})-\E\left[ V(\vtheta_{t,k},\vtheta_{t,0},S_1,A_1)\middle| (S_0,A_0)=s_{t,k},a_{t,k} \right] \right)\\
     &+\left\| \nabla L_{\eta}(\vtheta_{t,k},\vtheta_{t-1,K})\right\|_2^2 .
\end{align*}
The first equality follows from using simple algebraic decomposition. 

Now, plugging in $I_1$ and $I_2$, the inequality in Lemma~\ref{lem:smooth} becomes:
\begin{align*}
     &L_{\eta}(\vtheta_{t,{k+1}},\vtheta_{t-1,K} )  -L_{\eta}(\vtheta_{t,k} ,\vtheta_{t-1,K}) \nonumber  \\
\leq & - \alpha_k  \underbrace{\nabla L_{\eta}(\vtheta_{t,k},\vtheta_{t-1,K})^{\top} \left(V(\vtheta_{t,k},\vtheta_{t,0},s_{t,k},a_{t,k})-\E\left[ V(\vtheta_{t,k},\vtheta_{t,0},S_1,A_1)\middle|(S_0,A_0)=s_{t,k},a_{t,k}\right] \right)}_{:=\gE }\\
     &-\alpha_k\left\| \nabla L_{\eta}(\vtheta_{t,k},\vtheta_{t-1,K})\right\|_2^2 +\frac{1}{2}\alpha_k^2l_{\eta}\left\| g(\vtheta_{t,k},\vtheta_{t-1,K};o_{t,k}) \right\|^2_2\\
     &+\nabla L_{\eta}(\vtheta_{t,k},\vtheta_{t-1,K})^{\top}(g(\vtheta_{t,k},\vtheta_{t-1,K};o_{t,k})-\bar{g}(\vtheta_{t,k},\vtheta_{t-1,K};s_{t,k},a_{t,k}))
\end{align*}

Taking conditional expectation, we get 
\begin{align}
    &  \E\left[ L_{\eta}(\vtheta_{t,{k+1}},\vtheta_{t-1,K} )  \middle | \gF_{t,k}\right] -L_{\eta}(\vtheta_{t,k} ,\vtheta_{t-1,K}) \nonumber \\
    \leq & -\alpha_k \E\left[ \gE \middle | \gF_{t,k} \right] -\alpha_k\left\| \nabla L_{\eta}(\vtheta_{t,k},\vtheta_{t-1,K}) \right\|_2^2+\frac{1}{2}\alpha_k^2l_{\eta}\E\left[  \left\| g(\vtheta_{t,k},\vtheta_{t-1,K};o_{t,k}) \right\|^2_2\middle|\gF_{t,k} \right]. \label{ineq:1}
\end{align}
This is because
\begin{align*}
    &\E\left[\nabla L_{\eta}(\vtheta_{t,k},\vtheta_{t-1,K})^{\top}(g(\vtheta_{t,k},\vtheta_{t-1,K};o_{t,k})-\bar{g}(\vtheta_{t,k},\vtheta_{t-1,K};s_{t,k},a_{t,k})) \middle | \gF_{t,k} \right] \\
    =& \nabla L_{\eta}(\vtheta_{t,k},\vtheta_{t-1,K})^{\top} \E\left[(g(\vtheta_{t,k},\vtheta_{t-1,K};o_{t,k})-\bar{g}(\vtheta_{t,k},\vtheta_{t-1,K};s_{t,k},a_{t,k})) \middle | \gF_{t,k} \right]\\
    =& 0.
\end{align*}
Now, bounding $\left\| \nabla L_{\eta}(\vtheta_{t,k},\vtheta_{t-1,K}) \right\|^2_2$ with $L_{\eta}(\vtheta_{t,{k}},\vtheta_{t-1,K} ) - L_{\eta}(\vtheta^*(\vtheta_{t-1,K}),\vtheta_{t-1,K})$ from Lemma~\ref{lem:gdgap} completes the proof.
\end{proof}
From the above Proposition, we need to bound the following term in~(\ref{ineq:1}):
\begin{align*}\nabla L_{\eta}(\vtheta_{t,k},\vtheta_{t-1,K})^{\top}\left(V(\vtheta_{t,k},\vtheta_{t,0},s_{t,k},a_{t,k})-\E\left[ V(\vtheta_{t,k},\vtheta_{t,0},S_{1},A_{1})\middle|(S_0,A_0)=(s_{t,k},a_{t,k})\right] \right).
\end{align*} 
To derive this bound, we introduce the following auxiliary term:
    \begin{align}
        d_{t,k} = \nabla L_{\eta}(\vtheta_{t,k},\vtheta_{t-1,K})^{\top} V(\vtheta_{t,k},\vtheta_{t,0},s_{t,k},a_{t,k}). \label{eq:d_tk}
    \end{align}

\begin{lemma}\label{lem:cross-term-poisson}
For $t\in\sN$ and $1\leq k \leq K-1$, we have  
    \begin{align*}
        & -\nabla L_{\eta}(\vtheta_{t,k},\vtheta_{t-1,K})^{\top}\left(V(\vtheta_{t,k},\vtheta_{t,0},s_{t,k},a_{t,k})-\E\left[ V(\vtheta_{t,k},\vtheta_{t,0},S_{1},A_{1})\middle|(S_0,A_0)=(s_{t,k},a_{t,k})\right] \right)  \\
        \leq & -d_{t,k} + \E\left[ d_{t,k+1} \middle| \gF_{t,k} \right]\\
        &+ \alpha_k D_1 \E\left[ \left\| g(\vtheta_{t,k},\vtheta_{t-1,K};o_{t,k}) \right\|^2_2 \middle| \gF_{t,k} \right]\\
        &+ \alpha_k D_2 \left( L (\vtheta_{t,k},\vtheta_{t-1,K})- L(\vtheta^*(\vtheta_{t-1,K}),\vtheta_{t-1,K}) \right) \\
        &+ 2\alpha_kD_3\left\| \mPhi\vtheta_{t,0}-\mPhi\vtheta^*_{\eta}\right\|^2_{\infty} + 2\alpha_k l_{\eta}l_{V_3}.
    \end{align*}
\end{lemma}
\begin{proof}
A simple algebraic decomposition yields
\begin{align*}
     & \nabla L_{\eta}(\vtheta_{t,k},\vtheta_{t-1,K})^{\top} \left(V(\vtheta_{t,k},\vtheta_{t,0},s_{t,k},a_{t,k})-\E\left[ V(\vtheta_{t,k},\vtheta_{t,0},S_{1},A_{1})\middle|(S_0,A_0)=s_{t,k},a_{t,k} \right] \right)\\
    =& \nabla L_{\eta}(\vtheta_{t,k},\vtheta_{t-1,K})^{\top} (V(\vtheta_{t,k},\vtheta_{t,0},s_{t,k},a_{t,k}) -V(\vtheta_{t,k},\vtheta_{t,0},s_{t,k+1},a_{t,k+1}))\\
    &+ \underbrace{\nabla L_{\eta}(\vtheta_{t,k},\vtheta_{t-1,K})^{\top} \left(V(\vtheta_{t,k},\vtheta_{t,0},s_{t,k+1},a_{t,k+1}  ) - \E\left[ V(\vtheta_{t,k},\vtheta_{t,0},S_{1},A_{1})\middle| (S_0,A_0)=s_{t,k},a_{t,k} \right]  \right)}_{T_4}\\
    =&  \nabla L_{\eta}(\vtheta_{t,k},\vtheta_{t-1,K})^{\top} (V(\vtheta_{t,k},\vtheta_{t,0},s_{t,k},a_{t,k}) -V(\vtheta_{t,k+1},\vtheta_{t,0},s_{t,k+1},a_{t,k+1}  ))\\
    &+ \underbrace{\nabla L_{\eta}(\vtheta_{t,k},\vtheta_{t-1,K})^{\top} \left(V(\vtheta_{t,k+1},\vtheta_{t,0},s_{t,k+1},a_{t,k+1}  ) - V(\vtheta_{t,k},\vtheta_{t,0},s_{t,k+1},a_{t,k+1})  \right)}_{T_3}\\
    &+T_4\\
    =& \underbrace{\nabla L_{\eta}(\vtheta_{t,k},\vtheta_{t-1,K})^{\top} V(\vtheta_{t,k},\vtheta_{t,0},s_{t,k},a_{t,k}) - \nabla L_{\eta}(\vtheta_{t,k+1},\vtheta_{t-1,K})^{\top}V(\vtheta_{t,k+1},\vtheta_{t,0},s_{t,k+1},a_{t,k+1}  )}_{T_1}\\
    &+ \underbrace{(\nabla L_{\eta}(\vtheta_{t,k},\vtheta_{t-1,K})- \nabla L_{\eta}(\vtheta_{t+1,k},\vtheta_{t-1,K}))^{\top}V(\vtheta_{t,k+1},\vtheta_{t,0},s_{t,k+1},a_{t,k+1}  )}_{T_2}\\
    &+T_3+T_4.
\end{align*}

Then, we have
\begin{align}
    -& \nabla L_{\eta}(\vtheta_{t,k},\vtheta_{t-1,K})^{\top} \left(V(\vtheta_{t,k},\vtheta_{t,0},s_{t,k},a_{t,k})-\E\left[ V(\vtheta_{t,k},\vtheta_{t,0},S_{1},A_{1})\middle|(S_0,A_0)=s_{t,k},a_{t,k} \right] \right) \nonumber\\
    \leq &  -T_1+ |T_2|+|T_3|- T_4. \label{ineq:T-sum}
\end{align}

Let us bound the terms $T_2$ and $T_3$. First, observe the following:
\begin{align*}
 & |T_2|\\
     =& |(\nabla L_{\eta}(\vtheta_{t,k},\vtheta_{t-1,K})-\nabla L_{\eta}(\vtheta_{t,k+1},\vtheta_{t-1,K})^{\top} V(\vtheta_{t,k+1},\vtheta_{t,0},s_{t,k+1},a_{t,k+1})|\\
    \leq  & l_{\eta} \left\| \vtheta_{t,k}-\vtheta_{t,k+1} \right\|_2 \left\| V(\vtheta_{t,k+1},\vtheta_{t,0},s_{t,k+1},a_{t,k+1}) \right\|_2 \\
    = & l_{\eta}\alpha_k \left\| g(\vtheta_{t,k},\vtheta_{t-1,K};o_{t,k}) \right\|_2 \left\| V(\vtheta_{t,k+1},\vtheta_{t,0},s_{t,k+1},a_{t,k+1}) \right\|_2\\
    \leq & \alpha_k^2l_{\eta}l_{V_1}\left\| g(\vtheta_{t,k},\vtheta_{t-1,K};o_{t,k}) \right\|_2^2+\alpha_k l_{\eta}l_{V_1} \left\| g(\vtheta_{t,k},\vtheta_{t-1,K};o_{t,k} \right\|_2 \left\| \vtheta_{t,k}-\vtheta^*(\vtheta_{t,0}) \right\|_2\\
    &+ \alpha_k l_{\eta}\left( \frac{l_{V_1}\gamma}{\mu_{\eta}}+l_{V_2}\right)\left\| g(\vtheta_{t,k},\vtheta_{t-1,K};o_{t,k})\right\|_2\left\| \mPhi(\vtheta_{t,0}-\vtheta^*_{\eta})\right\|_{\infty}+\alpha_kl_{\eta}l_{V_3}\left\|  g(\vtheta_{t,k},\vtheta_{t-1,K};o_{t,k}) \right\|_2\\
    \leq &   \left( l_{\eta}(4l_{V_1}+l_{V_3})+\kappa l_{V_1}\gamma \right)\left\| g(\vtheta_{t,k},\vtheta_{t-1,K};o_{t,k}) \right\|_2^2\\
    &+2\alpha_kl_{\eta}l_{V_1}\left\| \vtheta_{t,k}-\vtheta^*(\vtheta_{t,0}) \right\|_2^2\\
    &+ 2\alpha_k \left( \kappa l_{V_1}+l_{\eta}l_{V_2}\right)\left\| \mPhi(\vtheta_{t,0}-\vtheta^*_{\eta})\right\|^2_{\infty} + 2\alpha_k l_{\eta}l_{V_3}\\
    \leq &  \left( l_{\eta}(4l_{V_1}+l_{V_3})+\kappa l_{V_1}\gamma \right)\left\| g(\vtheta_{t,k},\vtheta_{t-1,K};o_{t,k}) \right\|_2^2\\
    &+4\alpha_kl_{V_1}\kappa \left( L_{\eta}(\vtheta_{t,k},\vtheta_{t-1,K})-L_{\eta}(\vtheta^*(\vtheta_{t-1,K}),\vtheta_{t-1,K}) \right)\\
    &+ 2\alpha_k\left(  \kappa l_{V_1}+l_{\eta}l_{V_2} \right)\left\| \mPhi(\vtheta_{t,0}-\vtheta^*_{\eta})\right\|^2_{\infty} + 2\alpha_k l_{\eta}l_{V_3}.
\end{align*}
The first inequality follows smoothness of $L_{\eta}(\cdot)$ in Lemma~\ref{lem:smooth}. The bound on the term $\left\| V(\vtheta_{t,k+1},\vtheta_{t,0},s_{t,k+1},a_{t,k+1}) \right\|_2$ comes from Lemma~\ref{lem:V-tk-bound} in the Appendix. The last inequality comes from the quadratic growth condition in Lemma~\ref{quadraticgrowth} in the Appendix.

Next, we will bound $T_3$. From the Lipschitzness of $V(\cdot)$ in Lemma~\ref{lem:poisson-property}, we have
\begin{align*}
        &|T_3|\\
      =&|\nabla L_{\eta}(\vtheta_{t,k},\vtheta_{t-1,K})^{\top} \left( V(\vtheta_{t,k},\vtheta_{t,0},s_{t,k+1},a_{t,k+1} ) - V(\vtheta_{t,k+1},\vtheta_{t,0},s_{t,k+1},a_{t,k+1} ) \right) |\\
     \leq & l_{V_1}\left\| \nabla L_{\eta}(\vtheta_{t,k},\vtheta_{t-1,K}) \right\|_2 \left\| \vtheta_{t,k}-\vtheta_{t,k+1}\right\|_2\\
     \leq & 2 \alpha_k  l_{V_1} \left( \left\| \nabla L_{\eta}(\vtheta_{t,k},\vtheta_{t-1,K})\right\|^2_2 + \left\| g(\vtheta_{t,k},\vtheta_{t,0};o_{t,k})\right\|^2_2\right)\\
     \leq &  \alpha_k  \frac{4l_{V_1}l_{\eta}^2}{\mu_{\eta}}\left( L_{\eta} (\vtheta_{t,k},\vtheta_{t-1,K})- L_{\eta}(\vtheta^*(\vtheta_{t-1,K}),\vtheta_{t-1,K}) \right) \\
     &+ 2\alpha_k  l_{V_1}\left\| g(\vtheta_{t,k},\vtheta_{t,0};o_{t,k})\right\|^2_2.
\end{align*}

The second inequality follows from the Cauchy-Schwarz inequality. The last inequality follows from Lemma~\ref{quadraticgrowth} in the Appendix.

Now, collecting the bound on $T_2$ and $T_3$, from~(\ref{ineq:T-sum}), we get
\begin{align*}
     & -\nabla L_{\eta}(\vtheta_{t,k},\vtheta_{t-1,K})^{\top}\left(V(\vtheta_{t,k},\vtheta_{t,0},s_{t,k},a_{t,k})-\E\left[ V(\vtheta_{t,k},\vtheta_{t,0},S_{1},A_{1})\middle|(S_0,A_0)=(s_{t,k},a_{t,k})\right] \right)  \\
        \leq & -d_{t,k}+d_{t,k+1}\\
        &+ \alpha_k\left( \kappa \left( \mu_{\eta}(6l_{V_1}+l_{V_3})+l_{V_1}\gamma \right) \right) \left\| g(\vtheta_{t,k},\vtheta_{t-1,K};o_{t,k}) \right\|^2_2\\
        &+ \alpha_k\kappa \left(4l_{V_1}(1+l_{\eta})  \right) \left( L_{\eta} (\vtheta_{t,k},\vtheta_{t-1,K})- L_{\eta}(\vtheta^*(\vtheta_{t-1,K}),\vtheta_{t-1,K}) \right) \\
        &+ 2\alpha_k\left( \kappa l_{V_1}+ l_{\eta}l_{V_2} \right)\left\| \mPhi\vtheta_{t,0}-\mPhi\vtheta^*_{\eta}\right\|^2_{\infty} + 2\alpha_k l_{\eta}l_{V_3}\\
        &-T_4.
\end{align*}

Taking the conditional expectation, noting that $\E\left[  T_4 \middle| \gF_{t,k}\right]=0$, we get the desired result.



\end{proof}

The above lemma allows us to bound the cross term in Lemma~\ref{lem:descent-lemma-poisson}. Now, applying the bound on $\E\left[\left\| g(\vtheta_{t,k},\vtheta_{t-1,K};o_{t,k}) \right\|^2_2 \middle| \gF_{t,k} \right]$, we obtain the following result:

\begin{proposition}[Descent-lemma for inner loop]\label{prop:descent-lemma-final}
    For $\alpha_k \leq \frac{\mu_{\eta}}{  \left( D_1+\frac{l_{\eta}}{2} \right)g_{1,\eta}+ 2D_2  }   $, we have
    \begin{align*}
   \E\left[ L_{\eta}(\vtheta_{t,{k+1}},\vtheta_{t-1,K} )  \middle | \gF_{t,k}\right] -L_{\eta}(\vtheta_{t,k} ,\vtheta_{t-1,K}) 
        \leq &  -\alpha_k (d_{t,k}-\E\left[ d_{t,k+1} \middle | \gF_{t,k}\right])\\
        &-\mu_{\eta}\alpha_k \left( L_{\eta} (\vtheta_{t,k},\vtheta_{t-1,K})- L_{\eta}(\vtheta^*(\vtheta_{t-1,K}),\vtheta_{t-1,K})  \right) \\
        &+\alpha_k^2 \left( \gE_1\left\| \mPhi(\vtheta_{t-1,K}-\vtheta^*_{\eta})\right\|^2_{\infty}+\gE_2 \right)
        \end{align*}
\end{proposition}
\begin{proof}
Applying the result of Lemma~\ref{lem:cross-term-poisson} to Lemma~\ref{lem:descent-lemma-poisson},
        \begin{align*}
               &  \E\left[ L_{\eta}(\vtheta_{t,{k+1}},\vtheta_{t-1,K} )  \middle | \gF_{t,k}\right] -L_{\eta}(\vtheta_{t,k} ,\vtheta_{t-1,K}) \nonumber \\
        \leq &  -\alpha_k (d_{t,k}-\E\left[ d_{t,k+1} \middle | \gF_{t,k}\right])\\
        & +\alpha_k^2\left(  D_1 +   \frac{l_{\eta}}{2} \right) \E\left[  \left\| g(\vtheta_{t,k},,\vtheta_{t-1,K};o_{t,k}) \right\|^2_2 \middle | \gF_{t,k}\right] \\
        &+ 2\alpha_k^2 D_3 \left\| \mPhi\vtheta_{t,0}-\mPhi\vtheta^*_{\eta}\right\|^2_{\infty} + 2\alpha_k^2l_{\eta}l_{V_3}\\
        &+  \left(  \alpha_k^2 D_2 - 2  \mu_{\eta}\alpha_k  \right)\left(L_{\eta} (\vtheta_{t,k},\vtheta_{t-1,K})- L_{\eta}(\vtheta^*(\vtheta_{t-1,K}),\vtheta_{t-1,K})  \right)\\
        \leq & -\alpha_k ( d_{t,k}-\E\left[ d_{t,k+1}\middle|\gF_{t,k} \right])\\
        &+\alpha_k^2  \left(  D_1 +   \frac{l_{\eta}}{2} \right)  g_{1,\eta} (L_{\eta}(\vtheta_{t,k},\vtheta_{t-1,K})-L_{\eta}(\vtheta^*(\vtheta_{t-1,K}),\vtheta_{t-1,K} ))\\
 &+ \alpha_k^2  \left(  D_1 +   \frac{l_{\eta}}{2} \right)  g_{2,\eta} \left\| \mPhi(\vtheta_{t-1,K}-\vtheta^*_{\eta}) \right\|^2_{\infty}\\
 &+ \alpha_k^2 \left(  D_1 +   \frac{l_{\eta}}{2} \right)  g_{3,\eta} \\
  &+ 2\alpha_k^2D_3\left\| \mPhi\vtheta_{t,0}-\mPhi\vtheta^*_{\eta}\right\|^2_{\infty} + 2\alpha_k^2l_{\eta}l_{V_3}\\
        &-\left(- D_2 \alpha_k^2 + 2\alpha_k \mu_{\eta}  \right)\left( L_{\eta} (\vtheta_{t,k},\vtheta_{t-1,K})- L_{\eta}(\vtheta^*(\vtheta_{t-1,K}),\vtheta_{t-1,K})  \right)\\
        \leq & -\alpha_k ( d_{t,k}-\E\left[ d_{t,k+1}\middle|\gF_{t,k} \right])\\
        &+  \left(   \alpha_k^2 \left( \left( D_1+\frac{l_{\eta}}{2} \right)g_{1,\eta}+ 2D_2 \right) - \alpha_k 2\mu_{\eta}  \right)(L_{\eta}(\vtheta_{t,k},\vtheta_{t-1,K})-L_{\eta}(\vtheta^*(\vtheta_{t-1,K}),\vtheta_{t-1,K} ))\\
        &+\alpha_k^2 \left(      \left( D_1+\frac{l_{\eta}}{2} \right) g_{2,\eta} + D_3\right)  \left\|  \mPhi( \vtheta_{t-1,K}-\vtheta^*_{\eta})\right\|^2_{\infty}\\
        &+ \alpha_k^2 \left( D_1+\frac{l_{\eta}}{2} \right) g_{3,\eta}+ 2\alpha_k^2 l_{\eta}l_{V_3} .
        \end{align*}
        The second inequality follows from the bound on $\E\left[  \left\| g(\vtheta_{t,k},,\vtheta_{t-1,K};o_{t,k}) \right\|^2_2 \middle | \gF_{t,k}\right]$ in Lemma~\ref{lem:g-bound}. The step-size condition
        \begin{align*}
       &\alpha_k \leq \frac{\mu_{\eta}}{  \left( D_1+\frac{l_{\eta}}{2} \right)g_{1,\eta}+ 2D_2  } \\
       \Rightarrow & \alpha_k^2 \left( \left( D_1+\frac{l_{\eta}}{2} \right)g_{1,\eta}+ 2D_2 \right) - \alpha_k 2\mu_{\eta}  \leq - \alpha_k \mu_{\eta}
        \end{align*}
        yields the desired result.
\end{proof}

Before proceeding, we introduce the constants that determine the step-size:
\begin{equation}\label{constants:alpha}
\begin{aligned}
    \bar{\alpha}_1 =& \frac{\mu_{\eta}}{  \left( D_1+\frac{l_{\eta}}{2} \right)g_{1,\eta}+ 2D_2  },  \\
    \bar{\alpha}_2 =&  \frac{\mu_{\eta}}{  16l_{\eta}(6l_{V_1}+l_{V_3})(1+\eta) + 24\kappa l_{V_1}(1+\eta) }, \\
    \bar{\alpha}_3 =& \frac{\mu_{\eta}^2(\gamma||\mGamma_{\eta}||_{\infty})^2(1-(\gamma||\mGamma_{\eta}||_{\infty})^2)}{8(1+(\gamma||\mGamma_{\eta}||_{\infty})^2)(2\mu_{\eta}\gE_1+ 4\mu_{\eta}\left( l_{V_1}\gamma\kappa + l_{\eta}l_{V_2} \right) )},\\
    \bar{\alpha}_4=& \frac{2\mu_{\eta}^2\eps\gamma^2(1-(\gamma||\mGamma_{\eta}||_{\infty})^2)}{4(1+(\gamma||\mGamma_{\eta}||_{\infty})^2)\left( \gE_2+3l_{\eta}l_{V_3}\mu_{\eta}\right)}.    
\end{aligned}
\end{equation}

Now, using the above descent lemma for the inner loop, we are ready to derive the convergence rate result of the inner-loop iteration:

\begin{proposition}\label{prop:inner-loop-final}
        For $\alpha \leq \min \left\{  \bar{\alpha}_1 ,  \bar{\alpha}_2 \right\}$, which is defined in~(\ref{constants:alpha}),  we have
\begin{align*}
 &   \E \left[L_{\eta}(\vtheta_{t,k},\vtheta_{t-1,K})-L_{\eta}(\vtheta^*(\vtheta_{t-1,K}),\vtheta_{t-1,K}) \right] \\
 \leq & 2\left( 1-\frac{\mu_{\eta}}{2} \right)^k  \E\left[  L_{\eta}(\vtheta_{t,0},\vtheta_{t-1,K})-L_{\eta}(\vtheta^*(\vtheta_{t-1,K}),\vtheta_{t-1,K})\right] \\
    &+ 2\alpha  \left( \left(\frac{2}{\mu_{\eta}}\gE_1+4\left( l_{V_1}\gamma\kappa + l_{\eta}l_{V_2} \right) \right) \E\left[\left\| \mPhi(\vtheta_{t,K-1}-\vtheta^*_{\eta})\right\|^2_{\infty}\middle| \gF_{t,0}\right]+\frac{2}{\mu_{\eta}}\gE_2+ 6l_{\eta}l_{V_3} \right).
\end{align*}
\end{proposition}
\begin{proof}
For simplicity of the proof, let $x_k = \E\left[ L_{\eta}(\vtheta_{t,k},\vtheta_{t-1,K})-L_{\eta}(\vtheta^*(\vtheta_{t-1,K}),\vtheta_{t-1,K}) \middle | \gF_{t,0} \right]$. Then, taking the conditional expectation on $\gF_{t,0}$ to the result of Proposition~\ref{prop:descent-lemma-final}, we have
\begin{align*}
    &x_{k+1}\\
    \leq &  \left( 1-\mu_{\eta}\alpha_k\right)x_k-\alpha_k\E\left[ d_{t,k}-d_{t,k+1} \middle | \gF_{t,0} \right]+ \alpha_k^2 \left( \gE_1 \E\left[\left\|\mPhi(\vtheta_{t,K-1}-\vtheta^*_{\eta} ) \right\|^2_{\infty} \middle|\gF_{t,0}\right]+\gE_2 \right)\\
    =&\left(1-\mu_{\eta}\alpha_k \right)x_k - \left(1-\frac{\mu_{\eta}}{2}\alpha_k \right)\alpha_k  \E\left[d_{t,k} \middle|\gF_{t,0} \right]-\frac{\mu_{\eta}}{2}\alpha_k^2 \E\left[d_{t,k}\middle|\gF_{t,0}  \right]+\alpha\E\left[d_{t,k+1} \middle|\gF_{t,0}\right]\\
    &+\alpha_k^2 \left( \gE_1 \E\left[\left\|\mPhi(\vtheta_{t,K-1}-\vtheta^*_{\eta} ) \right\|^2_{\infty} \middle| \gF_{t,0}  \right]+\gE_2 \right)\\
    \leq & \left(1-\mu_{\eta}\alpha_k \right)x_k - \left(1-\frac{\mu_{\eta}}{2}\alpha_k \right)\alpha_k \E\left[d_{t,k}\middle| \gF_{t,0} \right]+\alpha  \E\left[d_{t,k+1} \middle|\gF_{t,0}\right] +\alpha_k^2 \left( \gE_1 \E\left[\left\|\mPhi(\vtheta_{t,K-1}-\vtheta^*_{\eta} ) \right\|^2_{\infty} \middle| \gF_{t,0} \right]+\gE_2 \right)\\
    &+\alpha_k^2 \left(\left( l_{\eta}l_{V_1}+4l_{\eta}l_{V_3} + 2\left( \kappa l_{V_1}\gamma +l_{\eta}l_{V_2}\right)\right) x_k+ \mu_{\eta}\left( l_{V_1}\gamma\kappa + l_{\eta}l_{V_2} \right)\E\left[\left\|\mPhi(\vtheta_{t,0}-\vtheta^*_{\eta}) \right\|^2_{\infty} \middle|\gF_{t,0}\right]+\mu_{\eta}l_{\eta}l_{V_3}\right)\\
    =&  \left(1-\mu_{\eta}\alpha_k+\left( l_{\eta}l_{V_1}+4l_{\eta}l_{V_3} + 2\left( \kappa l_{V_1}\gamma +l_{\eta}l_{V_2}\right)\right)\alpha_k^2 \right)x_k -\left( 1-\frac{\mu_{\eta}}{2}\alpha_k \right)\alpha_k \E\left[d_{t,k}\middle| \gF_{t,0} \right]+\alpha \E\left[d_{t,k+1} \middle|\gF_{t,0}\right] \\
    &+ \alpha_k^2 \left( \left(\gE_1+ \mu_{\eta}\left( l_{V_1}\gamma\kappa + l_{\eta}l_{V_2} \right) \right) \E\left[\left\| \mPhi(\vtheta_{t,K-1}-\vtheta^*_{\eta})\right\|^2_{\infty}\middle| \gF_{t,0}\right]+\gE_2+ 2\mu_{\eta}l_{\eta}l_{V_3} \right).
\end{align*}
where the first equality follows from simple algebraic decomposition. The last inequality follows from bounding $d_{t,k}$ from Lemma~\ref{lem:d-bound} in the Appendix Section~\ref{app:sec:aux-leemma-for-inner-loop}.

Since $\frac{\mu_{\eta}}{  16l_{\eta}(6l_{V_1}+l_{V_3})(1+\eta) + 24\kappa l_{V_1}(1+\eta) }  \leq \frac{\mu_{\eta}}{2( l_{\eta}l_{V_1}+4l_{\eta}l_{V_3} + 2\left( \kappa l_{V_1}\gamma +l_{\eta}l_{V_2}\right)  )} $, the step-size condition
\begin{align*}
 -\mu_{\eta}\alpha_k + (l_{\eta}l_{V_1}+4l_{\eta}l_{V_3} + 2\left( \kappa l_{V_1}\gamma +l_{\eta}l_{V_2}\right))\alpha^2 \leq -\frac{\mu_{\eta}\alpha}{2}
\end{align*}
yields the following:
\begin{align*}
    x_{k+1}  \leq& \left(  1 - \frac{\mu_{\eta}}{2}\alpha\right) x_k- \left( 1-\frac{\mu_{\eta}}{2}\alpha_k \right) \alpha_k d_{t,k}+\alpha d_{t,k+1}\\
    &+ \alpha_k^2 \left( \left(\gE_1+ \mu_{\eta}\left( l_{V_1}\gamma\kappa + l_{\eta}l_{V_2} \right) \right) \E\left[\left\| \mPhi(\vtheta_{t,K-1}-\vtheta^*_{\eta})\right\|^2_{\infty}\middle| \gF_{t,0}\right]+\gE_2+ 2\mu_{\eta}l_{\eta}l_{V_3} \right).
\end{align*}

Recursively expanding the terms, we get
\begin{align*}
    x_{k+1} \leq &  \left( 1-\frac{\mu_{\eta}}{2}\alpha \right)^kx_0 +\alpha d_{t,k+1} \\
    &+ \frac{2}{\mu_{\eta}}\alpha_k  \left( \left(\gE_1+ \mu_{\eta}\left( l_{V_1}\gamma\kappa + l_{\eta}l_{V_2} \right) \right) \E\left[\left\| \mPhi(\vtheta_{t,K-1}-\vtheta^*_{\eta})\right\|^2_{\infty}\middle| \gF_{t,0}\right]+\gE_2+ 2\mu_{\eta}l_{\eta}l_{V_3} \right)\\
    \leq & \left( 1-\frac{\mu_{\eta}}{2}\alpha  \right)^kx_0+\alpha \frac{2}{\mu_{\eta}} \left( l_{\eta}l_{V_1}+4l_{\eta}l_{V_3} + 2\left( \kappa l_{V_1}\gamma +l_{\eta}l_{V_2}\right)\right) x_{k+1}\\
    &+ \alpha \left( \left(\frac{2}{\mu_{\eta}}\gE_1+4\left( l_{V_1}\gamma\kappa + l_{\eta}l_{V_2} \right) \right) \E\left[\left\| \mPhi(\vtheta_{t,K-1}-\vtheta^*_{\eta})\right\|^2_{\infty}\middle| \gF_{t,0}\right]+\frac{2}{\mu_{\eta}}\gE_2+ 6l_{\eta}l_{V_3} \right).
\end{align*}

The last inequality follows from the bounding $d_{t,k+1}$ of Lemma~\ref{lem:d-bound} in Appendix Section~\ref{app:sec:aux-leemma-for-inner-loop}.\

Noting that $\frac{\mu_{\eta}}{  16l_{\eta}(6l_{V_1}+l_{V_3})(1+\eta) + 24\kappa l_{V_1}(1+\eta) }  \leq \frac{\mu_{\eta}}{4\left( l_{\eta}l_{V_1}+4l_{\eta}l_{V_3} + 2\left( \kappa l_{V_1}\gamma +l_{\eta}l_{V_2}\right)\right)}$,
we have
\begin{align*}
    x_{k+1} \leq & 2\left( 1-\frac{\mu_{\eta}}{2} \right)^kx_0 \\
    &+ 2\alpha  \left( \left(\frac{2}{\mu_{\eta}}\gE_1+4\left( l_{V_1}\gamma\kappa + l_{\eta}l_{V_2} \right) \right) \E\left[\left\| \mPhi(\vtheta_{t,K-1}-\vtheta^*_{\eta})\right\|^2_{\infty}\middle| \gF_{t,0}\right]+\frac{2}{\mu_{\eta}}\gE_2+ 6l_{\eta}l_{V_3} \right).
\end{align*}

Taking the total expectation, we have the desired results. 
\end{proof}

We are now ready to present the main result in the proof of Theorem~\ref{thm:markovian_obs_model}. By applying the result of the inner iteration analysis to the outer iteration decomposition established in Proposition~\ref{prop:outer-loop-decomposition}, we obtain the desired conclusion.

\subsection{Proof of Theorem~\ref{thm:markovian_obs_model}}\label{app:thm:markovian_obs_model}
\begin{proof}
    For simplicity of the proof, let $y_t =\E\left[ \left\|\mPhi(\vtheta_{t,K}-\vtheta^*_{\eta}) \right\|^2_{\infty} \right] $.
    
    Applying the result in Proposition~\ref{prop:inner-loop-final} to the bound in Proposition~\ref{prop:outer-loop-decomposition} with $\delta=\frac{2(\gamma||\mGamma_{\eta}||_{\infty})^2}{1-(\gamma||\mGamma_{\eta}||_{\infty})^2}$, we have
    \begin{align}
             & y_t  \nonumber\\
        \leq &\left( \frac{1+(\gamma||\mGamma_{\eta}||_{\infty})^2}{\mu_{\eta}(\gamma||\mGamma_{\eta}||_{\infty})^2}   \left(  16\left(1-\frac{\mu_{\eta}}{2}\alpha_0\right)^K+2\alpha\left(\frac{2}{\mu_{\eta}}\gE_1+4\left( l_{V_1}\gamma\kappa + l_{\eta}l_{V_2} \right) \right)\right) +\frac{1+(\gamma||\mGamma_{\eta}||_{\infty})^2}{2}\right)y_{t-1} \label{inqe:final-1}\\
        &+\underbrace{\frac{1+(\gamma||\mGamma_{\eta}||_{\infty})^2}{\mu_{\eta}(\gamma||\mGamma_{\eta}||_{\infty})^2} \left( 2 \left(1-\frac{\mu_{\eta}}{2}\alpha_0 \right)^K(R_{\max}^2+8\left\| \mPhi\vtheta^*_{\eta} \right\|^2_{\infty}) +2 \alpha  \left( \frac{2}{\mu_{\eta}}\gE_2+ 6l_{\eta}l_{V_3} \right) \right)}_{:=\gE_{K,\alpha_0}}. \nonumber
    \end{align}

    Let us first bound the coefficient of $y_t$ with $\frac{3+(\gamma||\mGamma_{\eta}||_{\infty})^2}{4}$. Then, it is enough to bound the coefficient in~(\ref{inqe:final-1}) with $\frac{1-(\gamma||\mGamma_{\eta}||_{\infty})^2}{4}$, i.e., we require
    \begin{align*}
        \frac{1+(\gamma||\mGamma_{\eta}||_{\infty})^2}{\mu_{\eta}(\gamma||\mGamma_{\eta}||_{\infty})^2}  \left( 16\left(1-\frac{\mu_{\eta}}{2}\alpha_0\right)^K+ 2\alpha\left(\frac{2}{\mu_{\eta}}\gE_1+4\left( l_{V_1}\gamma\kappa + l_{\eta}l_{V_2} \right)\right)\right)  \leq \frac{1-(\gamma||\mGamma_{\eta}||_{\infty})^2}{4}.
    \end{align*}
    The above condition is satisfied if
    \begin{align*}
        16\left( 1-\frac{\mu_{\eta}}{2}\alpha_0\right)^K\leq &  \frac{\mu_{\eta}(\gamma||\mGamma_{\eta}||_{\infty})^2(1-(\gamma||\mGamma_{\eta}||_{\infty})^2)}{8(1+(\gamma||\mGamma_{\eta}||_{\infty})^2)} , \\
       2\alpha\left(\frac{2}{\mu_{\eta}}\gE_1+4\left( l_{V_1}\gamma\kappa + l_{\eta}l_{V_2} \right) \right)\leq &  \frac{\mu_{\eta}(\gamma||\mGamma_{\eta}||_{\infty})^2(1-(\gamma||\mGamma_{\eta}||_{\infty})^2)}{8(1+(\gamma||\mGamma_{\eta}||_{\infty})^2)}.
    \end{align*}
   These inequalities are, in turn, ensured by choosing $K$ and $\alpha_0$ such that
    \begin{align}
        K \geq \frac{2}{\mu_{\eta}\alpha_0} \ln \left(\frac{\mu_{\eta}(1-(\gamma||\mGamma_{\eta}||_{\infty})^2)}{128(1+(\gamma||\mGamma_{\eta}||_{\infty})^2)}\right) ,\quad \alpha_0 \leq  \bar{\alpha}_3 \label{K-alpha-condiiton-1}
    \end{align}
    Applying this result to~(\ref{inqe:final-1}), we get
    \begin{align*}
       y_t\leq & \frac{3+(\gamma||\mGamma_{\eta}||_{\infty})^2}{4} y_{t-1}+ \gE_{K,\alpha_0}\\
        \leq & \left(\frac{3+(\gamma||\mGamma_{\eta}||_{\infty})^2}{4} \right)^2 y_{t-2} + \sum_{j=t-1}^t\left( \frac{3+(\gamma||\mGamma_{\eta}||_{\infty})^2}{4} \right)^{t-j} \gE_{K,\alpha_0} , \\
        \leq & \left(\frac{3+(\gamma||\mGamma_{\eta}||_{\infty})^2}{4} \right)^t y_0 + \frac{4}{1-(\gamma||\mGamma_{\eta}||_{\infty})^2} \gE_{K,\alpha_0}.
    \end{align*}

For the above bound to be smaller than $\eps$, a sufficient condition is to make each terms smaller than $\frac{\eps}{2}$:
\begin{align*}
&    \left(\frac{3+(\gamma||\mGamma_{\eta}||_{\infty})^2}{4} \right)^t \E\left[ \left\| \mPhi\vtheta_{0,K}-\mPhi\vtheta^*_\eta \right\|^2_{\infty} \right] \leq \frac{\eps}{2},
\end{align*}
which is satisfied if we choose $t$ as follows:
\begin{align}
    t \geq \frac{4}{1-(\gamma||\mGamma_{\eta}||_{\infty})^2}\ln\left(\frac{2\left\|\mPhi(\vtheta_{0,K}-\vtheta^*_\eta) \right\|^2_{\infty}}{\eps}\right).
\end{align}

To bound the remaining term, $\frac{1}{1-(\gamma||\mGamma_{\eta}||_{\infty})^2}\gE_{K,\alpha_0}$, with $\frac{\eps}{2}$, we require
\begin{align*}
&    \frac{1+(\gamma||\mGamma_{\eta}||_{\infty})^2}{\mu_{\eta}(\gamma||\mGamma_{\eta}||_{\infty})^2} \left( 2 \left(1-\frac{\mu_{\eta}}{2}\alpha_0 \right)^K(R_{\max}^2+8\left\| \mPhi\vtheta^*_{\eta} \right\|^2_{\infty}) +2 \alpha  \left( \frac{2}{\mu_{\eta}}\gE_2+ 6l_{\eta}l_{V_3} \right) \right)\leq \frac{(1-(\gamma||\mGamma_{\eta}||_{\infty})^2)\eps}{2}.
\end{align*}

Now, bound each terms with $\frac{(1-(\gamma||\mGamma_{\eta}||_{\infty})^2)\eps}{4}$, we need 
\begin{align}
&    \exp(-K\mu_{\eta}\alpha_0/2)(R^2_{\max}+8\left\|\mPhi\vtheta^*_{\eta}\right\|^2_{\infty}) \leq \frac{\eps(\gamma||\mGamma_{\eta}||_{\infty})^2 \mu_{\eta}(1-(\gamma||\mGamma_{\eta}||_{\infty})^2)}{4(1+(\gamma||\mGamma_{\eta}||_{\infty})^2)} \nonumber\\
\iff & K \geq \frac{2}{\mu_{\eta}\alpha_0} \ln \left(\frac{1}{R^2_{\max}+8\left\|\mPhi\vtheta^*_{\eta} \right\|^2_{\infty}}\frac{4(1+(\gamma||\mGamma_{\eta}||_{\infty})^2)}{\eps(\gamma||\mGamma_{\eta}||_{\infty})^2 \mu_{\eta}(1-(\gamma||\mGamma_{\eta}||_{\infty})^2)} \right)\label{cond:K-2}
\end{align}
Likewise, bounding the remaining term with $\frac{(1-\gamma||\mGamma_{\eta}||_{\infty})^2\eps}{4}$, we require
\begin{align}
& \alpha_0 \left( \gE_2+3l_{\eta}l_{V_3}\mu_{\eta}\right) \leq \frac{\eps\gamma^2 \mu_{\eta}(1-(\gamma||\mGamma_{\eta}||_{\infty})^2)}{4(1+(\gamma||\mGamma_{\eta}||_{\infty})^2)} \nonumber \\
\iff & \alpha_0 \leq \bar{\alpha}_4. \label{cond:alpha-2}
\end{align}

Now, collecting the conditions on $\alpha$ in~(\ref{K-alpha-condiiton-1}) and~(\ref{cond:alpha-2}), we need
\begin{align*}
     &\alpha_0 \\
    =&  \min\left\{ \frac{\mu_{\eta}}{  l_{\eta}(6l_{V_1}+l_{V_3})(1+\eta) + \kappa l_{V_1}(1+\eta) } , \bar{\alpha}_2, \bar{\alpha}_3\right\}
\end{align*}
Moreover, collecting the bound on $K$ in~(\ref{K-alpha-condiiton-1}) and~(\ref{cond:K-2}), we have
\begin{align*}
    K =  \gO \left( \max\left\{ \frac{ l_{\eta}(6l_{V_1}+l_{V_3})(1+\eta) + \kappa l_{V_1}(1+\eta)}{\mu_{\eta}^2} , \frac{2\mu_{\eta}\gE_1+ 4\mu_{\eta}\left( l_{V_1}\kappa + l_{\eta}l_{V_2} \right)}{\mu_{\eta}^3(1-\gamma||\mGamma_{\eta}||_{\infty})}, \frac{\gE_2+\mu_{\eta}l_{\eta}l_{V_3}}{\eps\mu_{\eta}^2(1-\gamma||\mGamma_{\eta}||_{\infty})}\right\}\right).
\end{align*}

    This completes the proof.
\end{proof}





\subsection{Auxiliary Lemmas for Markovian Observation Model Analysis}\label{app:sec:aux-leemma-for-inner-loop}

\begin{lemma}\label{lem:V-tk-bound}
    We have for $t\in\sN$ and $1\leq k \leq K-1$,
    \begin{align*}
   \left\| V(\vtheta_{t,k+1},\vtheta_{t,0},s_{t,k+1},a_{t,k+1}) \right\|_2 \leq &      l_{V_1}\alpha_k \left\| g(\vtheta_{t,k},\vtheta_{t-1,K};o_{t,k}) \right\|_2 + l_{V_1} \left\| \vtheta_{t,k}-\vtheta^*(\vtheta_{t,0}) \right\|_2\\
    &+\left( \frac{l_{V_1 }\gamma }{\mu_{\eta}}+l_{V_2} \right)\left\| \mPhi(\vtheta_{t,0}-\vtheta^*_{\eta})\right\|_{\infty}+l_{V_3}.
    \end{align*}
\end{lemma}
\begin{proof}
    We have
    \begin{align*}
    &\left\| V(\vtheta_{t,k+1},\vtheta_{t,0},s_{t,k+1},a_{t,k+1}) \right\|_2 \\
    \leq &  \left\| V(\vtheta_{t,k+1},\vtheta_{t,0},s_{t,k+1},a_{t,k+1})- V(\vtheta^*(\vtheta_{t,0}),\vtheta_{t,0},s_{t,k+1},a_{t,k+1})  \right\|_2 \\
    &+ \left\| V(\vtheta^*(\vtheta_{t,0}),\vtheta_{t,0},s_{t,k+1},a_{t,k+1})  \right\|_2\\
    \leq & l_{V_1} \left\| \vtheta_{t,k+1}-\vtheta^*(\vtheta_{t,0}) \right\|_2+ \left\| V(\vtheta^*(\vtheta_{t,0}),\vtheta_{t,0},s_{t,k+1},a_{t,k+1})  \right\|_2\\
    \leq &  l_{V_1}\left\| \vtheta_{t,k+1}-\vtheta_{t,k} \right\|_2 + l_{V_1} \left\| \vtheta_{t,k}-\vtheta^*(\vtheta_{t,0}) \right\|_2\\
    &+l_{V_1}\left\| \vtheta^*(\vtheta_{t,0})-\vtheta^*_{\eta} \right\|_2+l_{V_2}\left\| \mPhi(\vtheta_{t,0}-\vtheta^*_{\eta})\right\|_{\infty}+l_{V_3}\\
    \leq & l_{V_1}\alpha_k \left\| g(\vtheta_{t,k},\vtheta_{t-1,K};o_{t,k}) \right\|_2 + l_{V_1} \left\| \vtheta_{t,k}-\vtheta^*(\vtheta_{t,0}) \right\|_2\\
    &+\left( \frac{l_{V_1 }\gamma }{\mu_{\eta}}+l_{V_2} \right)\left\|\mPhi( \vtheta_{t,0}-\vtheta^*_{\eta})\right\|_{\infty}+l_{V_3}.
\end{align*}
The first equality follows from algebraic decomposition and triangle inequality. The second inequality follows from lipschitzness of $V(\cdot)$ in Lemma~\ref{lem:poisson-property}. The last inequality follows lipschitzness of $\vtheta^*(\cdot)$ in Lemma~\ref{lem:theta-star-lipschitz}. This completes the proof.
\end{proof}

\begin{lemma}\label{lem:d-bound}
For $t\in\sN$ and $1\leq k \leq K$, we have
    \begin{align*}
        |d_{t,k}|  \leq & \frac{2}{\mu_{\eta}}\left( l_{\eta}l_{V_1}+4l_{\eta}l_{V_3} + 2\left( \kappa l_{V_1}\gamma +l_{\eta}l_{V_2}\right)\right)\left(L_{\eta}(\vtheta_{t,k},\vtheta_{t-1,K})-L_{\eta}(\vtheta^*(\vtheta_{t-1,K}),\vtheta_{t-1,K} ) \right) \\
        &+2\left( l_{V_1}\gamma\kappa + l_{\eta}l_{V_2} \right)\left\| \mPhi(\vtheta_{t,0}-\vtheta^*_{\eta}) \right\|_{\infty}^2+2l_{\eta}l_{V_3}.
    \end{align*}
\end{lemma}
\begin{proof}
From the definition of $d_{t,k}$ in~(\ref{eq:d_tk}),
\begin{align*}
    |d_{t,k}| = & | \nabla L_{\eta}(\vtheta_{t,k},\vtheta_{t,0})^{\top}V(\vtheta_{t,k},\vtheta_{t,0},s_{t,k},a_{t,k}) | \\
    \leq & \left\|\nabla L_{\eta}(\vtheta_{t,k},\vtheta_{t,0}) \right\|_2\left\| V(\vtheta_{t,k},\vtheta_{t,0},s_{t,k},a_{t,k})\right\|_2\\
        = & \left\|\nabla L_{\eta}(\vtheta_{t,k},\vtheta_{t,0}) -\nabla L_{\eta}(\vtheta^*(\vtheta_{t,0}),\vtheta_{t,0})\right\|_2\left\| V(\vtheta_{t,k},\vtheta_{t,0},s_{t,k},a_{t,k})\right\|_2\\
        \leq & l_{\eta}\left\| \vtheta_{t,k}- \vtheta^*(\vtheta_{t,0}) \right\|_2 \left( l_{V_1} \left\| \vtheta_{t,k}-\vtheta^*_{\eta} \right\|_2+l_2\left\|\mPhi(\vtheta_{t,0}-\vtheta^*_{\eta}) \right\|_{\infty}+l_{V_3} \right)\\
        \leq & l_{\eta}\left\| \vtheta_{t,k}- \vtheta^*(\vtheta_{t,0}) \right\|_2\left( l_{V_1} \left\| \vtheta_{t,k}-\vtheta^*(\vtheta_{t,0}) \right\|_2+l_{V_1} \left\|\vtheta^*(\vtheta_{t,0})-\vtheta^*_{\eta} \right\|_2+l_{V_2}\left\|\mPhi(\vtheta_{t,0}-\vtheta^*_{\eta}) \right\|_{\infty}+l_{V_3} \right)\\
        \leq & l_{\eta}l_{V_1}\left\| \vtheta_{t,k}-\vtheta^*(\vtheta_{t,0}) \right\|_2^2\\
        &+l_{\eta}\left\| \vtheta_{t,k}-\vtheta^*(\vtheta_{t,0}) \right\|_2\left( \left(\frac{l_{V_1}\gamma}{\mu_{\eta}}+l_{V_2} \right)\left\| \mPhi(\vtheta_{t,0}-\vtheta^*_{\eta}) \right\|_{\infty}+l_{V_3} \right)\\
        \leq & \left( l_{\eta}l_{V_1}+4l_{\eta}l_{V_3} + 2\left( \kappa l_{V_1}\gamma +l_{\eta}l_{V_2}\right)\right)\left\| \vtheta_{t,k}-\vtheta^*(\vtheta_{t,0})\right\|^2_2\\
        &+2\left( l_{V_1}\gamma\kappa + l_{\eta}l_{V_2} \right)\left\| \mPhi(\vtheta_{t,0}-\vtheta^*_{\eta}) \right\|_{\infty}^2+2 l_{\eta}l_{V_3}\\
         \leq & \frac{2}{\mu_{\eta}}\left( l_{\eta}l_{V_1}+4l_{\eta}l_{V_3} + 2\left( \kappa l_{V_1}\gamma +l_{\eta}l_{V_2}\right)\right)\left(L_{\eta}(\vtheta_{t,k},\vtheta_{t-1,K})-L_{\eta}(\vtheta^*(\vtheta_{t-1,K}),\vtheta_{t-1,K} ) \right) \\
        &+2\left( l_{V_1}\gamma\kappa + l_{\eta}l_{V_2} \right)\left\| \mPhi(\vtheta_{t,0}-\vtheta^*_{\eta}) \right\|_{\infty}^2+2l_{\eta}l_{V_3}.
\end{align*}

The first inequality follows from the Cauchy-Schwarz inequality. The third inequality follows from the smoothness of $L$.



\end{proof}

\section{Additional figure}


\begin{figure}[H]
    \centering
    \includegraphics[width=1.0\linewidth]{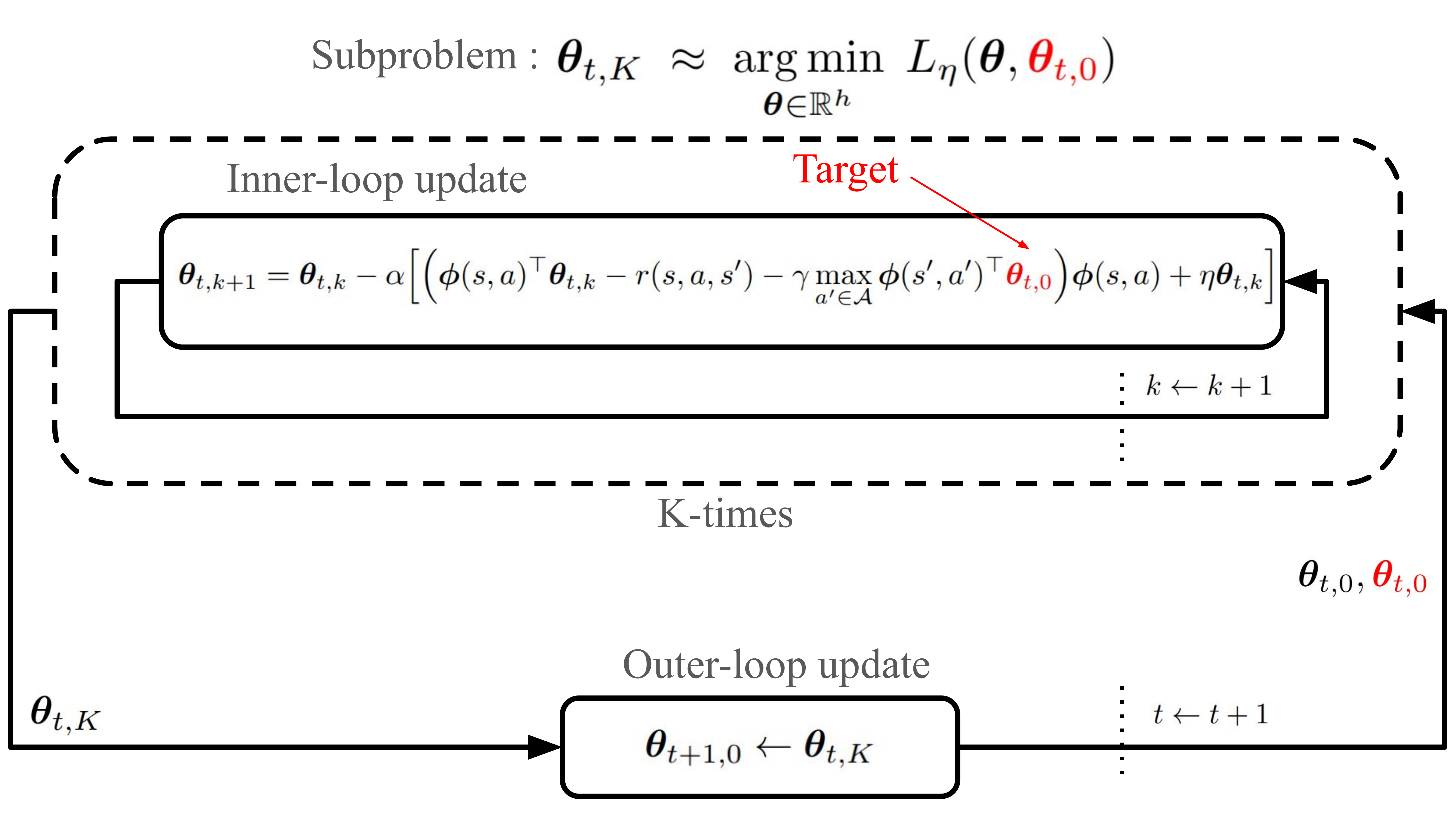}
    \caption{Double-loop structure of PRQ. The inner loop performs gradient descent to solve the regularized subproblem for K times, and the outer loop updates the target parameter.}
    \label{fig:f1-1}
\end{figure}




%% file: main_paper.bib
@inproceedings{karimi2016linear,
  title={Linear convergence of gradient and proximal-gradient methods under the polyak-{\l}ojasiewicz condition},
  author={Karimi, Hamed and Nutini, Julie and Schmidt, Mark},
  booktitle={Machine Learning and Knowledge Discovery in Databases: European Conference, ECML PKDD 2016, Riva del Garda, Italy, September 19-23, 2016, Proceedings, Part I 16},
  pages={795--811},
  year={2016},
  organization={Springer}
}

@article{bottou2018optimization,
  title={Optimization methods for large-scale machine learning},
  author={Bottou, L{\'e}on and Curtis, Frank E and Nocedal, Jorge},
  journal={SIAM review},
  volume={60},
  number={2},
  pages={223--311},
  year={2018},
  publisher={SIAM}
}

@inproceedings{lee2020periodic,
  title={Periodic {Q}-learning},
  author={Lee, Donghwan and He, Niao},
  booktitle={Learning for dynamics and control},
  pages={582--598},
  year={2020},
  organization={PMLR}
}

@article{farahmand2016regularized,
  title={Regularized policy iteration with nonparametric function spaces},
  author={Farahmand, Amir-massoud and Ghavamzadeh, Mohammad and Szepesv{\'a}ri, Csaba and Mannor, Shie},
  journal={Journal of Machine Learning Research},
  volume={17},
  number={139},
  pages={1--66},
  year={2016}
}

@inproceedings{geist2019theory,
  title={A theory of regularized markov decision processes},
  author={Geist, Matthieu and Scherrer, Bruno and Pietquin, Olivier},
  booktitle={International Conference on Machine Learning},
  pages={2160--2169},
  year={2019},
  organization={PMLR}
}

@article{bertsekas2011temporal,
  title={Temporal difference methods for general projected equations},
  author={Bertsekas, Dimitri P},
  journal={IEEE Transactions on Automatic Control},
  volume={56},
  number={9},
  pages={2128--2139},
  year={2011},
  publisher={IEEE}
}

@inproceedings{zhang2021breaking,
  title={Breaking the deadly triad with a target network},
  author={Zhang, Shangtong and Yao, Hengshuai and Whiteson, Shimon},
  booktitle={International Conference on Machine Learning},
  pages={12621--12631},
  year={2021},
  organization={PMLR}
}

@article{chen2023target,
  title={Target network and truncation overcome the deadly triad in {Q}-learning},
  author={Chen, Zaiwei and Clarke, John-Paul and Maguluri, Siva Theja},
  journal={SIAM Journal on Mathematics of Data Science},
  volume={5},
  number={4},
  pages={1078--1101},
  year={2023},
  publisher={SIAM}
}

@article{che2024target,
  title={{Target Networks and Over-parameterization Stabilize Off-policy Bootstrapping with Function Approximation}},
  author={Che, Fengdi and Xiao, Chenjun and Mei, Jincheng and Dai, Bo and Gummadi, Ramki and Ramirez, Oscar A and Harris, Christopher K and Mahmood, A Rupam and Schuurmans, Dale},
  journal={arXiv preprint arXiv:2405.21043},
  year={2024}
}

@article{asadi2024td,
  title={Td convergence: {A}n optimization perspective},
  author={Asadi, Kavosh and Sabach, Shoham and Liu, Yao and Gottesman, Omer and Fakoor, Rasool},
  journal={Advances in Neural Information Processing Systems},
  volume={36},
  year={2024}
}

@inproceedings{fellows2023target,
  title={Why target networks stabilise temporal difference methods},
  author={Fellows, Mattie and Smith, Matthew JA and Whiteson, Shimon},
  booktitle={International Conference on Machine Learning},
  pages={9886--9909},
  year={2023},
  organization={PMLR}
}

@inproceedings{lee2019target,
  title={Target-based temporal-difference learning},
  author={Lee, Donghwan and He, Niao},
  booktitle={International Conference on Machine Learning},
  pages={3713--3722},
  year={2019},
  organization={PMLR}
}

@article{manek2022pitfalls,
  title={The pitfalls of regularization in off-policy {TD} learning},
  author={Manek, Gaurav and Kolter, J Zico},
  journal={Advances in Neural Information Processing Systems},
  volume={35},
  pages={35621--35631},
  year={2022}
}

@inproceedings{limregularized,
  title={Regularized {Q}-Learning},
  author={Lim, Han-Dong and Lee, Donghwan},
    year={2024},
  booktitle={The Thirty-eighth Annual Conference on Neural Information Processing Systems}
}

@inproceedings{
gallici2024simplifying,
title={Simplifying Deep Temporal Difference Learning},
author={Matteo Gallici and Mattie Fellows and Benjamin Ellis and Bartomeu Pou and Ivan Masmitja and Jakob Nicolaus Foerster and Mario Martin},
booktitle={The Thirteenth International Conference on Learning Representations},
year={2025},
url={https://openreview.net/forum?id=7IzeL0kflu}
}

@article{wu2025unifying,
  title={{A Unifying View of Linear Function Approximation in Off-Policy RL Through Matrix Splitting and Preconditioning}},
  author={Wu, Zechen and Greenwald, Amy and Parr, Ronald},
  journal={arXiv preprint arXiv:2501.01774},
  year={2025}
}

@article{haque2024stochastic,
  title={{Stochastic Approximation with Unbounded Markovian Noise: A General-Purpose Theorem}},
  author={Haque, Shaan Ul and Maguluri, Siva Theja},
  journal={arXiv preprint arXiv:2410.21704},
  year={2024}
}

@article{glynn1996liapounov,
  title={A Liapounov bound for solutions of the {P}oisson equation},
  author={Glynn, Peter W and Meyn, Sean P},
  journal={The Annals of Probability},
  pages={916--931},
  year={1996},
  publisher={JSTOR}
}

@article{chen2022finite,
  title={Finite-sample analysis of nonlinear stochastic approximation with applications in reinforcement learning},
  author={Chen, Zaiwei and Zhang, Sheng and Doan, Thinh T and Clarke, John-Paul and Maguluri, Siva Theja},
  journal={Automatica},
  volume={146},
  pages={110623},
  year={2022},
  publisher={Elsevier}
}

@article{silver2017mastering,
  title={Mastering the game of go without human knowledge},
  author={Silver, David and Schrittwieser, Julian and Simonyan, Karen and Antonoglou, Ioannis and Huang, Aja and Guez, Arthur and Hubert, Thomas and Baker, Lucas and Lai, Matthew and Bolton, Adrian and others},
  journal={nature},
  volume={550},
  number={7676},
  pages={354--359},
  year={2017},
  publisher={Nature Publishing Group UK London}
}

@article{mnih2013playing,
  title={Playing atari with deep reinforcement learning},
  author={Mnih, Volodymyr and Kavukcuoglu, Koray and Silver, David and Graves, Alex and Antonoglou, Ioannis and Wierstra, Daan and Riedmiller, Martin},
  journal={arXiv preprint arXiv:1312.5602},
  year={2013}
}

@article{watkins1992q,
  title={{Q}-learning},
  author={Watkins, Christopher JCH and Dayan, Peter},
  journal={Machine learning},
  volume={8},
  number={3},
  pages={279--292},
  year={1992},
  publisher={Springer}
}

@book{sutton1998reinforcement,
  title={Reinforcement learning: {A}n introduction},
  author={Sutton, Richard S and Barto, Andrew G and others},
  volume={1},
  number={1},
  year={1998},
  publisher={MIT press Cambridge}
}

@inproceedings{melo2007convergence,
  title={Convergence of {Q}-learning with linear function approximation},
  author={Melo, Francisco S and Ribeiro, M Isabel},
  booktitle={2007 European control conference (ECC)},
  pages={2671--2678},
  year={2007},
  organization={IEEE}
}

@inproceedings{melo2008analysis,
  title={An analysis of reinforcement learning with function approximation},
  author={Melo, Francisco S and Meyn, Sean P and Ribeiro, M Isabel},
  booktitle={Proceedings of the 25th international conference on Machine learning},
  pages={664--671},
  year={2008}
}

@inproceedings{yang2019sample,
  title={Sample-optimal parametric q-learning using linearly additive features},
  author={Yang, Lin and Wang, Mengdi},
  booktitle={International conference on machine learning},
  pages={6995--7004},
  year={2019},
  organization={PMLR}
}

@article{lee2020unified,
  title={A unified switching system perspective and convergence analysis of {Q}-learning algorithms},
  author={Lee, Donghwan and He, Niao},
  journal={Advances in neural information processing systems},
  volume={33},
  pages={15556--15567},
  year={2020}
}

@article{lim2025understanding,
  title={Understanding the theoretical properties of projected {B}ellman equation, linear {Q}-learning, and approximate value iteration},
  author={Lim, Han-Dong and Lee, Donghwan},
  journal={arXiv preprint arXiv:2504.10865},
  year={2025}
}

@inproceedings{maei2010toward,
  title={Toward off-policy learning control with function approximation.},
  author={Maei, Hamid Reza and Szepesv{\'a}ri, Csaba and Bhatnagar, Shalabh and Sutton, Richard S},
  booktitle={ICML},
  volume={10},
  pages={719--726},
  year={2010}
}

@inproceedings{lu2021convex,
  title={Convex {Q}-learning},
  author={Lu, Fan and Mehta, Prashant G and Meyn, Sean P and Neu, Gergely},
  booktitle={2021 American Control Conference (ACC)},
  pages={4749--4756},
  year={2021},
  organization={IEEE}
}

@article{devraj2017zap,
  title={Zap {Q}-learning},
  author={Devraj, Adithya M and Meyn, Sean},
  journal={Advances in Neural Information Processing Systems},
  volume={30},
  year={2017}
}

@article{clarke1981generalized,
  title={Generalized gradients of {L}ipschitz functionals},
  author={Clarke, Frank H},
  journal={Advances in Mathematics},
  volume={40},
  number={1},
  pages={52--67},
  year={1981},
  publisher={Elsevier}
}

@article{de2000existence,
  title={On the existence of fixed points for approximate value iteration and temporal-difference learning},
  author={De Farias, Daniela Pucci and Van Roy, Benjamin},
  journal={Journal of Optimization theory and Applications},
  volume={105},
  number={3},
  pages={589--608},
  year={2000},
  publisher={Springer}
}

@article{meyn2024projected,
  title={The projected Bellman equation in reinforcement learning},
  author={Meyn, Sean},
  journal={IEEE Transactions on Automatic Control},
  volume={69},
  number={12},
  pages={8323--8337},
  year={2024},
  publisher={IEEE}
}

@inproceedings{baird1995residual,
  title={Residual algorithms: {R}einforcement learning with function approximation},
  author={Baird, Leemon and others},
  booktitle={Proceedings of the twelfth international conference on machine learning},
  pages={30--37},
  year={1995}
}

@book{bertsekas2012dynamic,
  title={Dynamic programming and optimal control: {Volume I}},
  author={Bertsekas, Dimitri},
  volume={4},
  year={2012},
  publisher={Athena scientific}
}

@article{clarke1976new,
  title={A new approach to {L}agrange multipliers},
  author={Clarke, Frank H},
  journal={Mathematics of Operations Research},
  volume={1},
  number={2},
  pages={165--174},
  year={1976},
  publisher={INFORMS}
}

@article{clarke1975generalized,
  title={Generalized gradients and applications},
  author={Clarke, Frank H},
  journal={Transactions of the American Mathematical Society},
  volume={205},
  pages={247--262},
  year={1975}
}

@book{nesterov2018lectures,
  title={Lectures on convex optimization},
  author={Nesterov, Yurii and others},
  volume={137},
  year={2018},
  publisher={Springer}
}

@article{zhang2023convergence,
  title={On the convergence and sample complexity analysis of deep q-networks with $\epsilon$-greedy exploration},
  author={Zhang, Shuai and Li, Hongkang and Wang, Meng and Liu, Miao and Chen, Pin-Yu and Lu, Songtao and Liu, Sijia and Murugesan, Keerthiram and Chaudhury, Subhajit},
  journal={Advances in Neural Information Processing Systems},
  volume={36},
  pages={13064--13102},
  year={2023}
}
